\documentclass[letterpaper]{article} 
\RequirePackage{multirow}
\usepackage{times}  
\usepackage{helvet}  
\usepackage{courier}  
\usepackage{url}  
\usepackage{graphicx}  
\usepackage{authblk}
 \newlength\titlebox 
\usepackage{algorithm}
\usepackage[noend]{algorithmic}
\usepackage{amssymb}
\usepackage{amsmath}
\usepackage{amsthm}
\usepackage{color}
\usepackage{bm}
\usepackage[version=3]{mhchem}
\usepackage{subcaption}
\usepackage{latexsym}
\usepackage{booktabs}
\usepackage[round]{natbib}
\usepackage{hyperref}
\renewcommand{\cite}{\citep}

\newcommand*{\email}[1]{\href{mailto:#1}{\nolinkurl{#1}} }
\DeclareMathOperator*{\argmax}{arg\,max}

\newtheorem{thm}{Theorem}

\newtheorem{lemma}[thm]{Lemma}

\newtheorem{deff}{Definition}
\newtheorem{cond}[thm]{Condition}

\newcommand{\cD}{\mathcal{D}}
\newcommand{\cG}{\mathcal{G}}
\newcommand{\cH}{\mathcal{H}}

\newcommand{\cL}{\mathcal{L}}
\newcommand{\cM}{\mathcal{M}}
\newcommand{\cN}{\mathcal{N}}
\newcommand{\cS}{\mathcal{S}}

\newcommand{\cY}{\mathcal{Y}}

\newcommand{\cZ}{\mathcal{Z}}
\newcommand{\bbR}{\mathbb{R}}

\newcommand{\bbE}{\mathbb{E}}
\newcommand{\bbN}{\mathbb{N}}

\newcommand{\ie}{\textit{i.e.}}
\newcommand{\eg}{\textit{e.g.}}

\newcommand{\pa}{\mathrm{pa}}
\newcommand{\ch}{\mathrm{ch}}
\newcommand{\hrg}{\mathsf{HRG}}
\newcommand{\mhg}{\mathsf{MHG}}
\newcommand{\ext}{\mathrm{ext}}

\newcommand{\enc}{\mathsf{Enc}}
\newcommand{\dec}{\mathsf{Dec}}

\begin{document}
%
\title{Molecular Hypergraph Grammar\\with Its Application to Molecular Optimization}
\author[1]{Hiroshi Kajino\\ \email{kajino@jp.ibm.com}}
\affil[1]{MIT-IBM Watson AI Lab\\IBM Research}
\maketitle
\begin{abstract}
 Molecular optimization aims to discover novel molecules with desirable properties.
 Two fundamental challenges are:
 (i)~it is not trivial to generate valid molecules in a controllable way due to hard chemical constraints such as the valency conditions, and
 (ii)~it is often costly to evaluate a property of a novel molecule, and therefore, the number of property evaluations is limited.
 These challenges are to some extent alleviated by a combination of a variational autoencoder~(VAE) and Bayesian optimization~(BO).
 VAE converts a molecule into/from its latent continuous vector,
 and BO optimizes a latent continuous vector~(and its corresponding molecule) within a limited number of property evaluations.
 While the most recent work, for the first time, achieved 100\% validity,
 its architecture is rather complex due to auxiliary neural networks other than VAE,
 making it difficult to train.
 This paper presents a \emph{molecular hypergraph grammar variational autoencoder}~(MHG-VAE), which uses a single VAE to achieve 100\% validity.
 Our idea is to develop a graph grammar encoding the hard chemical constraints, called \emph{molecular hypergraph grammar}~(MHG), which guides VAE to always generate valid molecules.
 We also present an algorithm to construct MHG from a set of molecules.
\end{abstract}

\begin{figure*}[t]
\centering
 \includegraphics[width=.9\hsize]{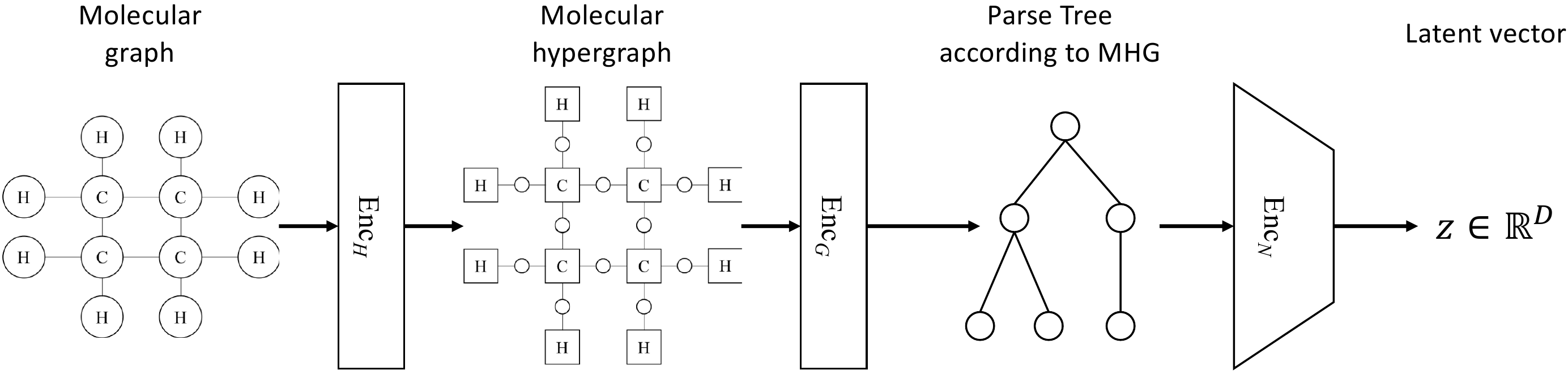}
\caption{Illustration of our encoder. For a molecular hypergraph, squares represent hyperedges, circles represent nodes, and a circle-square line indicates that the node is a member of the hyperedge. The decoder is defined by inversing the encoder.}
\label{fig:encoder-overview}
\end{figure*}

\section{Introduction}
Molecular optimization aims to discover a novel molecule that possesses prescribed properties given by a user.
For example, \citet{bombarelli2016a} aim to maximize the efficiency of an organic light-emitting diode.
Letting $\cM$ be a set of valid molecules, the molecular optimization problem is to obtain $m^\star\in\cM$ such that,
\begin{align}
\label{eq:2} m^\star = \argmax_{m\in\cM} f(m),
\end{align}
where $f\colon\cM\rightarrow\bbR$ is an unknown function that outputs a chemical property of the input molecule to be maximized.

There are two obstacles to solve Eq.~\eqref{eq:2}.
First, the set of feasible solutions $\cM$ is discrete, and it is difficult to generate a candidate $m$ from $\cM$.
Second, $f$ is unknown and costly to evaluate, and only a small sample $\{(m_n,y_n)\}_{n=1}^N\subset \cM \times \bbR$ is available, where $y_n=f(m_n)$.
In fact, the function evaluation often requires wet-lab experiments or day-long computer simulation based on quantum mechanics.

A recent innovation~\cite{bombarelli2016} facilitates the optimization by leveraging VAE~\cite{kingma2014} and BO~\cite{mokus1975}.
The first challenge is addressed by casting the discrete optimization problem into continuous with the help of VAE.
In specific, they first train a VAE, a pair of $\enc\colon\cM\rightarrow\bbR^D$ and $\dec\colon\bbR^D\rightarrow\cM$ 
such that $\dec(\enc(m))\approx m$ holds for any $m\in\cM$, and they obtain $m^\star$ as follows:
\begin{align}
\label{eq:3}
\begin{split}
m^\star &= \dec\left(\argmax_{z\in\bbR^D} f(\dec(z))\right),
\end{split}
\end{align}
The second challenge is addressed by using BO, which iteratively optimizes a black-box function in a limited number of function evaluations.
In specific, the inner optimization problem in Eq.~\eqref{eq:3} is solved by BO with sample $\{(\enc(m_1), y_1), \dots, (\enc(m_N), y_N)\}\subset\bbR^D\times\bbR$,
and the resultant latent vector $x_{N+1}\in\bbR^D$ is added to the sample with property evaluation $f(\dec(x_{N+1}))$.

While they elegantly address the two obstacles, the decoding sometimes fails, and no molecule is obtained, which we call the \emph{decoding error issue}.
%
This is mainly due to the use of SMILES~\cite{Weininger:1988aa} to represent a molecule.
Let $\Sigma$ be a set of symbols used in SMILES, and $\enc_S\colon\cM\rightarrow\Sigma^\ast$ be a SMILES encoder.
For example, $\Sigma$ includes atomic symbols, \eg, \ce{C}, \ce{H} $\in \Sigma$; given a phenol as input, the encoder outputs \texttt{c1c(O)cccc1}, where the digits represent the start and end points of the benzene ring, the parentheses represent branching, and hydrogen atoms are omitted.
Letting $\enc_S[\cM] := \{\enc_S(m) \mid m \in \cM\}\subsetneq\Sigma^\ast$ be the set of all \emph{valid} SMILES strings, a SMILES decoder $\dec_S\colon\enc_S[\cM]\rightarrow\cM$ can be defined.
Notice that the domain of $\dec_S$ is not $\Sigma^\ast$ but $\enc_S[\cM]$, the set of strings that follow SMILES' grammar, because any string that violates the grammar cannot be decoded into any molecule.

In their implementation, the encoder is composed as $\enc=\enc_N \circ \enc_S$, where $\enc_{N}\colon\Sigma^\ast\rightarrow\bbR^D$ is a neural network encoder,
and the decoder is composed as $\dec=\dec_S\circ\dec_N$, where $\dec_N\colon\bbR^D\rightarrow\Sigma^\ast$ is a neural network decoder, generating symbols one by one.
The decoding error issue occurs when the output of $\dec_N$ does not belong to $\enc_S[\cM]$, the domain of $\dec_S$.
For the phenol example, if $\dec_N$ fails to output the end digit, \texttt{c1c(O)cccc}, it cannot be converted into a molecule because the ring cannot be closed;
if $\dec_N$ generates more than one pair of parentheses, \texttt{c1c(O)(O)cccc1}, it violates the valence condition of carbon.
Since SMILES' grammar is a context-\emph{sensitive} grammar, it is not straightforward to develop a neural network that always generates a string belonging to $\enc_S[\cM]$.

Recently, several studies have been conducted towards addressing the decoding error issue~\cite{kusner2017,dai2018,jin2018}.
Among them, \citet{jin2018} for the first time report 100\% validity, addressing the decoding error issue.
Their idea is to represent a molecular graph as fragments~(such as rings and atom branches) connected in a tree structure.
Such a tree representation is preferable because it is easier to generate a tree than a general graph with degree constraints.
By forcing the decoder to generate only valid combination of fragments,
the decoder can always generate a valid molecule.

While their tree representation successfully addresses the decoding error issue,
it models only part of molecular properties, and the rest is left to neural networks.
For example, their representation only specifies fragment-level connections, and does not specify which atoms in the fragments to be connected.
In addition, since it does not specify atom-level connections, the stereochemistry information disappears.
They instead enumerate all possible configurations and pick one by training several auxiliary neural networks.

Given this literature, we are interested in addressing the decoding error issue without the auxiliary neural networks.
Such a simple architecture will facilitate model training, and it will be easier to achieve high performance with less effort.
Our idea is to develop a (i) context-free graph grammar of a molecular graph that (ii) never generates an invalid molecule, and that (iii) can represent the atom-level connection and stereochemistry information.
With such a grammar, we can use a parse tree as an intermediate representation of a molecule, which is easy for the decoder to generate~(as guaranteed by the context-freeness~(i)).
Combined with (ii) and (iii), our decoder can always output a valid molecule with stereochemistry information, using a single VAE only.

Our technical highlight is a \emph{molecular hypergraph grammar}~(MHG) along with an algorithm to infer MHG from a set of molecules.
MHG is designed to satisfy all of the requirements mentioned above, 
We also develop a molecular hypergraph grammar variational autoencoder~(MHG-VAE), which combines MHG and VAE to obtain an autoencoder for a molecule.
Figure~\ref{fig:encoder-overview} illustrates our encoder, where the first two encoders, $\enc_H$ and $\enc_G$, and their corresponding decoders $\dec_H$ and $\dec_G$, are our main contributions, while we use a standard seq2seq VAE for $(\enc_N, \dec_N)$.

In details, we develop MHG by tailoring a hyperedge replacement grammar~(HRG)~\cite{drewes1997} for a \emph{molecular hypergraph}~(thus, MHG is a special case of HRG).
A molecular hypergraph models an atom by a hyperedge and a bond by a node.
HRG is a context-free graph grammar generating a hypergraph by replacing a non-terminal hyperedge with another hypergraph; it achieves atom-level connections when combined with a molecular hypergraph, and stereochemistry can also be encoded into the grammar.
It also preserves the number of nodes belonging to each hyperedge, which coincides with the valency of an atom in our case.
Therefore, these two ideas allow us to always generate valid molecules using a single VAE.

Our MHG inference algorithm extends the existing HRG inference algorithm~\cite{aguinaga2016} so that the resultant HRG always generates a molecular hypergraph.
The existing one infers HRG by extracting a set of production rules from a tree decomposition of each hypergraph, which is equivalent to a parse tree.
Our finding is that, while the inferred HRG preserves the valence condition, it sometimes generates a hypergraph that cannot be decoded into a molecular graph; the generated hypergraph may contain a node that is shared by more than two hyperedges.
To address this issue, we develop an \emph{irredundant} tree decomposition, with which HRG is guaranteed to generate a valid molecular hypergraph, \ie, the inferred HRG is always MHG.

\section{Preliminaries}

\begin{figure}[t]
\centering
 \includegraphics[width=.6\linewidth]{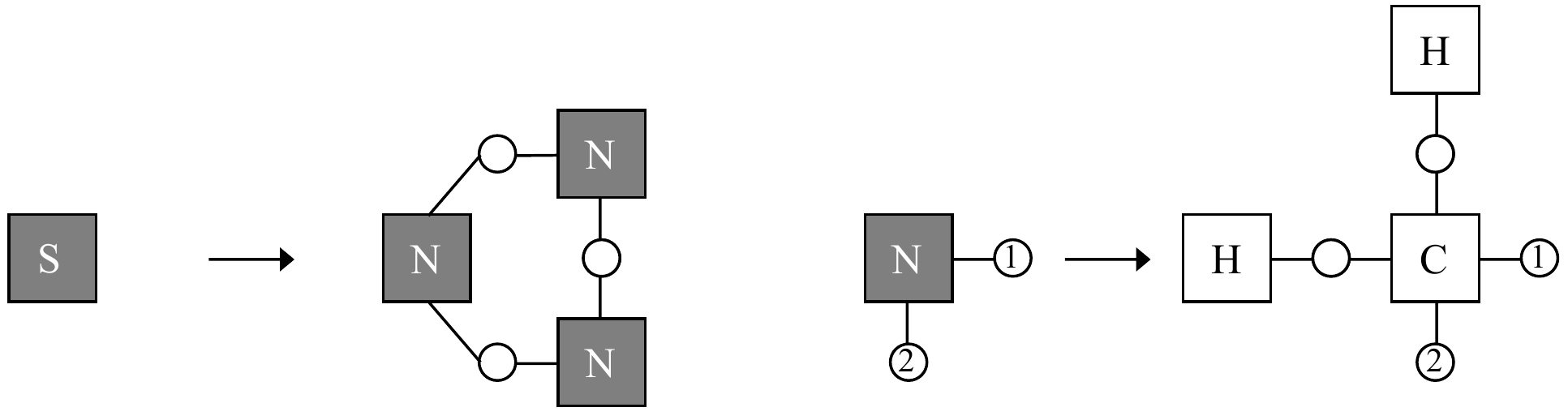}
\caption{Part of production rules extracted from the hypergraph in Fig.~\ref{fig:encoder-overview}. Filled squares represent non-terminals, and unfilled ones represent terminals. Numbers in nodes indicate the correspondence between the nodes in $A$ and $R$.}
\label{fig:hrg-example}
\end{figure}
A hypergraph is a pair $H=(V_H,E_H)$, where $V_H$ is a set of nodes, and $E_H$ is a set of non-empty subsets of $V_H$, called hyperedges.
A hypergraph is called \emph{$k$-regular} if every node has degree $k$.

A tree decomposition of a hypergraph~(Def.~\ref{def:tree-decomp}) discovers a tree-like structure of the hypergraph.
Figure~\ref{fig:tree-decomp} illustrates a tree decomposition of the hypergraph shown in Fig.~\ref{fig:encoder-overview}.

\begin{deff}
\label{def:tree-decomp}
 A \emph{tree decomposition} of hypergraph $H=(V_H,E_H)$ is tree $T=(V_T, E_T)$ with two labeling functions $\ell_T^{(V)}\colon V_T\rightarrow 2^{V_H}$ and $\ell_T^{(E)}\colon V_T\rightarrow 2^{E_H}$ such that:
\begin{enumerate}
 \item For each $v_H\in V_H$, there exists at least one node $v_T\in V_T$ such that $v_H\in\ell_T^{(V)}(v_T)$.
 \item For each $e_H\in E_H$, there exists exactly one node $v_T\in V_T$ such that $e_H\subseteq\ell_T^{(V)}(v_T)$ and $e_H\in\ell_T^{(E)}(v_T)$.
 \item For each $v_H\in V_H$, a set of nodes $\{v_T\in V_T \mid v_H\in \ell_T^{(V)}(v_T)\}$ is connected in $T$.
\end{enumerate}
\end{deff}
Let us denote the hypergraph on node $v_T\in V_T$ by $H(v_T):=(\ell_{T}^{(V)}(v_T), \ell_{T}^{(E)}(v_T))$.

HRG~\cite{drewes1997} is a context-free grammar generating hypergraphs with labeled nodes and hyperedges~(Def.~\ref{def:hyper-repl-gramm}, Fig.~\ref{fig:hrg-example}).
It starts from the starting symbol~$S$ and repeatedly replaces a non-terminal symbol $A$ in the hypergraph with a hypergraph~$R$, which may have both terminal and non-terminal symbols.

\begin{deff}
\label{def:hyper-repl-gramm}
 A hyperedge replacement grammar is a tuple~$\hrg=(N, T, S, P)$, where,
\begin{enumerate}
 \item $N$ is a set of non-terminal hyperedge labels.
 \item $T$ is a set of terminal hyperedge labels.
 \item $S\in N$ is a starting non-terminal hyperedge.
 \item $P$ is a set of production rules where,
       \begin{itemize}
	\item $p=(A, R)$ is a production rule,
	\item $A\in N$ is a non-terminal symbol, and
	\item $R$ is a hypergraph with hyperedge labels $T\cup N$ and has $|A|$ external nodes.
	      Non-terminals in $R$ are ordered.
       \end{itemize}
\end{enumerate}
\end{deff}

We define a \emph{parse tree} according to HRG as follows.
Each node of the parse tree is labeled by a production rule.
The production rules of the leaves of the parse tree must not contain non-terminals in their $R$s.
If the production rule $p$ is a starting rule, the node has $N_p$ ordered children, where $N_p$ denotes the number of non-terminals in $R$, and the edges are ordered by the orders of the non-terminals.
Otherwise, the node has one parent and $N_p$ ordered children, where the corresponding non-terminal in the parental production rule must coincide with $A$ of the production rule.

Given a parse tree, we can construct a hypergraph by sequentially applying the production rules.
Such sequential applications of production rules are equivalent to the parse tree, and we call it a \emph{parse sequence}.

\section{Molecular Graph and Hypergraph}
This section introduces our definitions of a molecular graph and a molecular hypergraph.
We also present a pair of encoder and decoder between them, $(\enc_H, \dec_H)$.

\subsection{Molecular Graph}
A molecular graph~(Def.~\ref{def:mol-graph}) represents the structural formula of a molecule using a graph, where atoms are modeled as labeled nodes and bonds as labeled edges.
Typically, the node label is defined by the atom's symbol~(\eg, \ce{H}, \ce{C}) and its formal charge, and the edge label by the bond type~(\eg, single, double).
The graph must satisfy the valency condition; the degree of each atom is specified by its label~(\eg, the degree of \ce{C} must equal four).
Let $\cG(L_G^{(V)}, L_G^{(E)}, d^{(V)})$ be the set of all possible molecular graphs, given the sets of node and edge labels and the degree constraint function.

\begin{deff}
 \label{def:mol-graph}
 Let $L_G^{(V)}$ and $L_G^{(E)}$ be sets of node and edge labels.
 Let $d^{(V)}\colon L_G^{(V)}\rightarrow\bbN$ be a degree constraint function.
 Let $G=(V_G, E_G, \ell_G^{(V)}, \ell_G^{(E)})$ be a node and edge-labeled graph, where $V_G$ is a set of nodes, $E_G$ is a set of undirected edges,
$\ell_G^{(V)}\colon V_G\rightarrow L_G^{(V)}$ is a node-labeling function, and $\ell_G^{(E)}\colon E_G\rightarrow L_G^{(E)}$ is an edge-labeling function.
 A \emph{molecular graph} $G$ is a node and edge-labeled graph that satisfies $d(v)=d^{(V)}(\ell_G^{(V)}(v))$ for all $v\in V_G$, where $d(v)$ indicates the degree of node $v$.
\end{deff}

There are two types of important properties that influence the chemical properties of a molecule.
The first one is the aromaticity of a ring~(\eg, benzene derivatives).
The bonds in an aromatic ring are different from a single or double bond, and are known to be more stable.
We do not explicitly encode any information related to the aromaticity, and instead, employ the Kekulé structure, where an aromatic ring is represented by alternating single and double bonds.
This does not lose generality because we can infer the aromaticity from the Kekulé representation.
The second one is the stereochemistry, which specifies 2D or 3D configuration of atoms.
We deal with the configuration at a double bond and tetrahedral carbon.
The double bond configuration is encoded by an E-Z configuration label assigned on the edge label.
Given the label and the whole structure of the molecule, the Cahn–Ingold–Prelog priority rules can specify the double bond direction.
For the tetrahedral chirality information, we assign a chirality tag in the node label, following the implementation of RDKit.

In summary, we employ the graph representation~(Def.~\ref{def:mol-graph}), where the node label contains the atom symbol, formal charge, and the tetrahedral chirality tag, and the edge label contains the bond type and the E-Z configuration.

\subsection{Molecular Hypergraph}
As an intermediate representation, we use a {molecular hypergraph}~(Def.~\ref{def:molecular-hypergraph}), where
an atom is modeled by a hyperedge and a bond between two atoms by a node shared by the corresponding two hyperedges.
Let $\cH(L_H^{(V)}, L_H^{(E)}, c^{(E)})$ be the set of all molecular hypergraphs, given the sets of node and hyperedge labels and the cardinality constraint function.
\begin{deff}
\label{def:molecular-hypergraph}
 Let $L_H^{(E)}$ and $L_H^{(V)}$ be sets of hyperedge and node labels.
 Let $c^{(E)}\colon L_H^{(E)}\rightarrow \bbN$ be a cardinality constraint function.
 Let $H=(V_H, E_H, \ell_H^{(E)}, \ell_H^{(V)})$ be a node and hyperedge-labeled hypergraph, where $V_H$ is a set of nodes, $E_H$ is a set of hyperedges,
$\ell_H^{(V)}\colon V_H\rightarrow L_H^{(V)}$ is a node-labeling function, and $\ell_H^{(E)}\colon E_H\rightarrow L_H^{(E)}$ is a hyperedge-labeling function.
 A \emph{molecular hypergraph} $H$ is a node and hyperedge-labeled hypergraph that satisfies the followings:\\
1. (Regularity) $H$ is $2$-regular.\\
2. (Cardinality) for each $e\in E_H$, $|e| = c^{(E)}(\ell_H^{(E)}(e))$ holds, where $|e|$ is the cardinality of hyperedge $e$.
\end{deff}
Note that the regularity condition in Def.~\ref{def:molecular-hypergraph} assures that any molecular hypergraph can be decoded into a graph.

\subsection{Encoder and Decoder}\label{sec:encod-decod-betw}
Finally, we present the encoder and decoder between a molecular graph and a molecular hypergraph, $(\enc_H, \dec_H)$.
They can be derived easily by swapping nodes-hyperedges and edges-nodes.
The regularity condition in Def.~\ref{def:molecular-hypergraph} assures the swap to work.
This equivalence immediately yields the following:
\begin{thm}
 If $L_G^{(V)}=L_H^{(E)}$, $L_G^{(E)}=L_H^{(V)}$, and $d^{(V)}(l)=c^{(E)}(l)$ for all $l\in L_G^{(V)}$ hold, then the followings hold:
\begin{align*}
\cH(L_H^{(V)}, L_H^{(E)}, c^{(E)}) = \enc_H[\cG(L_G^{(V)}, L_G^{(E)}, d^{(V)})],\\
\cG(L_G^{(V)}, L_G^{(E)}, d^{(V)}) = \dec_H[\cH(L_H^{(V)}, L_H^{(E)}, c^{(E)})],\\
G = \dec_H(\enc_H(G))~(\forall G\in\cG(L_G^{(V)}, L_G^{(E)}, d^{(V)})).
\end{align*}
\end{thm}

\section{Molecular Hypergraph Grammar}\label{sec:molec-hypergr-gramm}
A \emph{molecular hypergraph grammar}~(MHG) is defined as an HRG that always generates molecular hypergraphs.
In this section, we present $(\enc_G,\dec_G)$ that leverages MHG to represent a molecular hypergraph as a parse sequence.

Let $\mhg=(N,T,S,P)$ be a molecular hypergraph grammar, and $\cL_{\mhg}$ be its language, \ie, the set of molecular hypergraphs that can be generated by $\mhg$.
The encoder $\enc_G\colon \cL_\mhg\rightarrow P^\ast$ maps a molecular hypergraph into the corresponding parse sequence.
The decoder maps a parse sequence into a molecular hypergraph by sequentially applying the production rules.
It accepts a sequence of production rules obtained from $\cL_\mhg$ only, because other sequences cannot generate a molecular hypergraph.
Thus, the domain of the decoder is defined as $\dec_G\colon \enc_G[\cL_\mhg]\rightarrow \cL_\mhg$.
Clearly, for any $H\in\cL_\mhg$, $H=\dec_G(\enc_G(H))$ holds.

\section{MHG Inference Algorithm}\label{sec:mhg-infer-algor}
We present an algorithm to infer MHG from a set of molecular hypergraphs.
Our algorithm~(Sec.~\ref{sec:algorithm}) extends an existing HRG inference algorithm~(Sec.~\ref{sec:exist-hrg-infer}),
which extracts a set of production rules from tree decompositions of hypergraphs.
We need to tailor a novel inference algorithm for MHG because HRG inferred by applying the existing one to molecular hypergraphs is not MHG.
The inferred HRG does not necessarily generate a molecular hypergraph because it sometimes violates the regularity condition in Def.~\ref{def:molecular-hypergraph}.

\subsection{Existing HRG Inference Algorithm}\label{sec:exist-hrg-infer}
\citet{aguinaga2016} propose an algorithm to infer HRG from a set of hypergraphs.
Their key observation is that tree decompositions of hypergraphs yield HRG whose associated language includes the whole input hypergraphs.
%
Assume that we have a tree decomposition $T$ of hypergraph $H$.
We arbitrarily choose one node from $T$ as the root node.
For node $v_{T}\in V_T$, let $\pa(v_{T})$ be the parent of $v_{T}$ and $\ch(v_{T})$ be a set of children of $v_{T}$.
They first notice that connecting each pair $(v_{T}, \pa(v_{T}))$ by their common nodes yields the original hypergraph~(\eg, connecting such pairs in Fig.~\ref{fig:tree-decomp} yields the hypergraph in Fig.~\ref{fig:encoder-overview}).
In other words, a tree decomposition with an arbitrary root node is equivalent to a parse tree.
Given this observation, their algorithm extracts a production rule from a triplet $(\pa(v_{T}), v_{T}, \ch(v_{T}))$ so that the production rule can paste $H(v_T)$ on $H(\pa(v_T))$ with non-terminals left for applying the following production rules obtained from the children.
Note that the algorithm outputs not only HRG but also parse sequences of input hypergraphs.
For more details, see Appendix~\ref{sec:hrg-infer-algor}.



\subsection{Our MHG Inference Algorithm}\label{sec:algorithm}
We find that the existing algorithm cannot infer MHG, \ie, the inferred HRG sometimes violates the regularity condition.
We develop an \emph{irredundant} tree decomposition so that the violation does not occur, and substitute it for a generic tree decomposition to derive our MHG inference algorithm.


\noindent
\textbf{Irredundant Tree Decomposition.}
We introduce a key property of a tree decomposition called \emph{irredundancy}, which is necessary to guarantee the regularity.
Intuitively, a tree decomposition is irredundant if each node of the tree does not contain redundant nodes of the original hypergraph.
Figures~\ref{fig:tree-decomp} and \ref{fig:redundant} illustrate both irredundant and redundant tree decompositions.
The formal definition appears in Def.~\ref{def:irred-tree-decomp}.

\begin{figure}[t]
\centering
\begin{minipage}[t]{.4\hsize}
\centering
 \includegraphics[width=\hsize]{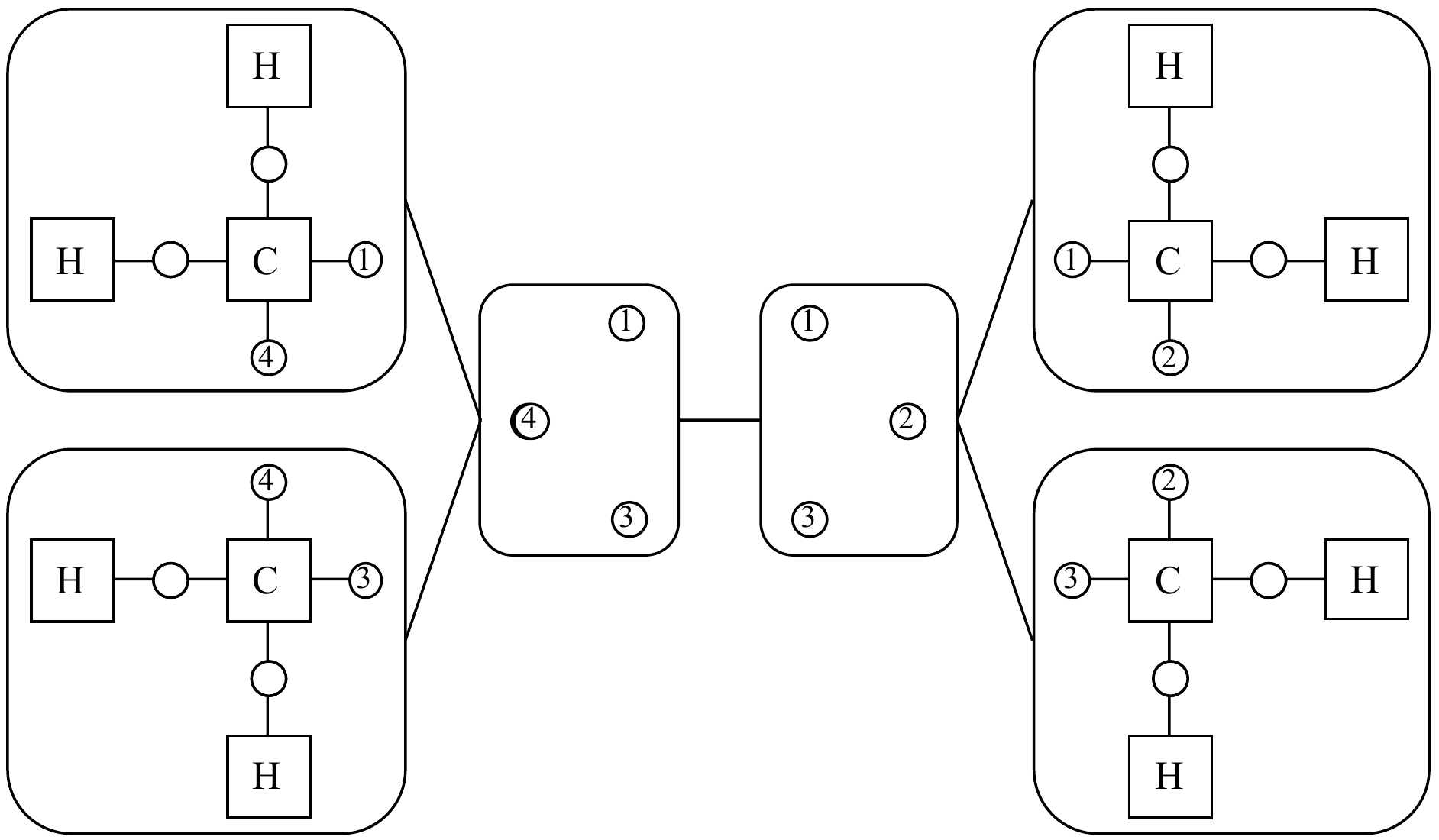}
 \caption{Irredundant tree decomposition of the hypergraph in Fig~\ref{fig:encoder-overview}. }
\label{fig:tree-decomp}
\end{minipage}
\hspace{.05\hsize}
\begin{minipage}[t]{.4\hsize}
\centering
 \includegraphics[width=\hsize]{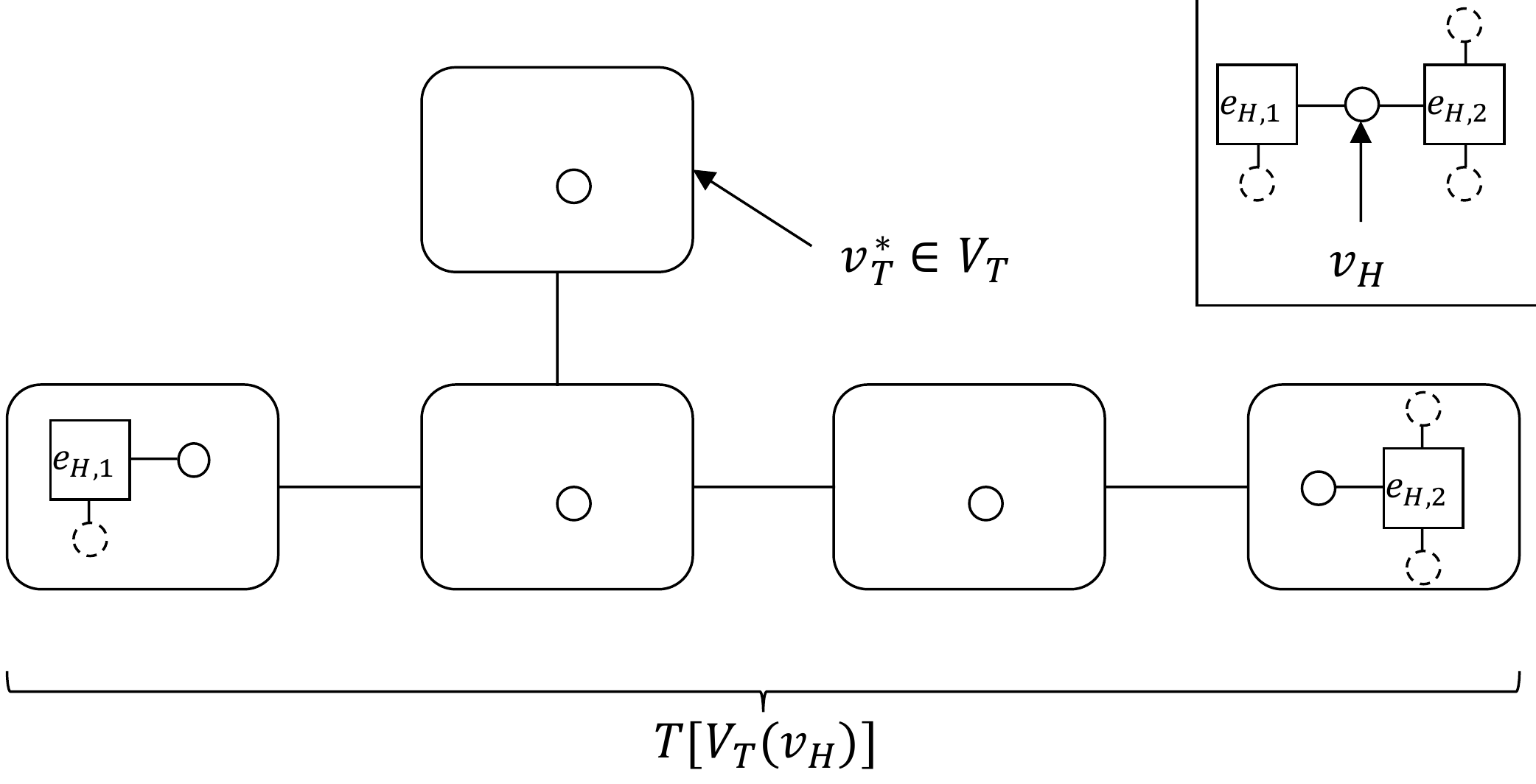}
 \caption{\emph{Redundant} tree decomposition~$T$. The upper-right shows a subhypergraph of $H$, whose tree decomposition induced by $V_T(v_H)$ is shown in the lower-left. It is necessary to remove $v_H$ from $v_T^\ast$ to transform it to be irredundant.}
\label{fig:redundant}
\end{minipage}
\end{figure} 

\begin{deff}
\label{def:irred-tree-decomp}
 Let $H=(V_H,E_H)$ be a hypergraph, and
 $(T, \ell_T^{(V)}, \ell_T^{(E)})$ be its tree decomposition.
 Let $V_T(v_H) = \{v_T\in V_T \mid v_H\in \ell_{T}^{(V)}(v_T)\}$ be a set of nodes in $T$ that contain $v_H\in V_H$.
 A tree decomposition is \emph{irredundant} if,
 \begin{align}
 \nonumber&\forall v_H\in V_H,\ \forall v_T\in V_T(v_H)\\
\label{eq:1} &\ell_T^{(E)}(v_T)\neq\emptyset\Leftrightarrow v_T\ \text{is a leaf in}\ T[V_T(v_H)],
 \end{align}
where $T[V_T(v_H)]$ is the subgraph induced by $V_T(v_H)$.
\end{deff}


We can make any tree decomposition to be irredundant in polynomial time; for each $v_H$, if it does not satisfy the condition~\eqref{eq:1}, remove $v_H$ from each $\ell_T^{(E)}(v_T)~(v_T\in V_T(v_H))$.

\noindent
\textbf{Implementation.}
We tailor a tree decomposition algorithm for a molecular hypergraph.
It starts from a one-node tree whose node contains the input hypergraph, and updates the tree by applying the following two steps.

The first step is to find a node such that the input hypergraph becomes disjoint when divided at the node, and to divide the hypergraph into two.
When dividing a hypergraph at a node, the node is duplicated so that the two hyperedges that the node belonged to still contain the node.
This operation is repeatedly applied to the subhypergraphs until there does not exist such a node.
As a result, the input hypergraph is divided into (i) hypergraphs containing exactly one hyperedge and (ii) hypergraphs that contain rings.
The type-(i) hypergraphs are obtained from branching structures of the input hypergraph, and the type-(ii) ones from rings.

The second step \emph{rips off} the hyperedges from the type-(ii) hypergraphs; \ie, we divide a node containing a ring into one node containing all of the nodes and other nodes each of which contains each of the hyperedges.
This operation is helpful to reduce the number of production rules.
The inferred grammar can generate a ring structure by first generating a skeleton of a ring, where all of the hyperedges are non-terminal, and then replacing the non-terminals in an arbitrary way.
Without the second step, since the number of possible atom configurations of a ring is enormous, that of production rules also greatly increases.

\noindent
\textbf{Theoretical Result.} 
Theorem~\ref{thm:hrg-dual-repr} summarizes the properties of our algorithm.
It suggests that 
(i) HRG inferred by our algorithm can generate the whole input hypergraphs and 
(ii) the inferred HRG always generates a molecular hypergraph, \ie, the HRG is MHG.
See Appendix~\ref{sec:proofs} for its proof.

\begin{thm}
\label{thm:hrg-dual-repr}
 Let $\cH(L_H^{(V)}, L_H^{(E)}, c^{(E)})$ be a set of all molecular hypergraphs and $\hat{\cH}$ be its finite subset.
 Let $\cL_{\hrg}(\hat{\cH})$ be the language generated by HRG inferred by applying our algorithm to $\hat{\cH}$. 
 Then, $\cL_{\hrg}(\hat{\cH})$ satisfies 
\begin{align*}
 \hat{\cH}\subseteq\cL_{\hrg}(\hat{\cH})\subseteq \cH(L_H^{(V)}, L_H^{(E)}, c^{(E)}).
\end{align*}
\end{thm}

\section{Application to Molecular Optimization}
So far, we have presented $(\enc_H, \dec_H)$ , a pair of encoder and decoder between a molecular graph and a molecular hypergraph~(Sec.~\ref{sec:encod-decod-betw}),
 and $(\enc_G, \dec_G)$, that between a molecular hypergraph and a parse sequence according to MHG~(Sec.~\ref{sec:molec-hypergr-gramm}).
We have also provided an algorithm to obtain $(\enc_G, \dec_G)$ from a set of molecular hypergraphs~(Sec.~\ref{sec:mhg-infer-algor}).
In this section, we finally present an application of our encoders and decoders to molecular optimization, which aims to search a molecule with desirable properties.

\subsection{Our Model}
Our model consists of encoder and decoder, $(\enc, \dec)$, and a predictive model from the latent space to a target value.
Our encoder and decoder are composed as: 
\begin{align*}
\enc &= \enc_N \circ \enc_G \circ \enc_H,\\
\dec &=\dec_H \circ \dec_G \circ \dec_N,
\end{align*}
where $(\enc_N, \dec_N)$ is a seq2seq GVAE~\cite{kusner2017}.
Since GVAE can output a parse sequence that follows a context-free grammar, $\dec_N$ is guaranteed to output a valid parse sequence that belongs to $\enc_G[\cL_\mhg]$, the domain of the following decoder $\dec_G$.
Our model configuration appears in Appendix~\ref{sec:model-config}.

\subsection{Molecular Optimization Algorithm}
Global molecular optimization~(Eq.~\eqref{eq:3}) aims to find novel molecules with desirable properties from the entire molecular space.
It consists of two steps.
Algorithm~\ref{alg:cont-repr} summarizes the former step to obtain latent representations of the input molecules.
Note that this algorithm does not outputs encoders, because molecular optimization algorithms require $\dec$ only.
Algorithm~\ref{alg:global-opt-alg} describes the latter step, which optimizes a molecule in the latent space using BO.

\begin{algorithm}[t]
\caption{Latent Representation Inference}\label{alg:cont-repr}
\textbf{In:} Mols and targets, $\cG_0=\{g_n\}_{n=1}^N$, $\cY_0=\{y_n\}_{n=1}^N$.

\begin{algorithmic}[1]
 \STATE Obtain molecular hypergraphs: $\cH_0\leftarrow\enc_H(\cG_0)$
 \STATE Obtain an MHG as well as parse sequences:\\ $\dec_G, \mhg, \cS_0\leftarrow \mhg\mathchar`-\mathsf{INF}(\cH_0)$.
 \STATE Train NNs using $(\cS_0,\cY_0)$ to obtain $(\enc_N,\dec_N)$, $\hat{f}$.
 \STATE Obtain latent vectors, $\cZ_0\leftarrow\bbE[\enc_N[\cS_0]]$.
\end{algorithmic}
 \textbf{return} latent vectors $\cZ_0$, $\dec_G$, $\dec_N$, $\hat{f}$.
\end{algorithm}
\begin{algorithm}[t]
\caption{Global Molecular Optimization}\label{alg:global-opt-alg}
 \textbf{In:} $\cZ_0$, $\cY_0$, $\dec$, \#iterations $K$, \#candidates $M$.

 \begin{algorithmic}[1]
  \STATE $\cD_1 \leftarrow \{(z_n,y_n)\}_{n=1}^N$
  \FOR{$k=1,\dots,K$}
  \STATE Fit GP using $\cD_k$.
  \STATE Obtain candidates $\cZ_k=\{z_m\in\bbR^D\}_{m=1}^M$ from BO.
  \STATE Obtain molecular graphs as $\cG_k\leftarrow\dec[\cZ_k]$.
  \STATE Obtain target values as $\cY_k\leftarrow f[\cG_k]$.
  \STATE $\cD_{k+1}\leftarrow \cD_k \cup \{(z_{k,m}, y_{k,m})\in\cZ_k\times\cY_k\}_{m=1}^M$.
  \ENDFOR
 \end{algorithmic}
 \textbf{return} novel molecules $\{g_{k,m}\}_{m=1,k=1}^{M,K}$.
\end{algorithm}

\section{Related Work}\label{sec:related-work}
%
Molecular optimization has a longstanding history especially in drug discovery, and mostly combinatorial methods have been used to generate novel molecules~\cite{Jorgensen:2009aa}.
The paper by \citet{bombarelli2016}, for the first time, applies modern machine learning techniques to this problem, and since then, growing number of papers have been tackling this problem.

There are two complementary approaches to molecular optimization.
One is the combination of VAE and Bayesian optimization, originally developed by \citet{bombarelli2016}, and the other is reinforcement learning, where the construction of a molecule is modeled as a Markov decision process~\cite{guimaraes2017,you2018,zhou2018}.
These two approaches cover complementary application areas, and therefore, they are not conflicting with each other.
The key difference is the assumption on function evaluation cost.
The former assumes that the cost is so high that the number of function evaluations should be kept at a minimum,
while the latter assumes that the cost is negligible and a number of trial-and-errors are allowed.
Therefore, the former is more favorable when the evaluation requires wet-lab experiments or computationally heavy first-principles calculation, and the latter is more favorable when the evaluation can be carried out by light-weight computer simulation.
Since our paper focuses on the former setting, in the following, we will discuss its literature.

As stated in the introduction, the decoding error issue has been one of the critical issues, and therefore, this section focuses on a series of studies alleviating it.
There are mainly two approaches to address it.
One approach is to devise the decoding network to generate as valid SMILES strings as possible.
For example, \citet{kusner2017} leverage SMILES' grammar to force the decoder to align the grammar.
This approach is limited because they assume a context-free grammar~(CFG), while SMILES' grammar is not totally context-free.
\citet{dai2018} propose to use an attribute grammar, which enhances CFG by introducing attributes and rules.
This enhancement allows us to enforce semantic constraints to the decoder.
However, they deal only with the ring-bond matching and valence conditions, and therefore, their decoder sometimes fails to generate valid molecules.

Another approach is to substitute another molecular representation for SMILES so that the output is guaranteed to be valid.
As far as we know, only the very recent paper by \citet{jin2018} takes this approach.
They represent a molecule by fragments connected in a tree structure.
While their work for the first time reports 100\% validity of decoded molecules, their method requires multiple neural networks other than VAE and the predictor.
Our work further pushes along this direction by formalizing the tree representation in terms of HRG.
This formalization allows us to model atom-level connections between fragments along with the stereochemistry information, and we realize 100\% validity using a single VAE and the predictor.

\section{Empirical Studies}
\begin{table*}[t]
\centering
\caption{Reconstruction rate, predictive performance, and global molecular optimization with the unlimited oracle. GCPN, the RL-based method, outperforms VAE-based ones when the target evaluation cost is negligible~(though, which is \emph{not} our focus).}\label{tab:exp-res}
 {
 \begin{tabular}[t]{cccccccccc}
\toprule
  \multirow{2}{*}{Method} & \multirow{2}{*}{\% Reconst.} & \multirow{2}{*}{Valid prior}  & \multirow{2}{*}{Log likelihood} & \multirow{2}{*}{RMSE} & \multicolumn{5}{c}{Unlimited oracle case}\\
                          &                              &                               &                                 &                       & 1st  & 2nd  & 3rd  & 50th & Top 50 Avg. \\
  \cmidrule(lr){1-1}\cmidrule(lr){2-3}\cmidrule(lr){4-5}\cmidrule(lr){6-10}
  CVAE & 44.6\% & 0.7\% & $-1.812 \pm 0.004$ & $1.504 \pm 0.006$ & 1.98 & 1.42 & 1.19 & -- & --\\
  GVAE & 53.7\% & 7.2\% & $-1.739 \pm 0.004$ & $1.404 \pm 0.006$ & 2.94 & 2.89 & 2.80 & -- & --\\
  SD-VAE & 76.2\% & 43.5\% & $-1.697 \pm 0.015$ & $1.366 \pm 0.023$ & 4.04 & 3.50 & 2.96 & -- & --\\
  JT-VAE & 76.7\% & {100\%} & $-1.658 \pm 0.023$ & $1.290 \pm 0.026$ & 5.30 & 4.93 & 4.49 & 3.48 & 3.93\\
  GCPN & -- & -- & -- & -- & 7.98 & 7.85 & 7.80 & -- & --\\\midrule
  \textbf{Ours} & {94.8\%} & {100\%} & ${-1.323 \pm 0.003}$  & $ 0.959 \pm 0.002$ & 5.56 & 5.40 & 5.34 & 4.12 & 4.49\\
\bottomrule
 \end{tabular}
}
\end{table*}

\begin{table}[t]
\centering
\caption{Global molecular optimization with the limited oracle. Our method outperforms the others including GCPN.}\label{tab:exp-res-global-limited}
 {
 \begin{tabular}[t]{cccccc}
\toprule
  \multirow{2}{*}{Method} & \multicolumn{5}{c}{Limited oracle case} \\ 
                          & 1st  & 2nd  & 3rd  & 50th & Top 50 Avg. \\\midrule
  JT-VAE                  & 1.69 & 1.68 & 1.60 & -9.93& -1.33       \\
  GCPN                    & 2.77 & 2.73 & 2.34 & 0.91 & 1.36        \\\midrule
  \textbf{Ours}           & 5.24 & 5.06 & 4.91 & 4.25 & 4.53        \\
\bottomrule
 \end{tabular}
}
\end{table}

We evaluate the effectiveness of MHG in the molecular optimization domain.
In particular, we are interested in the case when unlabeled molecules are abundant but the number of function evaluations is limited due to its cost.
We basically follow \citeauthor{jin2018}'s experimental procedures, and the baseline results are copied from the existing papers~\cite{jin2018,you2018} when appropriate.

\noindent
\textbf{Purposes and Baselines.}
As explained in Section~\ref{sec:related-work}, there are two complementary approaches to molecular optimization: VAE-based and RL-based ones.
The purposes of the empirical studies are to answer the following research questions:
\textbf{(Q1)}~Do RL-based approaches perform better than VAE-based ones when the number of function evaluations is \emph{not} limited~(as reported in \citeauthor{you2018}'s paper)?
\textbf{(Q2)}~Do VAE-based approaches generally work better than RL-based ones when the number of function evaluations is limited?
\textbf{(Q3)}~Does MHG-VAE outperform the existing VAE-based methods?
\textbf{(Q2)} and \textbf{(Q3)} are our primary concerns.
To this end, as baseline methods, we employ CVAE~\cite{bombarelli2016}, GVAE~\cite{kusner2017}, SD-VAE~\cite{dai2018}, and JT-VAE~\cite{jin2018} as VAE-based approaches, and GCPN~\cite{you2018} as an RL-based approach.\footnote{The algorithm by \citet{zhou2018} is not used because the open-sourced program consumed more than 96GB memory and was unable to reproduce the result in our environment.}
For their details, see Section~\ref{sec:related-work}.

\noindent
\textbf{Dataset.}
We use the ZINC dataset following the existing work.
This dataset is extracted from the ZINC database~\cite{sterling2015} and contains 220,011 molecules for training, 24,445 for validation, and 5,000 for testing.
For basic statistics of MHG inferred using this dataset, please refer to Appendix~\ref{sec:mhg-stat}.
For the target chemical property to be maximized, we employ a standardized penalized logP following the existing work:
\begin{align}
\label{eq:4} f(m) = \widehat{\mathrm{logP}}(m) - \widehat{\mathrm{SA}}(m) - \widehat{\mathrm{cycle}}(m),
\end{align}
where $\mathrm{logP}$ is the octanol-water partition coefficient, $\mathrm{SA}$ is the synthetic accessibility score,
and $\mathrm{cycle}$ is the size of the longest ring subtracted by six~(if its size is less than six, the function returns 0), and the hat represents that the function is standardized using the values calculated on the training set.
For example, letting $\mu_{\mathrm{logP}}$ and $\sigma_{\mathrm{logP}}$ be the sample mean and standard deviation of $\mathrm{logP}$ calculated using the training set, $\widehat{\mathrm{logP}}(m) = (\mathrm{logP}(m) - \mu_{\mathrm{logP}})/\sigma_{\mathrm{logP}}$.

\subsection{Reconstruction Rate}
We first investigate the quality of $\enc$ and $\dec$ of the VAE-based methods, which has much influence on the performance of molecular optimization.

\noindent
\textbf{Protocol.}
For each molecular graph $m$ in the test set, we obtain its reconstruction as $m^\prime=\dec(\enc(m))$.
If $m$ and $m^\prime$ are isomorphic, we regard the reconstruction succeeds.
We repeat the above procedure using all of the test molecules 100 times, and report the mean reconstruction success rate.

To investigate the quality of the latent space, we evaluate the success rate of decoding random latent vectors.
We sample $z$ from $\cN(0,I)$ and decode it to obtain $m=\dec(z)$. If $m$ is valid, we regard the decode succeeds.
We repeat this procedure 1,000 times and report the success rate.

\noindent
\textbf{Result.}
According to Table~\ref{tab:exp-res}~(left),
our method clearly improves the reconstruction rate, which justifies our molecular modeling approach.

\subsection{Global Molecular Optimization}\label{sec:glob-molec-optim}
We then investigate the performance on global molecular optimization using two scenarios.
The first scenario assumes that the function evaluation cost is negligible and the algorithms can query an arbitrary number of target properties, which is used to answer \textbf{(Q1)} and \textbf{(Q3)}.
This scenario is dubbed as the \emph{unlimited oracle case}.
All of the existing studies assume this scenario.
The second one assumes that it is expensive and only a limited number of oracle calls are allowed, which is used to answer \textbf{(Q2)} and \textbf{(Q3)}.
This scenario is dubbed as the \emph{limited oracle case}.
Since our problem setting assumes the expensive case, the second scenario is of primary interest, and the first scenario is examined for completeness.

\noindent
\textbf{Protocol 1 (unlimited oracle case).}
For our method, we first obtain latent representations by Algorithm~\ref{alg:cont-repr}.
Then, we apply PCA to the latent vectors to obtain $40$-dimensional latent representations.
Then, we run Algorithm~\ref{alg:global-opt-alg} with $M=50$, $K=5$.
As a result, we obtain 250 novel molecules.
We repeat this procedure ten times, resulting in 2,500 novel molecules.
We report the log-likelihood and root mean-squared error~(RMSE) of GP evaluated on the test set,
top three molecule property scores out of 2,500, and the mean of top 50 target properties.

For the baseline methods, we simply copied the results from \citeauthor{you2018}'s paper, where the VAE-based methods use the same protocol as ours.
Note that the number of queries by GCPN, the RL-based method, is larger than those of VAE-based methods.
In fact, the default configuration of GCPN requires $5\times 10^7$ steps.

\noindent
\textbf{Protocol 2 (limited oracle case).}
In this scenario, we compare all of the methods under the same number of queries.
For our method and JT-VAE, we initialize GP with $N=250$ labeled molecules randomly selected from the training set, and run Algorithm~\ref{alg:global-opt-alg} with $M=1$, $K=250$.
We use GPyOpt~\cite{gpyopt2016} as the BO module.
For GCPN, we run the algorithm and regard the first 500 molecules as the output.
For each method, this procedure is repeated ten times, and obtain 2,500 novel molecules for the VAE-based methods and 5,000 novel molecules for GCPN.
We report top three molecule property scores out of the novel molecules and statistics of top 50 target properties.
Since BO is a highly random procedure, examining only top three molecules could lead to unfair comparison, and we suggest to examine statistics of top-$K$ molecules.

\noindent
\textbf{Result (the unlimited oracle case).}
Table~\ref{tab:exp-res}~(mid) shows the predictive performance of GP.
Our method achieves better scores than the others, indicating that our latent space well encodes features necessary to predict the property.

Table~\ref{tab:exp-res}~(right) reports the top three target properties as well as the minimum and average scores of top 50 molecules obtained via Protocol~1.
Ours achieves the best scores among the VAE-based methods, and GCPN reports molecules with the highest target properties, which answers~\textbf{(Q1)} and \textbf{(Q3)} in the affirmative.
Furthermore, when we focus on statistics of top 50 molecules, ours achieves better scores than JT-VAE.
These results suggest that our method is more likely to discover better molecules than the other VAE-based ones, which also supports us to answer \textbf{(Q3)} in the affirmative.

\noindent
\textbf{Result (limited oracle case).}
Table~\ref{tab:exp-res-global-limited} shows the result obtained via Protocol~2.
Our method clearly outperforms the other methods including GCPN, answering \textbf{(Q2)} in the affirmative.
It is notable that our method in the limited oracle case performs almost comparably to our method in the unlimited oracle case.
This result supports the effectiveness of our method especially in the limited oracle case.

 \section{Conclusion and Future Work}
We have developed the molecular hypergraph grammar variational autoencoder~(MHG-VAE).
Our key idea is to employ MHG to represent a molecular graph as a parse tree, which is fed into VAE.
Since MHG models the atom-level connections as well as the stereochemistry information, MHG-VAE can learn a pair of encoder and decoder using a single VAE.
The highlights of our experiments include
(i) MHG-VAE achieves the best performance among VAE-based methods and
(ii) MHG-VAE performs better than the state-of-the-art RL-based method called GCPN when the number of target function evaluations is limited.

A future research direction will be to optimize MHG with respect to some goodness criteria such as the minimum description length~\cite{jonyer2004} or a Bayesian criterion~\cite{Chen:1995:BGI:981658.981689}.
Since the resultant MHG depends on tree decompositions, we can optimize MHG by manipulating tree decompositions.

Another direction involves retrosynthetic analysis, which derives a pathway to synthesize the target molecule.
With this capability, we will be able to immediately examine the property by synthesizing the output molecule.

\section*{Acknowledgments}
This work was supported by JST CREST Grant Number JPMJCR1304, Japan, and JSPS KAKENHI Grant Numbers 15H05711, Japan.
The author would like to thank Dr. Masakazu Ishihata for his helpful discussion.

\newpage

\bibliography{main}
\bibliographystyle{plainnat}

\newpage
\appendix

\section{HRG Inference Algorithm}\label{sec:hrg-infer-algor}
\begin{algorithm}[t]
\caption{Production rule extraction}
 \label{alg:ext-prod-rule}
\textbf{Input: } tree decomposition $T$ and node $v_{T}\in V_T$.\\ 
\textbf{Output: } production rule $p=(A,R)$.\\
\vspace{-5mm}
 \begin{algorithmic}[1]
  \STATE Find $\pa(v_T)$ and $\ch(v_T)$.
  \IF{$\pa(v_{T})$ does not exist}
  \STATE Set $A$ as the starting hyperedge $S$.
  \ELSE
  \STATE Set $A=\ell_T^{(V)}(v_T)\cap \ell_T^{(V)}(\pa(v_T))$.
  \ENDIF
  \STATE Set $R=H(v_T)$, where the nodes shared with $A$ are set to be the external nodes, $\ext(R)$.
  \IF{$\ch(v_T)$ is not empty}
  \FOR{each child $v_T^\prime\in\ch(v_T)$}
  \STATE Add a non-terminal hyperedge to $R$, which consists of nodes $\ell_T^{(V)}(v_T) \cap \ell_T^{(V)}(v_T^\prime)$ with the non-terminal label.
  \ENDFOR
  \ENDIF
  \STATE \textbf{return } $(A,R)$
 \end{algorithmic}
\end{algorithm}
\begin{figure}[t]
\centering
\begin{minipage}[t]{.45\hsize}
\centering
 \includegraphics[width=.8\hsize]{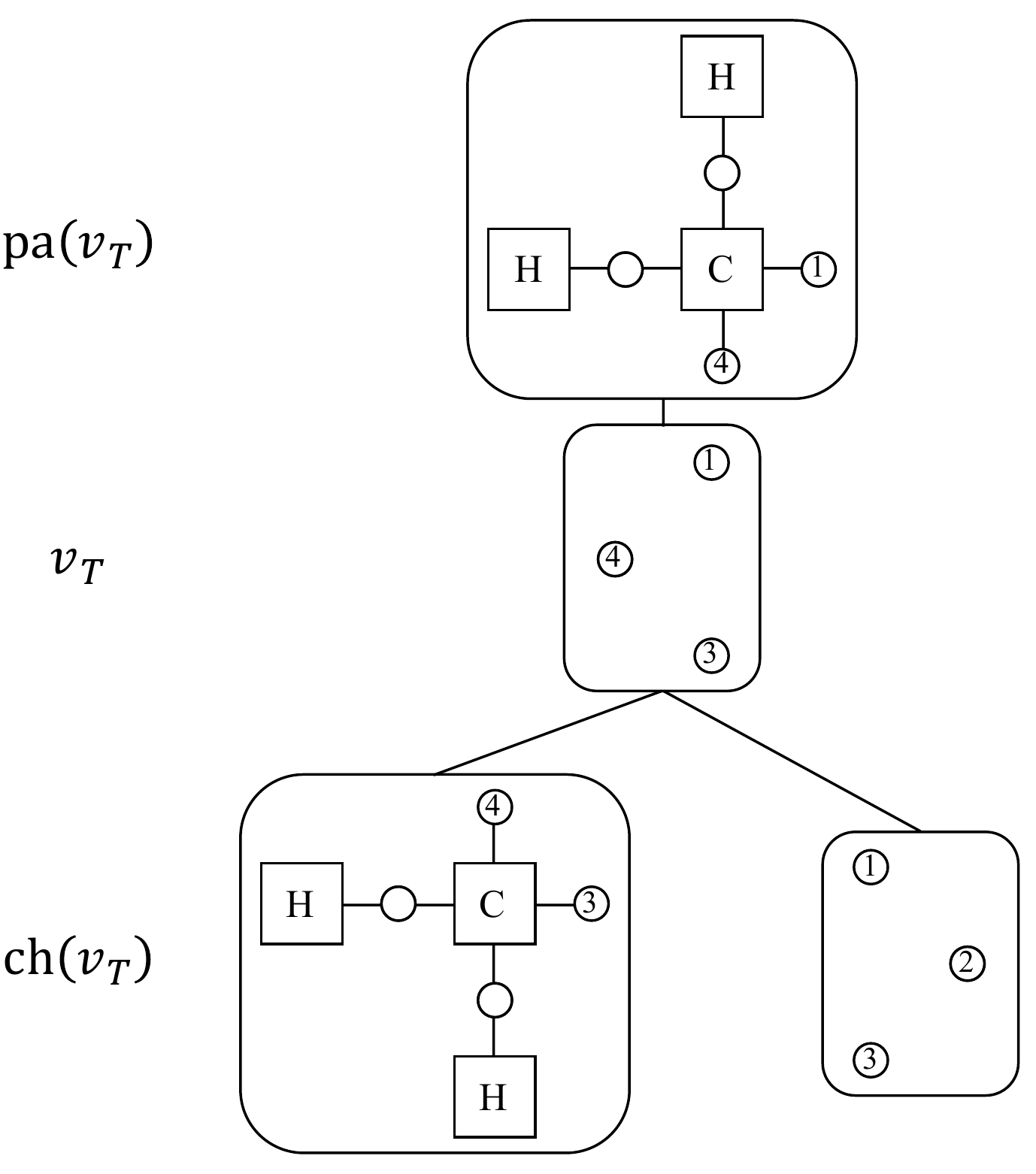}
 \subcaption{Triplet of $v_T$, $\pa(v_T)$, and $\ch(v_T)$ chosen from the tree decomposition in Figure~\ref{fig:tree-decomp}.}\label{fig:triplet}
\end{minipage}
\hspace{.05\hsize}
\begin{minipage}[t]{.45\hsize}
 \centering
 \includegraphics[width=\hsize]{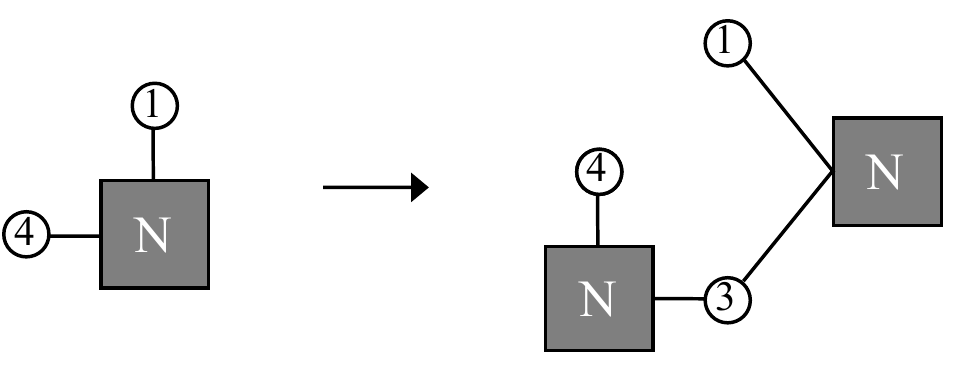}
 \subcaption{Production rule extracted from the triplet in Figure~\ref{fig:triplet}. Nodes IDs are added for explanation, and in the algorithm, they are removed except for the correspondence between the nodes in LHS and the external nodes in RHS.}\label{fig:extracted_prod_rule}
\end{minipage}
\caption{Illustration of Algorithm~\ref{alg:ext-prod-rule}.}
\end{figure}
This section describes the details of the algorithm proposed by \citet{aguinaga2016}.

The input of the algorithm is a set of hypergraphs, $\hat{\cH}=\{H_1,\dots,H_N\}$, and its output is a hyperedge replacement grammar $\cG$ whose language includes the input hypergraphs.
The algorithm extracts production rules from each input hypergraph. Let $H\in\hat{\cH}$ be any input hypergraph.
It first computes a tree decomposition of $H$, which we denote by $T$, and picks an arbitrary node of $T$ as its root.
Then, for each node $v_T\in V_T$, it applies Algorithm~\ref{alg:ext-prod-rule} to extract a production rule.
Algorithm~\ref{alg:ext-prod-rule} chooses the triplet of $v_T$ and its parent and children as shown in Fig.~\ref{fig:triplet}, and extracts a production rule that glues $H(v_T)$ to $H(\pa(v_T))$ with non-terminal hyperedges for gluing each child.
Figure~\ref{fig:extracted_prod_rule} illustrates the production rule extracted from the triplet shown in Fig.~\ref{fig:triplet}.
After applying Algorithm~\ref{alg:ext-prod-rule} to all of the nodes and all of the input hypergraphs, it removes duplicated production rules and outputs the set of production rules.

\section{Proofs}\label{sec:proofs}
This section provides a proof of Theorem~\ref{thm:hrg-dual-repr}.
Throughout this section, a \emph{node} refers to a hypergraph's node, and a \emph{tree node} refers to a tree decomposition's node.

To prove Theorem~\ref{thm:hrg-dual-repr}, we only have to prove that $\cL_{\hrg}(\hat{\cH})\subseteq \cH(L_H^{(V)}, L_H^{(E)}, c^{(E)})$ holds, because $\hat{\cH}\subseteq\cL_{\hrg}(\hat{\cH})$ has been proven by \citet{aguinaga2016}.

$\cL_{\hrg}(\hat{\cH})\subseteq \cH(L_H^{(V)}, L_H^{(E)}, c^{(E)})$ can be proven by combining Lemmata~\ref{lemma:2-reg-sat-cond}, \ref{lemma:2-reg}, and \ref{lemma:card-cons-pres}.
Lemmata~\ref{lemma:2-reg-sat-cond} and \ref{lemma:2-reg} guarantee that our algorithm preserves the first condition of a molecular hypergraph, and
Lemma~\ref{lemma:card-cons-pres} guarantees that our algorithm preserves the second condition of a molecular hypergraph.

\begin{cond}
\label{cond:rhs}
 For each production rule $p=(A,R)$, the degree of external nodes of $R$ is one, and the degree of internal nodes of $R$ is two.
\end{cond}

\begin{lemma}
\label{lemma:2-reg-sat-cond}
 HRG inferred by applying our algorithm to a set of $2$-regular hypergraphs satisfies Condition~\ref{cond:rhs}.
\end{lemma}

\begin{lemma}
\label{lemma:2-reg}
 If HRG satisfies Condition~\ref{cond:rhs}, then it always generates a $2$-regular hypergraph.
\end{lemma}

\begin{lemma}
\label{lemma:card-cons-pres}
 HRG inferred by applying our algorithm to a set of cardinality-consistent hypergraphs always generates a cardinality-consistent hypergraph.
\end{lemma}

\begin{proof}[Proof of Lemma~\ref{lemma:2-reg-sat-cond}]
 Let $H$ be an arbitrary input hypergraph, and $(T, \ell_T^{(V)}, \ell_T^{(E)})$ be its irredundant tree decomposition.
 Since $H$ is $2$-regular, for each $v_H\in V_H$, $T[V_T(v_H)]$ is a single tree node that contains both of the two hyperedges~(Fig.~\ref{fig:case1}, \textbf{Case 1}),
 or $T[V_T(v_H)]$ is a path where each of the leaf tree nodes contains one of two hyperedges adjacent to $v_H$~(Fig.~\ref{fig:case23}, \textbf{Cases 2, 3}).
 It is sufficient to prove that for each $v_H\in V_H$ and for each $v_T\in T[V_T(v_H)]$, the production rule extracted by running Algorithm~\ref{alg:ext-prod-rule} satisfies Condition~\ref{cond:rhs}.
 In the following, we fix $v_H\in V_H$ to be an arbitrary node and $e_{H,1}, e_{H,2}\in E_H$ to be the hyperedges incident with $v_H$.
 
\begin{figure*}[t]
\centering
\begin{minipage}[t]{.29\hsize}
\centering
 \includegraphics[width=.5\hsize]{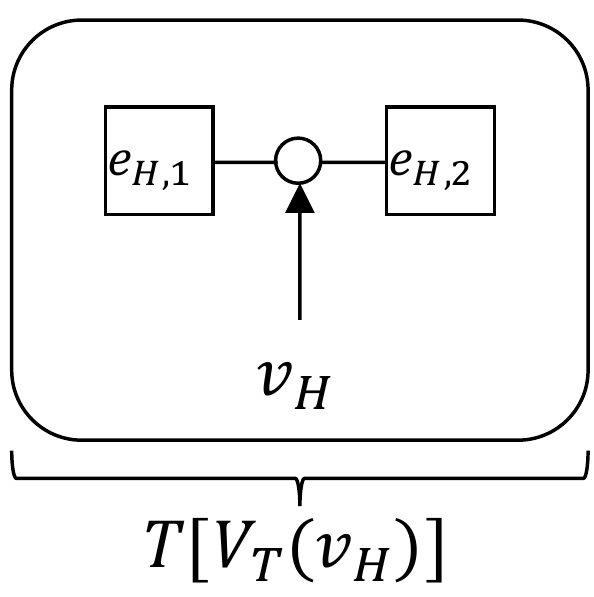}
\subcaption{Case 1: $|V_T(v_H)|=1$.}\label{fig:case1}
\end{minipage}
 \begin{minipage}[t]{.6\hsize}
\centering
  \includegraphics[width=\hsize]{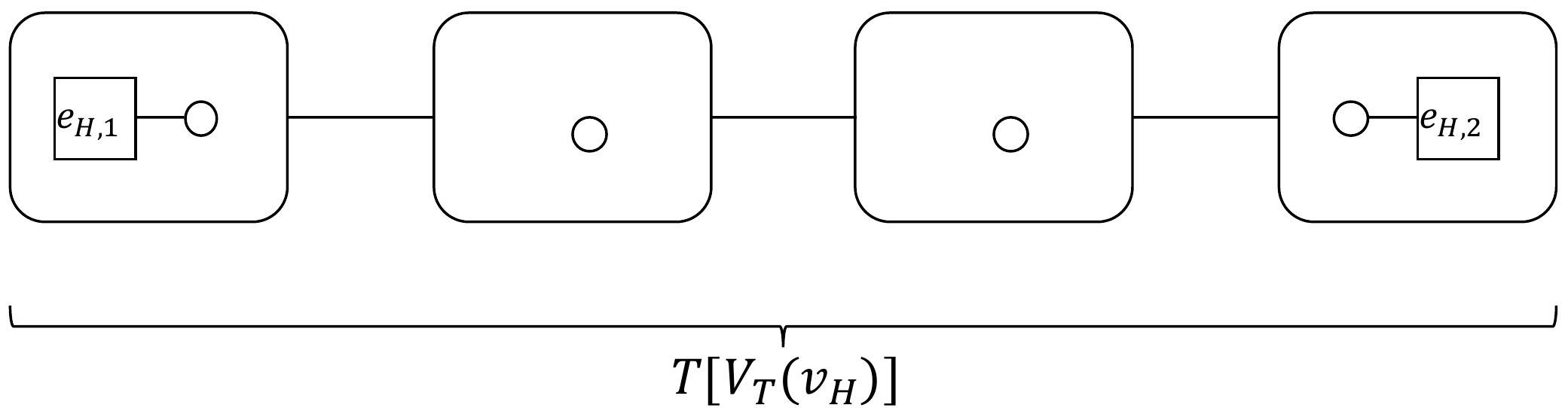}
  \subcaption{Cases 2, 3: $|V_T(v_H)|\geq 2$.}\label{fig:case23}
 \end{minipage}
\caption{Illustrations of $T[V_T(v_H)]$ in two cases.}
\end{figure*}
 
 \paragraph{Case 1.} $|V_T(v_H)|=1$\\
 First, let us assume $|V_T(v_H)|=1$ and let $v_T^\ast$ be the only node in $V_T(v_H)$, \ie, $V_T(v_H)=\{v_T^\ast\}$.
 In this case, $v_H$ cannot be an external node, because $\pa(v_T^\ast)$ does not contain $v_H$.
 Thus, $v_H$ must be an internal node with $e_{H,1}$ and $e_{H,2}$ when it appears in $R$.
 In addition, no non-terminal hyperedge will be incident with $v_H$ in $R$, because none of $\ch(v_T^\ast)$ contains $v_H$ in it.
 Therefore, in this case, $v_H$ is an internal node whose degree equals two in $R$.

\paragraph{Case 2.} $|V_T(v_H)|\geq 2$ and $v_H$ is an internal node in $R$\\
 Such a production rule is made from a triplet $(v_T, \pa(v_T), \ch(v_T))$ such that $v_H\in \ell_T^{(V)}(v_T)$ and $v_H\notin \ell_T^{(V)}(\pa(v_T))$ hold.
 
 If $e_{H,1}\in\ell_{T}^{(E)}(v_T)$ or $e_{H,2}\in\ell_{T}^{(E)}(v_T)$ holds, there exists exactly one node $v_T^{(\ch)}\in\ch(v_T)$ such that $v_H\in\ell_T^{(V)}(v_T^{(\ch)})$ holds.
 In this case, exactly one non-terminal hyperedge will be connected to $v_H$, and therefore, the degree of $v_H$ equals two in $R$.
 
 If $e_{H,1}\notin\ell_{T}^{(E)}(v_T)$ and $e_{H,2}\notin\ell_{T}^{(E)}(v_T)$ hold, there exist exactly two nodes $v_{T,1}^{(\ch)},v_{T,2}^{(\ch)}\in\ch(v_T)$ such that $v_H\in\ell_T^{(V)}(v_{T,1}^{(\ch)})$ and $v_H\in\ell_T^{(V)}(v_{T,2}^{(\ch)})$ hold.
 In this case, exactly two non-terminal hyperedges will be connected to $v_H$, and therefore, the degree of $v_H$ equals two in $R$.
 
 \paragraph{Case 3.} $|V_T(v_H)|\geq 2$ and $v_H$ is an external node in $R$\\
 Such a production rule is made from a triplet $(v_T, \pa(v_T), \ch(v_T))$ such that $v_H\in \ell_T^{(V)}(v_T)$ and $v_H\in \ell_T^{(V)}(\pa(v_T))$ hold.
 
 If $e_{H,1}\in\ell_{T}^{(E)}(v_T)$ or $e_{H,2}\in\ell_{T}^{(E)}(v_T)$ holds, for each node $v_T^{(\ch)}\in\ch(v_T)$, $v_H\notin\ell_T^{(V)}(v_T^{(\ch)})$ holds.
 In this case, no non-terminal hyperedge will be connected to $v_H$, and therefore, the degree of $v_H$ equals one in $R$.
 
 If $e_{H,1}\notin\ell_{T}^{(E)}(v_T)$ and $e_{H,2}\notin\ell_{T}^{(E)}(v_T)$ hold, there exists exactly one node $v_T^{(\ch)}\in\ch(v_T)$ such that $v_H\in\ell_T^{(V)}(v_T^{(\ch)})$ holds.
 In this case, exactly one non-terminal hyperedge will be connected to $v_H$, and therefore, the degree of $v_H$ equals one in $R$.

Therefore, for any case, Condition~\ref{cond:rhs} holds, and therefore, Lemma~\ref{lemma:2-reg-sat-cond} has been proven.
\end{proof}

\begin{proof}[Proof of Lemma~\ref{lemma:2-reg}]
 Let $H$ be an arbitrary hypergraph that can be derived from the starting symbol $S$. Note that $H$ may contain non-terminal hyperedges.
 It is sufficient to prove that $H$ is $2$-regular, if the production rules satisfy Condition~\ref{cond:rhs}.
 
 If $H$ is directly derivable from the starting symbol $S$, then it is $2$-regular, because Condition~\ref{cond:rhs} guarantees that for any production rule $p=(S,R)$, $R$ is $2$-regular.

 If $H$ is directly derivable from $H^\prime$ by applying $p=(A,R)$, and if $H^\prime$ is $2$-regular and derivable from $S$, then there exists a non-terminal hyperedge in $e_{H^\prime}\in E_{H^\prime}$ that is labeled as $A$ and is replaced with $R$ to yield $H$.
For each node $v_{H^\prime}\in V_{H\prime}\backslash e_{H^\prime}$, the replacement does not change the degree, and the degree of $v_{H^\prime}$ equals two in $H$.
For each node $v_{R}\in V_{R}\backslash \ext(R)$\footnote{$\ext(R)$ denotes the set of external nodes in $R$}, the replacement does not change the degree, and the degree of $v_R$ equals two in $H$.
For each node $v_{H\prime}\in e_{H^\prime}$, the replacement first deletes $e_{H^\prime}$, and then, connects the hyperedges adjacent to the external nodes $\ext(R)$.
Since the degrees of the external nodes are one, the replacement does not change its degree. The same applies to each external node.
 Since the degree of each node of $H$ is two, $H$ is proven to be $2$-regular.
\end{proof}

\begin{proof}[Proof of Lemma~\ref{lemma:card-cons-pres}]
 Let $T$ be a tree decomposition of a molecular hypergraph~$H$.
 For each $v_T\in V_T$, all of the hyperedges in $\ell_T^{(E)}(v_T)$ are cardinally-consistent, which is guaranteed by the second condition in Def.~\ref{def:tree-decomp}.
 Therefore, letting the production rule extracted from $v_T$ be $p=(A,R)$, all of the hyperedges in $R$ are cardinally-consistent.
 Since applying such a production rule preserves the cardinality consistency, this lemma has been proven.
\end{proof}

\section{Model Configuration}\label{sec:model-config}
We tune the model configuration using the reconstruction rate on the validation set.
Both $\enc_N$ and $\dec_N$ use three-layer GRU~\cite{kyunghyun2014} with $384$ hidden units~($\enc_N$ is bidirectional), handling
a sequence of production rule embeddings in $128$-dimensional space.
In $\enc_N$, the output of GRU is fed into a linear layer to compute the mean and log variance of a $72$-dimensional Gaussian distribution, 
and the latent vector $z\in\bbR^{72}$ is sampled from it as the output of $\enc_N$.
The encoder and decoder are trained by optimizing the objective function of $\beta$-VAE~\cite{higgins2017} with $\beta=0.01$ using ADAM~\cite{kingma2015} with initial learning rate $5\times 10^{-4}$.
As a predictive model $\hat{f}\colon\bbR^{72}\rightarrow\bbR$,
we employ a linear regression. Whenever target values are available, we jointly train seq2seq VAE and the predictive model.

\section{Basic Statistics of MHG}\label{sec:mhg-stat}
This section reports basic statistics of MHG inferred using the ZINC dataset.

From the ZINC dataset, our algorithm obtains 2,031 production rules, of which 1,424 are starting rules.
Each molecule is associated with 25 production rules on average.
While the grammar seems to be huge, 2/3 of the starting rules (1,073 rules) are used by less than ten molecules, and in total, 2,355 out of 250k molecules are using these starting rules.
If ignoring these rules, our grammar has moderate size.

We also investigate the coverage rate of the language associated with MHG~(\ie, how many molecules out of all possible molecules can MHG represent?).
While we could not provide any theoretical validation on the coverage rate,
we instead approximately estimate the coverage rate by the number of molecules in the test set that cannot be parsed using the grammar inferred using the training set.
As a result, 16 out of 5,000 molecules cannot be parsed, and thus, the coverage rate will be 99.68\%, and we believe this coverage is sufficiently high.

\section{Local Molecular Optimization}
\begin{table*}[t]
\caption{Results on local molecular optimization.}\label{tab:exp-res-local}
   \begin{tabular}[t]{ccccccccccccc}
\toprule
 &  \multicolumn{4}{c}{Improvement} & \multicolumn{4}{c}{Similarity} & \multicolumn{4}{c}{Success} \\\cmidrule(lr){2-5}\cmidrule(lr){6-9}\cmidrule(lr){10-13}
  $\delta$ & $0.0$ & $0.2$ & $0.4$ & $0.6$ & $0.0$ & $0.2$ & $0.4$ & $0.6$ & $0.0$ & $0.2$ & $0.4$ & $0.6$ \\\midrule
  JT-VAE & 
\shortstack{$1.91$ \\ $\pm 2.04$} & \shortstack{$1.68$ \\  $\pm 1.85$} & \shortstack{$0.84 $\\ $\pm 1.45$} & \shortstack{$0.21 $\\ $\pm 0.71$} & 
\shortstack{$0.28 $ \\ $\pm 0.15$} & \shortstack{$0.33 $ \\ $\pm 0.13$} & \shortstack{$0.51 $\\ $\pm 0.10$} & \shortstack{$0.69 $\\ $\pm 0.06$} & 
$97.5 \%$ & $97.1\%$ & $83.6 \%$ & $46.4 \%$\\\midrule
  GCPN   & 
\shortstack{$4.20$ \\ $\pm 1.28$} & \shortstack{$4.12$ \\ $\pm 1.19$}  & \shortstack{$2.49$ \\ $\pm 1.30$}  & \shortstack{$0.79$ \\ $\pm 0.63$}  & 
\shortstack{$0.32$ \\ $\pm 0.12$}  & \shortstack{$0.34$ \\ $\pm 0.11$}  & \shortstack{$0.47$ \\ $\pm 0.08$}  & \shortstack{$0.68$ \\ $\pm 0.08$} & 
$100\%$ & $100 \%$ & $100 \%$ & $100\%$\\\bottomrule

  Ours   & 
\shortstack{$3.28$ \\ $\pm 2.19$} & \shortstack{$2.40$ \\ $\pm 2.16$}  & \shortstack{$1.00$ \\ $\pm 1.87$}  & \shortstack{$0.61$ \\ $\pm 1.20$}  & 
\shortstack{$0.09$ \\ $\pm 0.06$}  & \shortstack{$0.26$ \\ $\pm 0.10$}  & \shortstack{$0.52$ \\ $\pm 0.11$}  & \shortstack{$0.70$ \\ $\pm 0.06$} & 
$100\%$ & $86.3 \%$ & $43.5 \%$ & $17.0\%$\\\bottomrule
 \end{tabular}
\end{table*}

Local molecular optimization~(Algorithm~\ref{alg:local-opt-alg}) aims to improve the property of a given molecule without modifying it too much.
This problem setting is originally proposed by \citet{jin2018} and is formalized as,
\begin{align}
\label{eq:5}
\begin{split}
m^\star &= \dec\left(\argmax_{z\colon\mathrm{sim}(m, \dec(z))\geq \tau} f(\dec(z))\right).
\end{split}
\end{align}
where $\mathrm{sim}(m, m^\prime)$ computes a similarity between molecules $m$ and $m^\prime$, and $\tau$ is a similarity threshold.
We use Tanimoto similarity with Morgan fingerprint~(radius=2).
Problem~\ref{eq:5} is approximately solved by Algorithm~\ref{alg:local-opt-alg},
where we substitute our predictive model $\hat{f}$ for the unknown target function $f$.

Note that the predictive model~$\hat{f}$ requires a large data set for training and it is difficult to apply the limited oracle scenario to local molecular optimization.
Since the unlimited oracle scenario is out of our scope, we leave this topic in the appendix.
We would like to leave a local optimization algorithm tailored for the limited oracle scenario as future work.

\begin{algorithm}[t]
\caption{Local Molecular Optimization}\label{alg:local-opt-alg}
\textbf{In:} Mol. graph $g_0$ and its latent vector $z_0$, $\dec$, $\hat{f}$, step size $\eta$, similarity measure $\mathrm{sim}(\cdot,\cdot)$, threshold $\tau$, \# iterations $K$.

\begin{algorithmic}[1]
 \STATE $g^\star\leftarrow\mathrm{null}$, $y^\star\leftarrow -\infty$ 
 \FOR{$k=1,\dots,K$}
 \STATE $z_k\leftarrow z_{k-1} + \eta\frac{\partial \hat{f}(z_{k-1})}{\partial z}$
 \STATE $g_k\leftarrow \dec(z_k)$, $y_k \leftarrow \hat{f}(z_k)$
 \IF{$\mathrm{sim}(g_0, g_k)\geq\tau$, $g_0\neq g_k$, and $y_k>y^\star$}
 \STATE $g^\star\leftarrow g_k$, $y^\star \leftarrow y_k$ 
 \ENDIF
 \ENDFOR
\end{algorithmic}
 \textbf{return} $(g^\star, y^\star)$
\end{algorithm}

\noindent
\textbf{Protocol.}
As a target property, we use a penalized logP:
\begin{align}
\label{eq:6}
 f(m) = \mathrm{logP}(m) - \mathrm{SA}(m).
\end{align}
We choose $800$ molecules with the lowest penalized logP from the test set.
For each initial molecule~$m$, we run Algorithm~\ref{alg:local-opt-alg} with $\tau\in\{0, 0.2, 0.4, 0.6\}$, $K=80$, and $\eta=0.01$.
For each threshold $\tau$, we report (i) the mean and standard deviation of the target value improvements, (ii) those of the similarity, and (iii) the success rate, 
where Algorithm~\ref{alg:local-opt-alg} succeeds if the output is not null, \ie, if there exists a modified molecule that satisfies the similarity constraint.

\noindent
\textbf{Result.}
Table~\ref{tab:exp-res-local} reports the scores.
First of all, we observe that GCPN outperforms VAE-based methods.
This result is reasonable considering that this task assumes the unlimited oracle scenario.
Among the VAE-based methods, for any similarity threshold, our method improves the target property better than JT-VAE, which demonstrates the effectiveness of our method over JT-VAE.

\section{Molecules Discovered by Global Molecular Optimization}\label{sec:molec-disc-glob}
In the following, we illustrate the top 50 molecules found by global molecular optimization in Section~\ref{sec:glob-molec-optim}.
\begin{figure*}
\begin{minipage}{.24\hsize}
\centering
\includegraphics[width=\hsize]{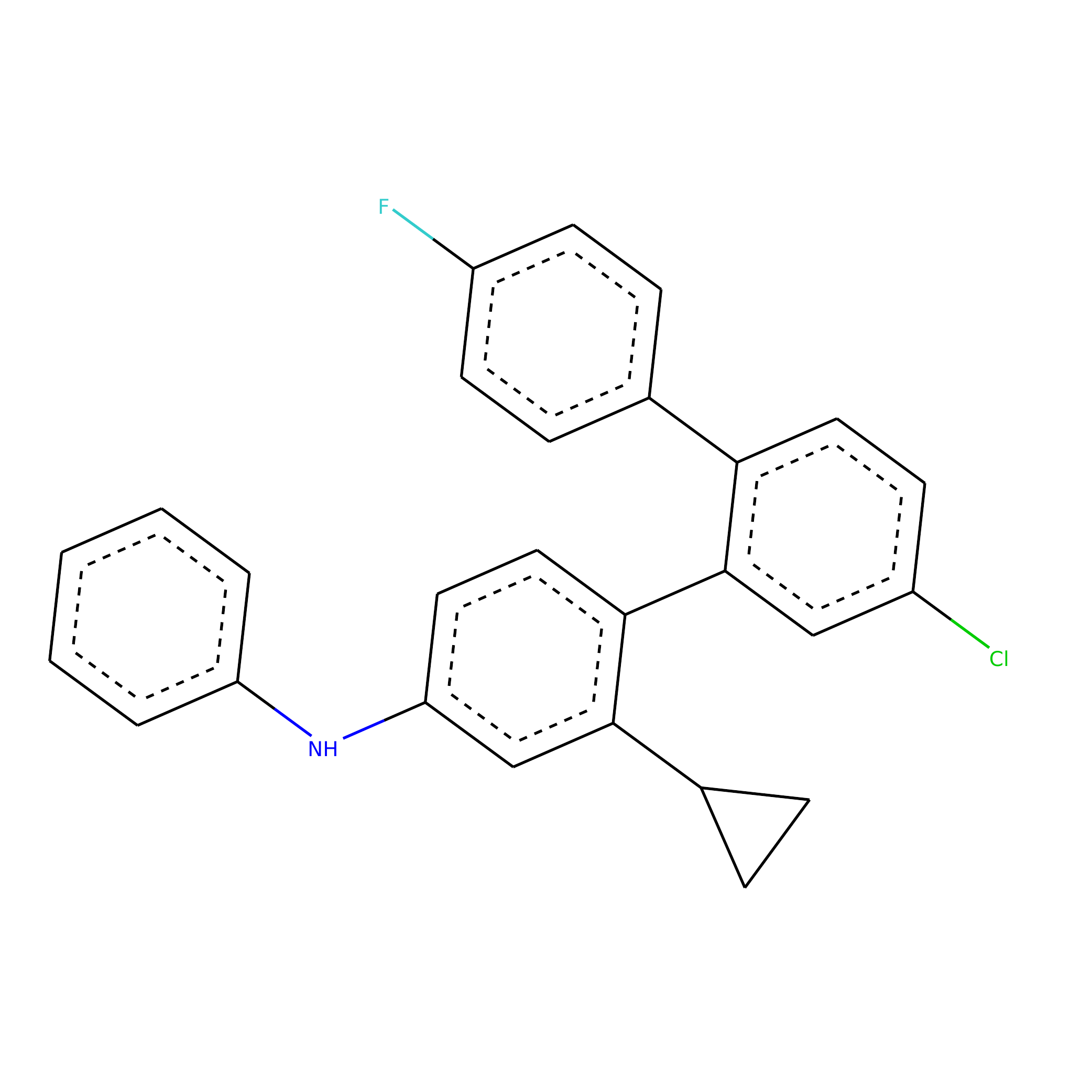}
5.558
\end{minipage}
\begin{minipage}{.24\hsize}
\centering
\includegraphics[width=\hsize]{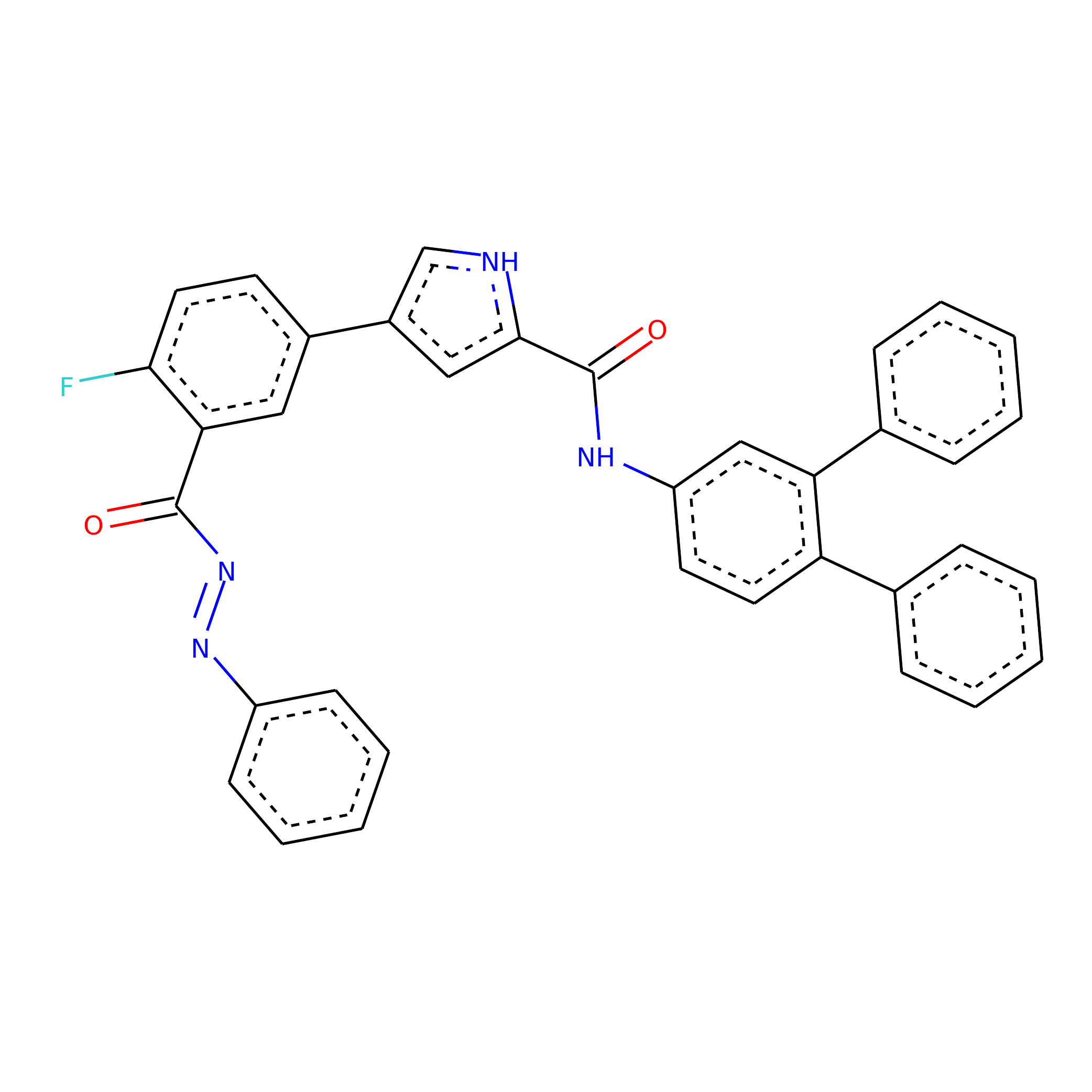}
5.401
\end{minipage}
\begin{minipage}{.24\hsize}
\centering
\includegraphics[width=\hsize]{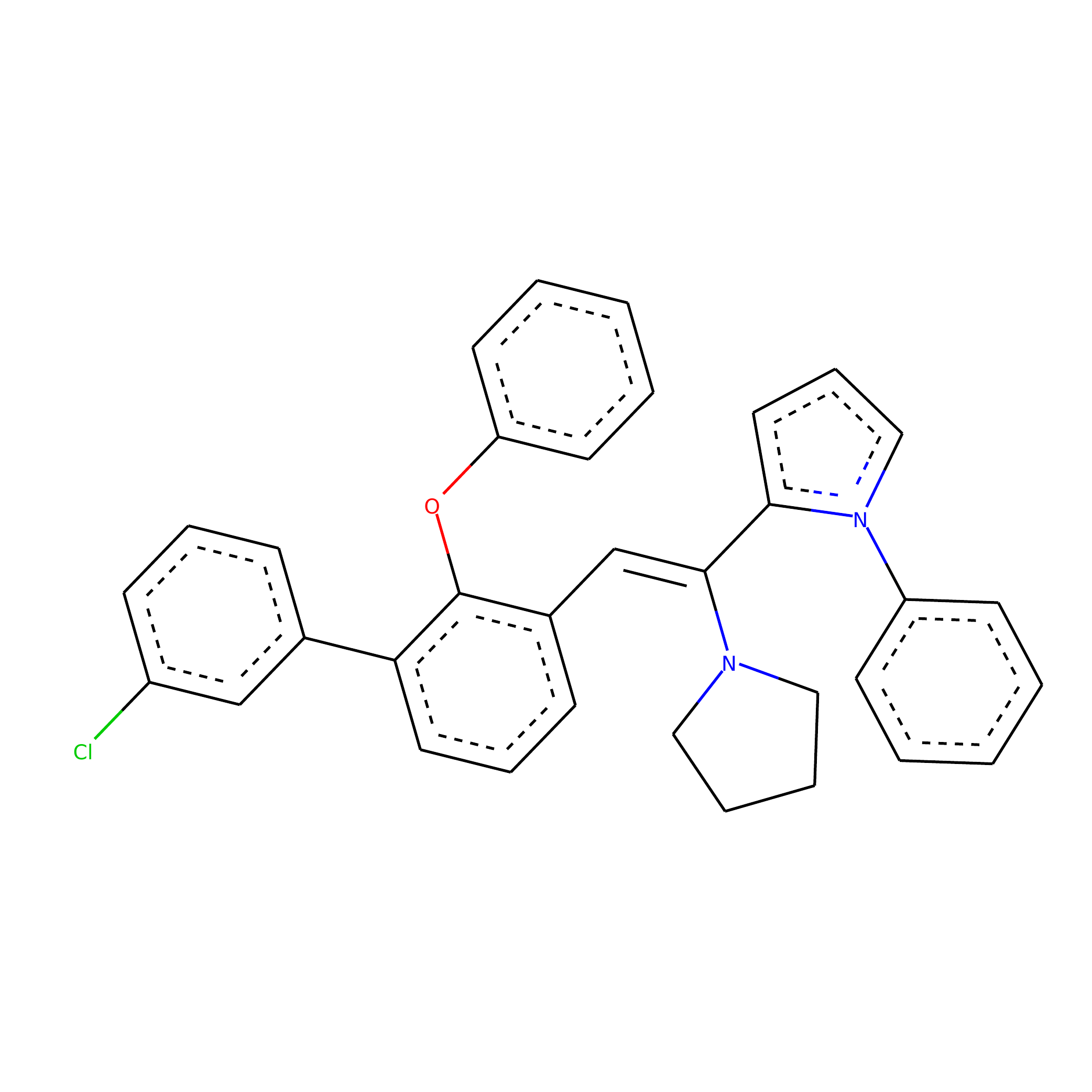}
5.344
\end{minipage}
\begin{minipage}{.24\hsize}
\centering
\includegraphics[width=\hsize]{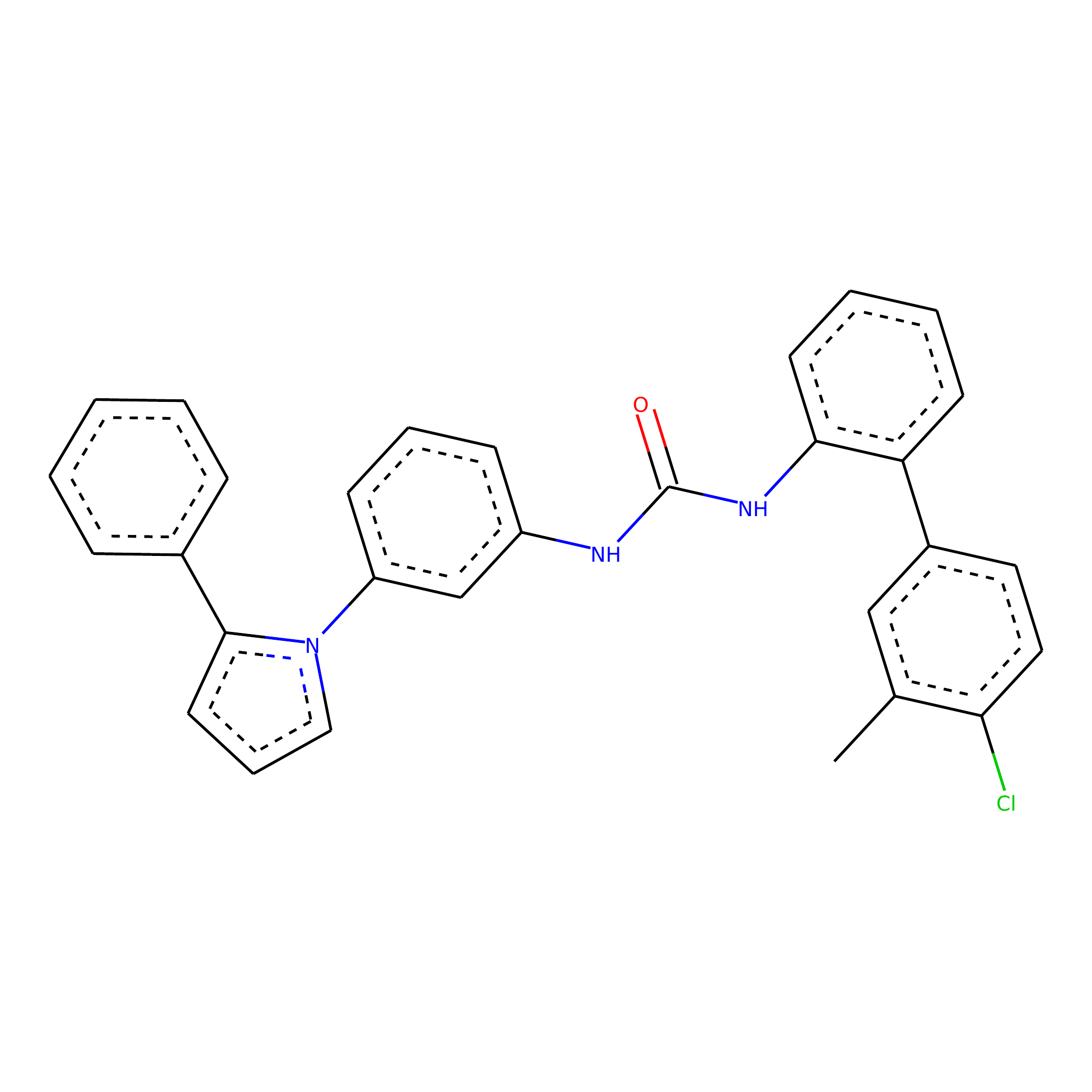}
5.289
\end{minipage}
\begin{minipage}{.24\hsize}
\centering
\includegraphics[width=\hsize]{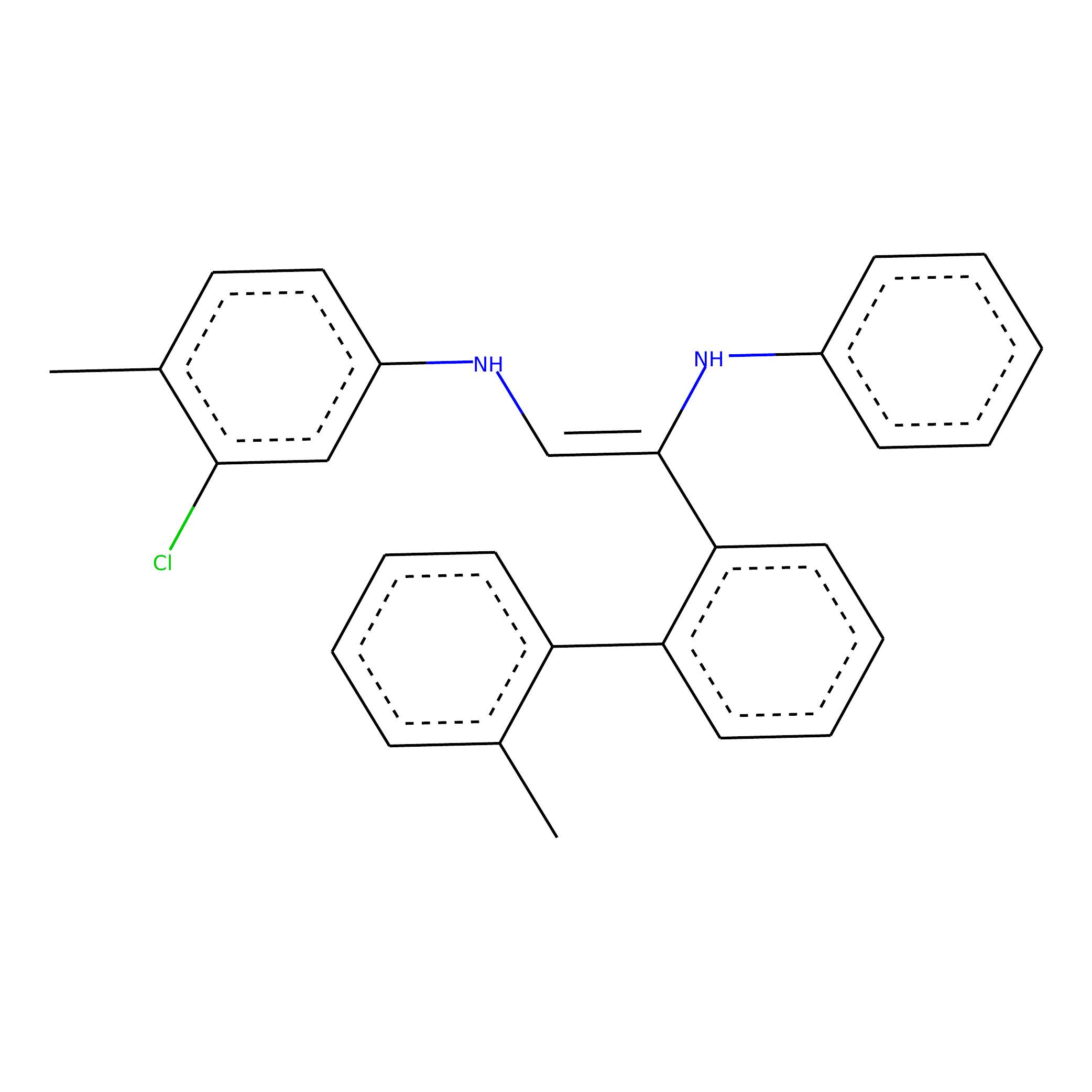}
5.063
\end{minipage}
\begin{minipage}{.24\hsize}
\centering
\includegraphics[width=\hsize]{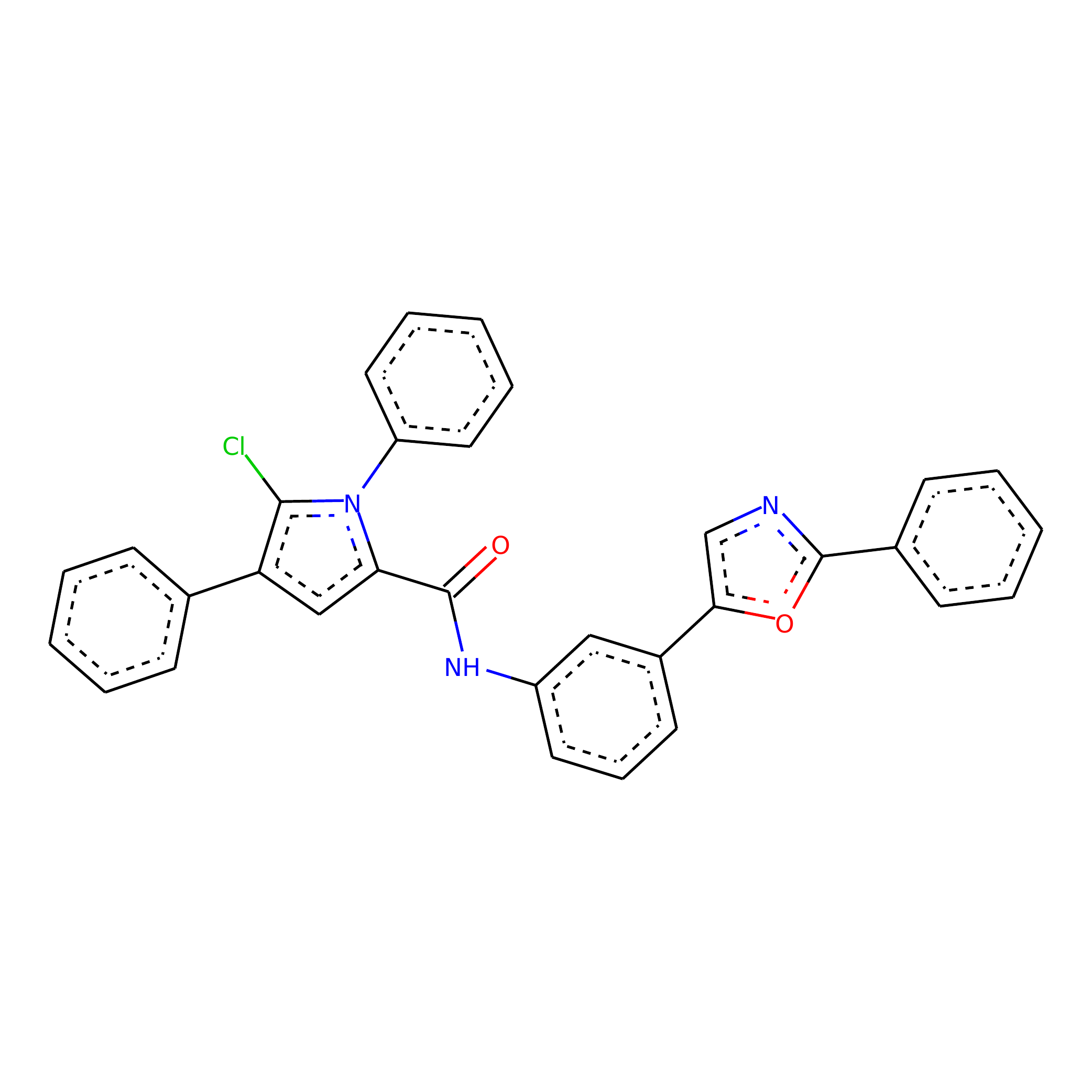}
4.979
\end{minipage}
\begin{minipage}{.24\hsize}
\centering
\includegraphics[width=\hsize]{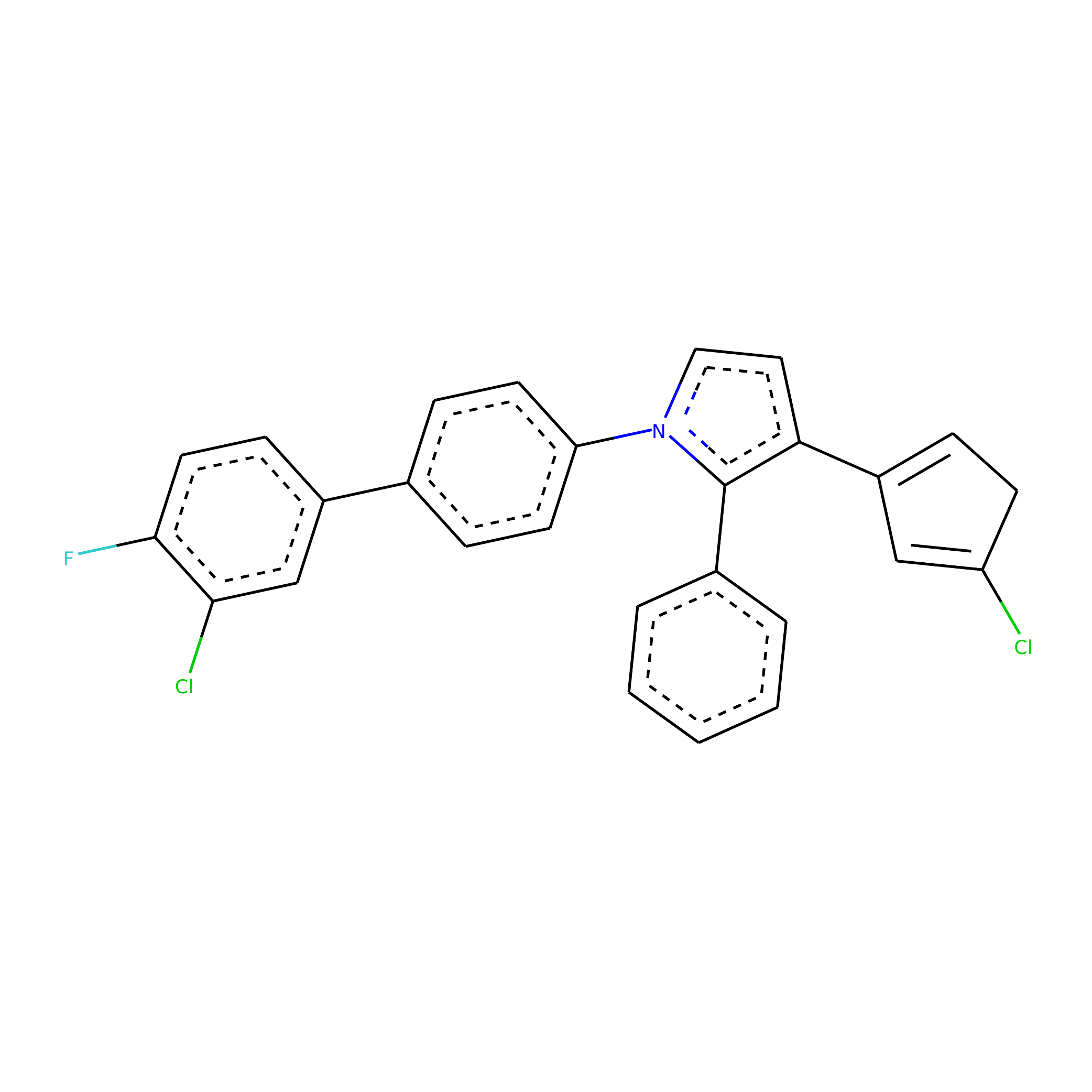}
4.811
\end{minipage}
\begin{minipage}{.24\hsize}
\centering
\includegraphics[width=\hsize]{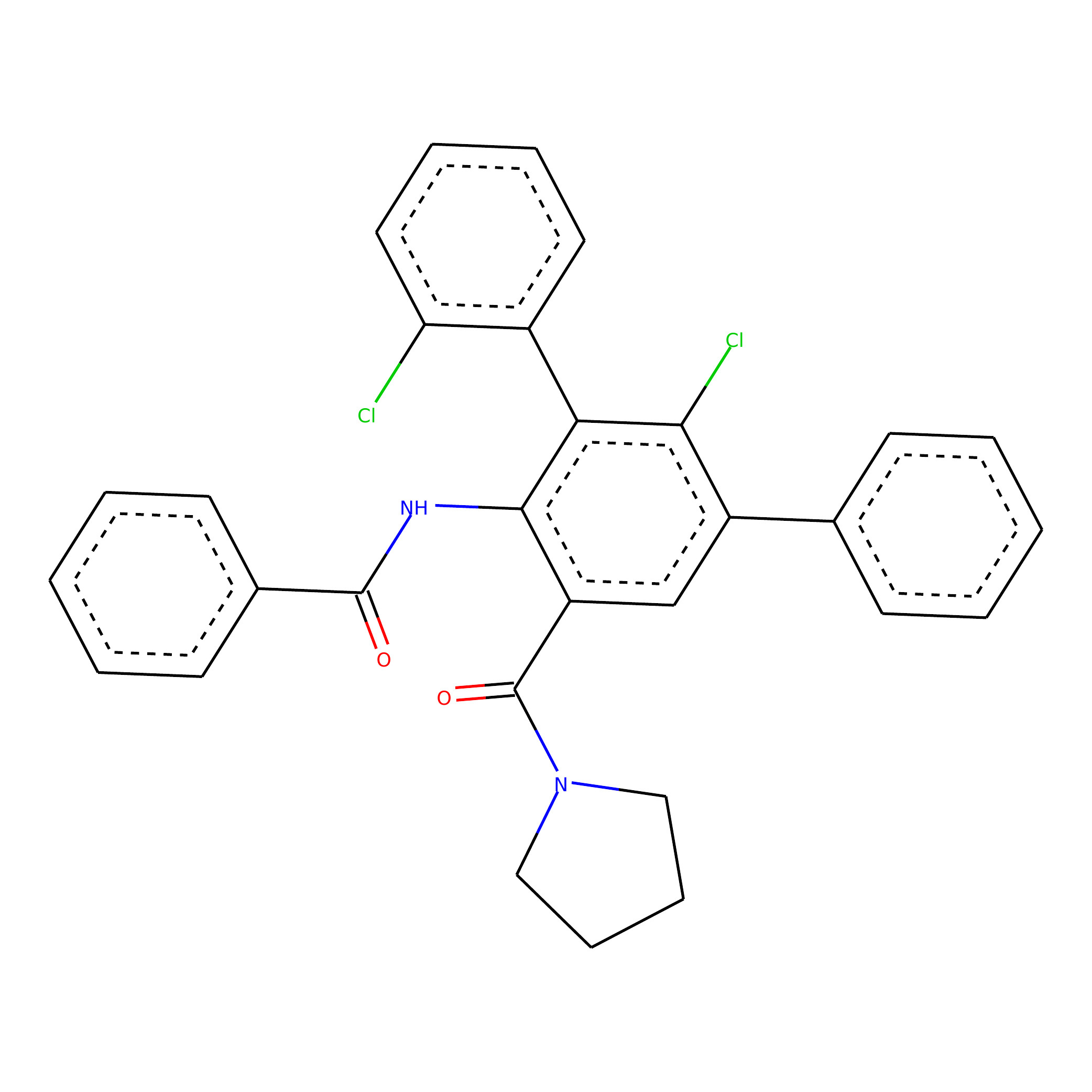}
4.779
\end{minipage}
\begin{minipage}{.24\hsize}
\centering
\includegraphics[width=\hsize]{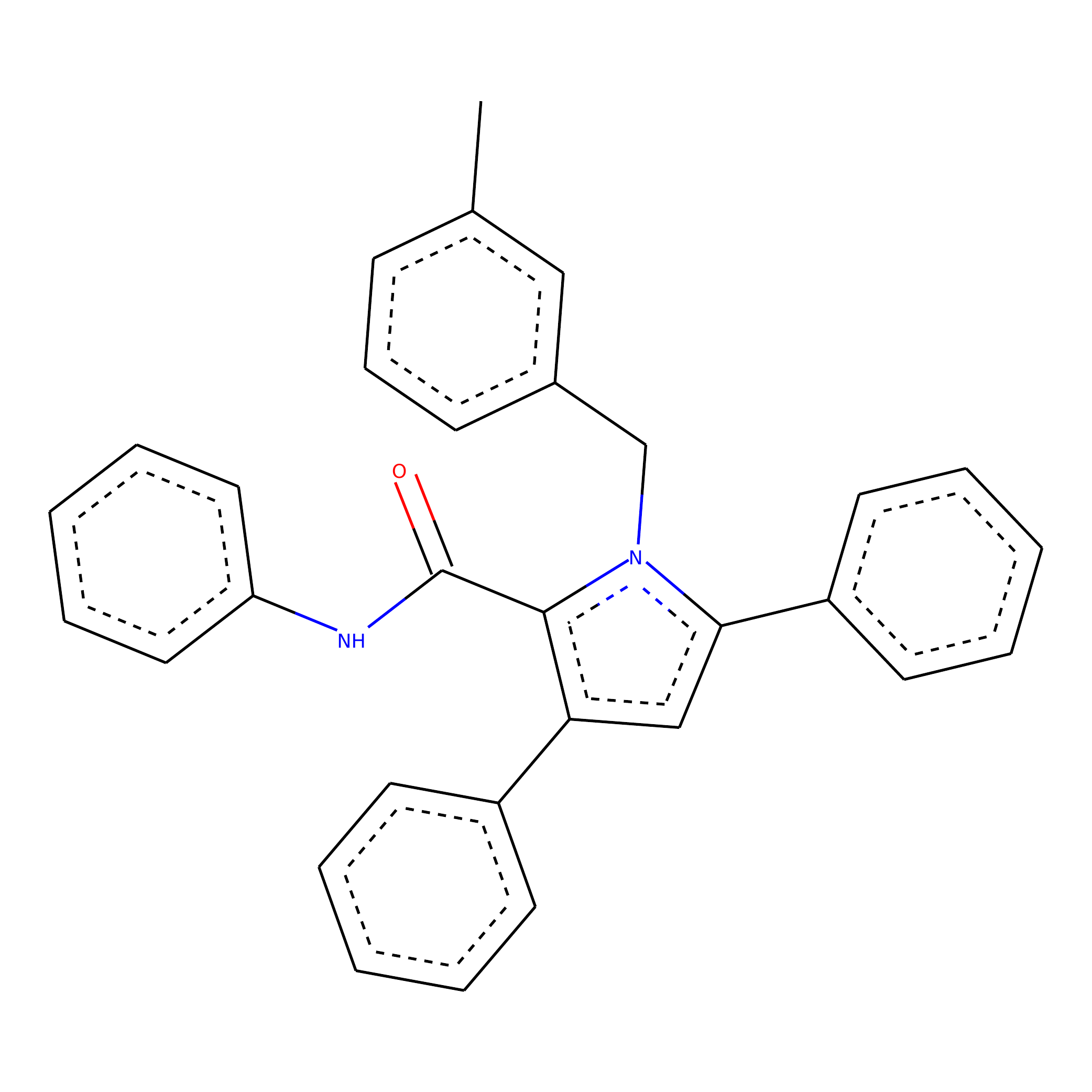}
4.775
\end{minipage}
\begin{minipage}{.24\hsize}
\centering
\includegraphics[width=\hsize]{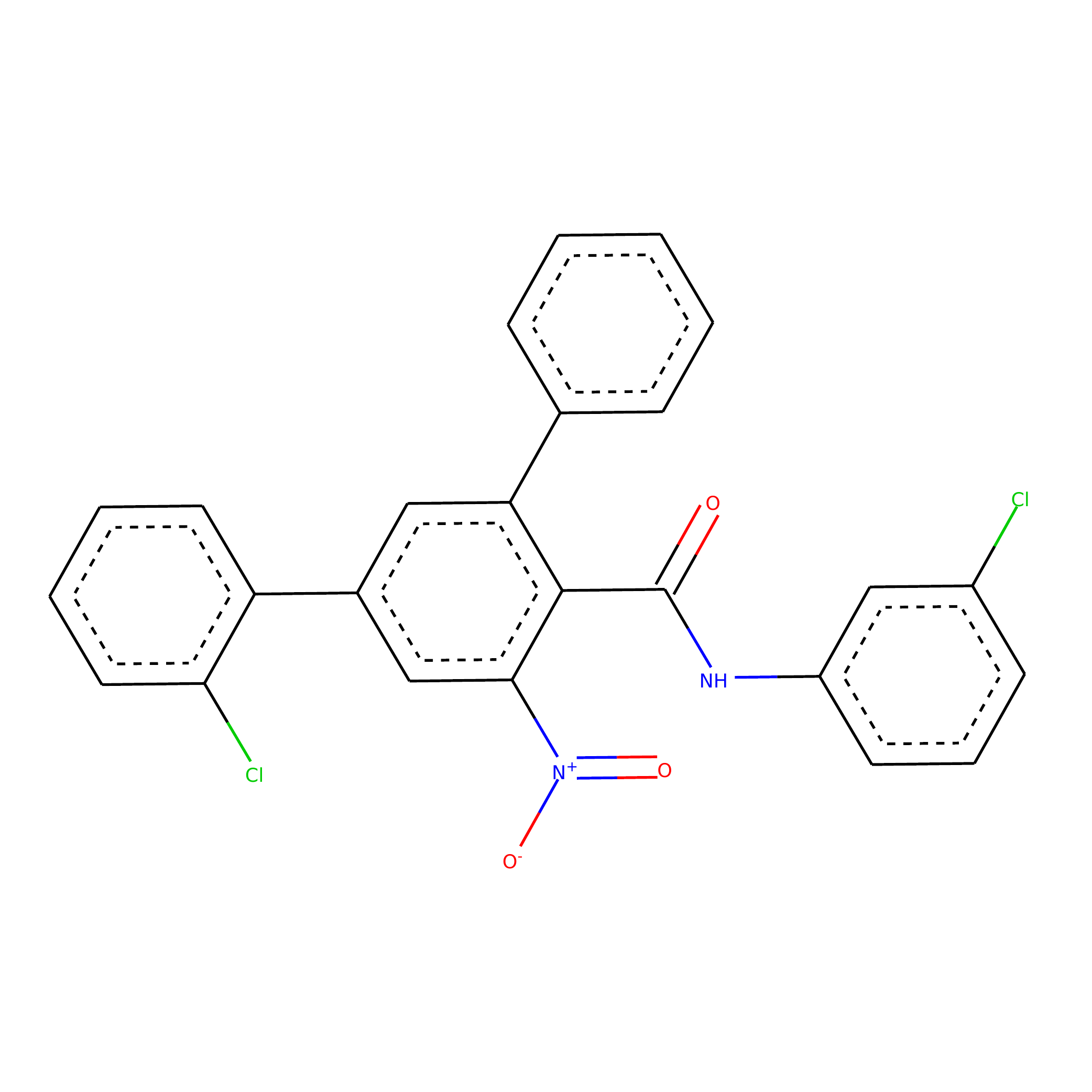}
4.730
\end{minipage}
\begin{minipage}{.24\hsize}
\centering
\includegraphics[width=\hsize]{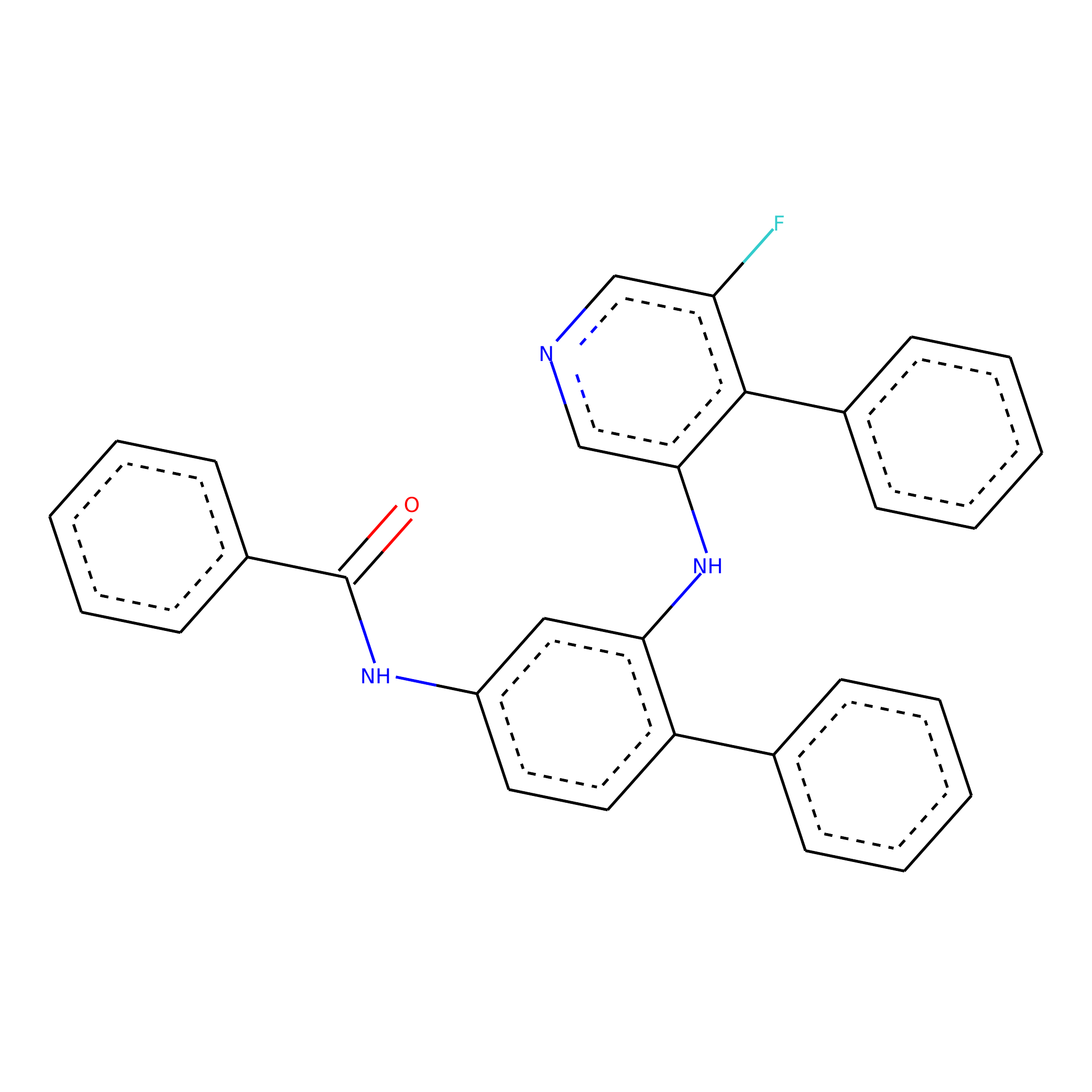}
4.712
\end{minipage}
\begin{minipage}{.24\hsize}
\centering
\includegraphics[width=\hsize]{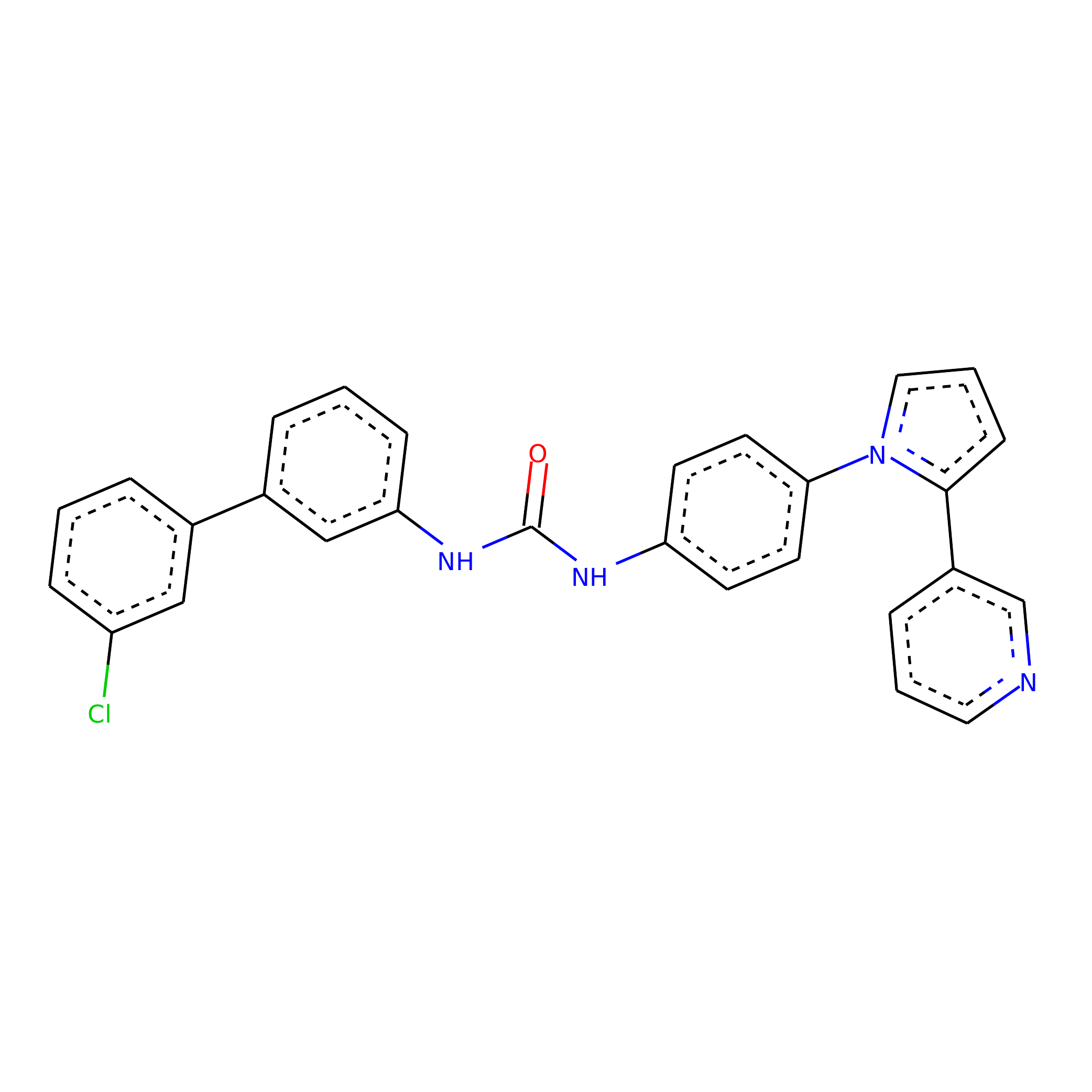}
4.641
\end{minipage}
\begin{minipage}{.24\hsize}
\centering
\includegraphics[width=\hsize]{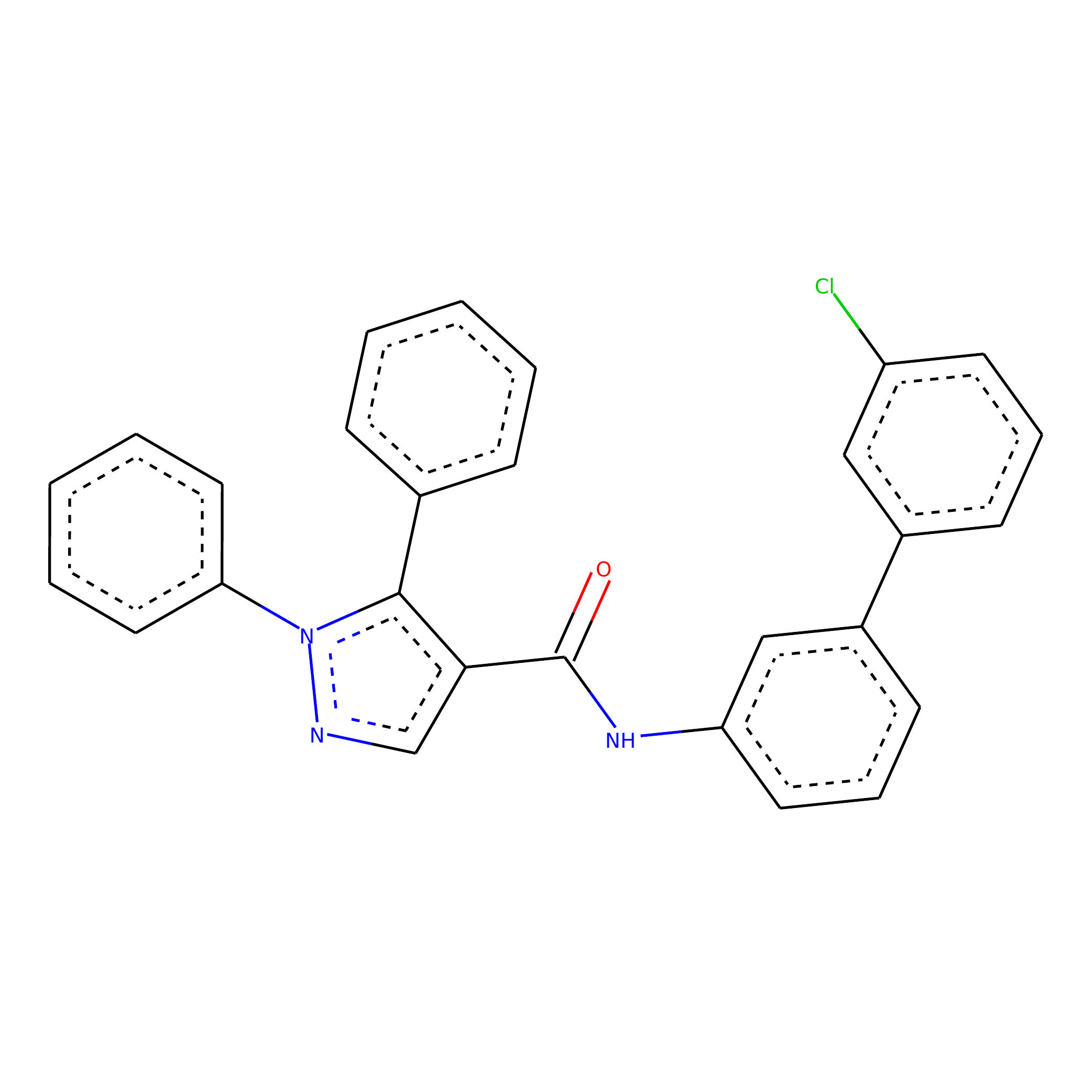}
4.617
\end{minipage}
\begin{minipage}{.24\hsize}
\centering
\includegraphics[width=\hsize]{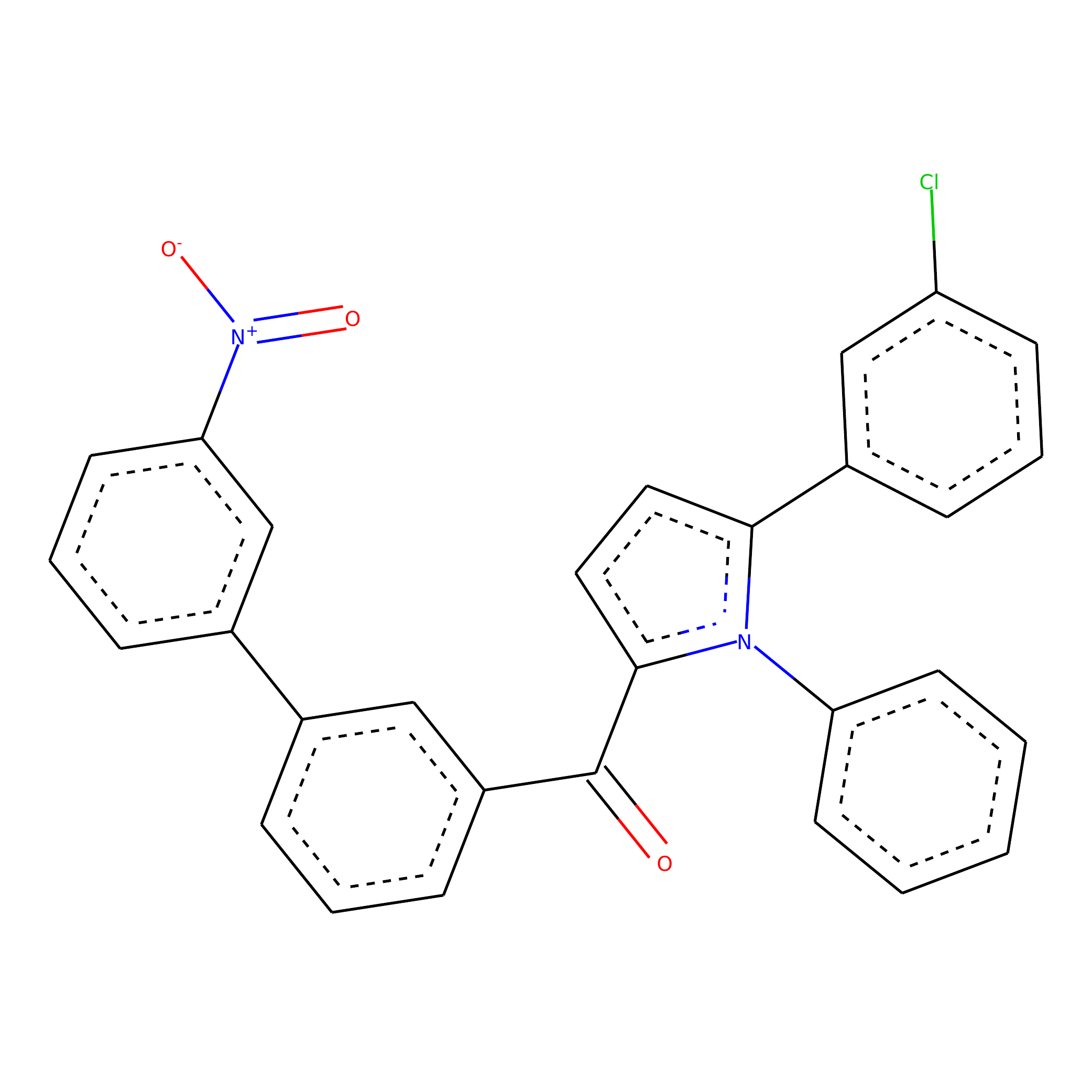}
4.598
\end{minipage}
\begin{minipage}{.24\hsize}
\centering
\includegraphics[width=\hsize]{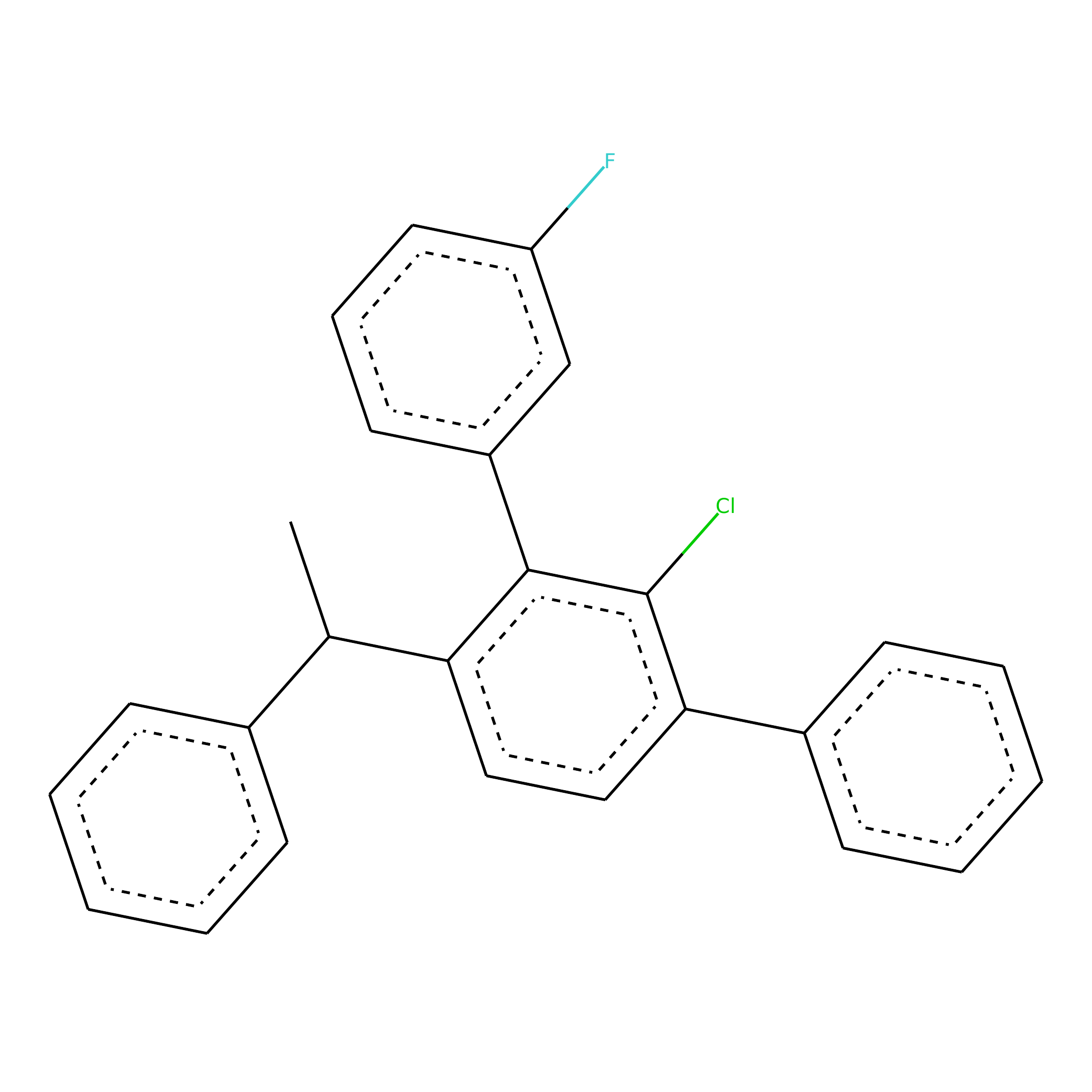}
4.595
\end{minipage}
\begin{minipage}{.24\hsize}
\centering
\includegraphics[width=\hsize]{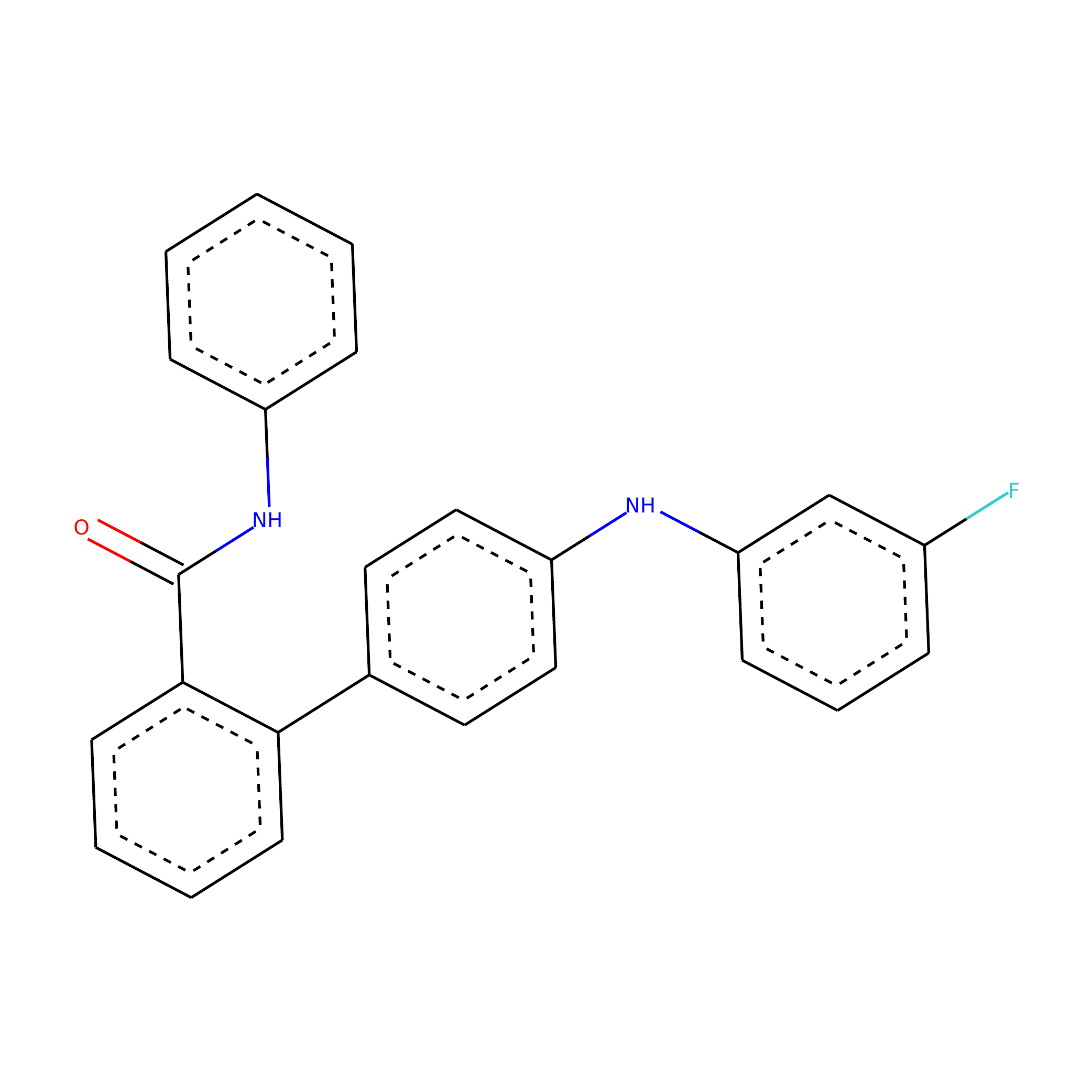}
4.555
\end{minipage}
\begin{minipage}{.24\hsize}
\centering
\includegraphics[width=\hsize]{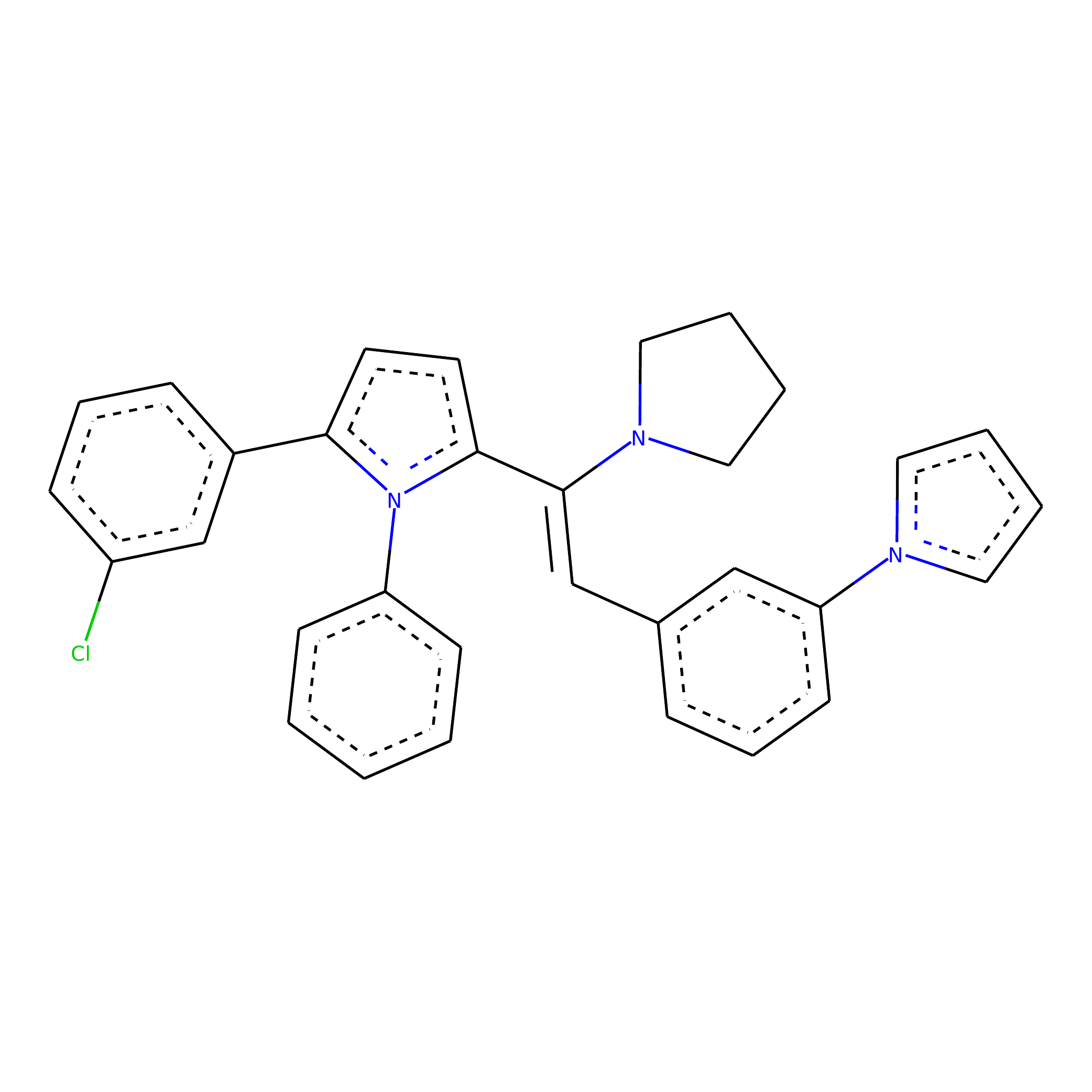}
4.546
\end{minipage}
\begin{minipage}{.24\hsize}
\centering
\includegraphics[width=\hsize]{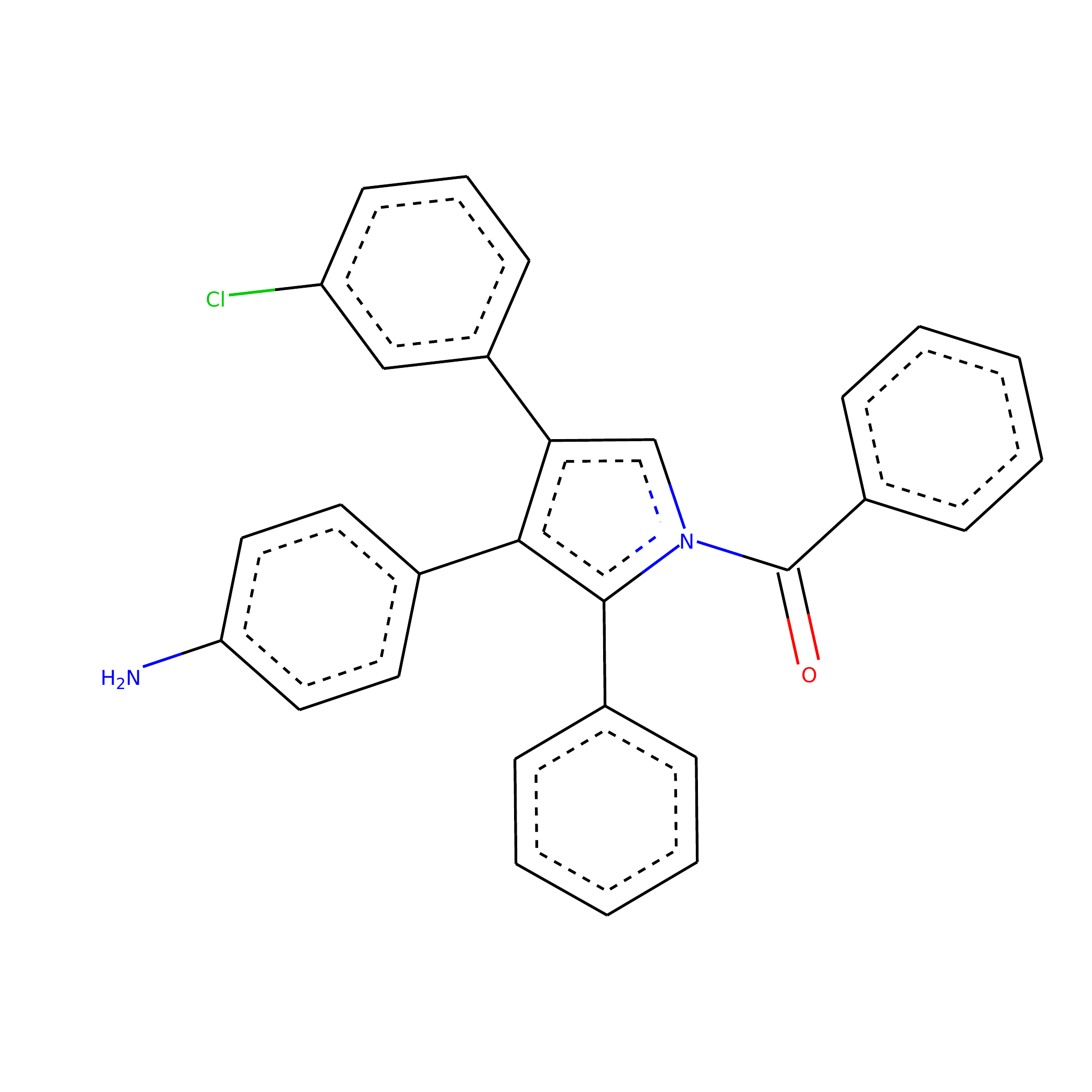}
4.538
\end{minipage}
\begin{minipage}{.24\hsize}
\centering
\includegraphics[width=\hsize]{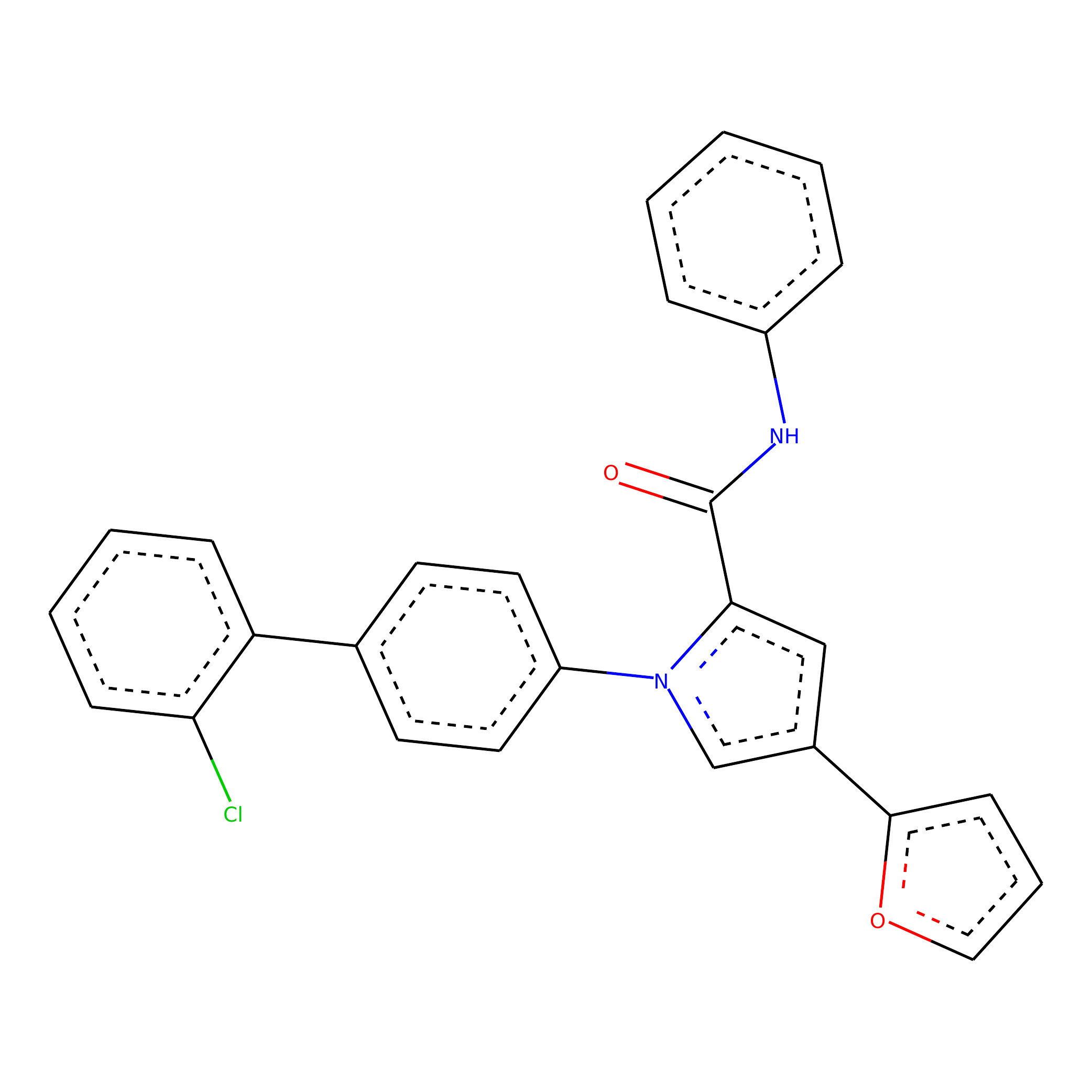}
4.484
\end{minipage}
\begin{minipage}{.24\hsize}
\centering
\includegraphics[width=\hsize]{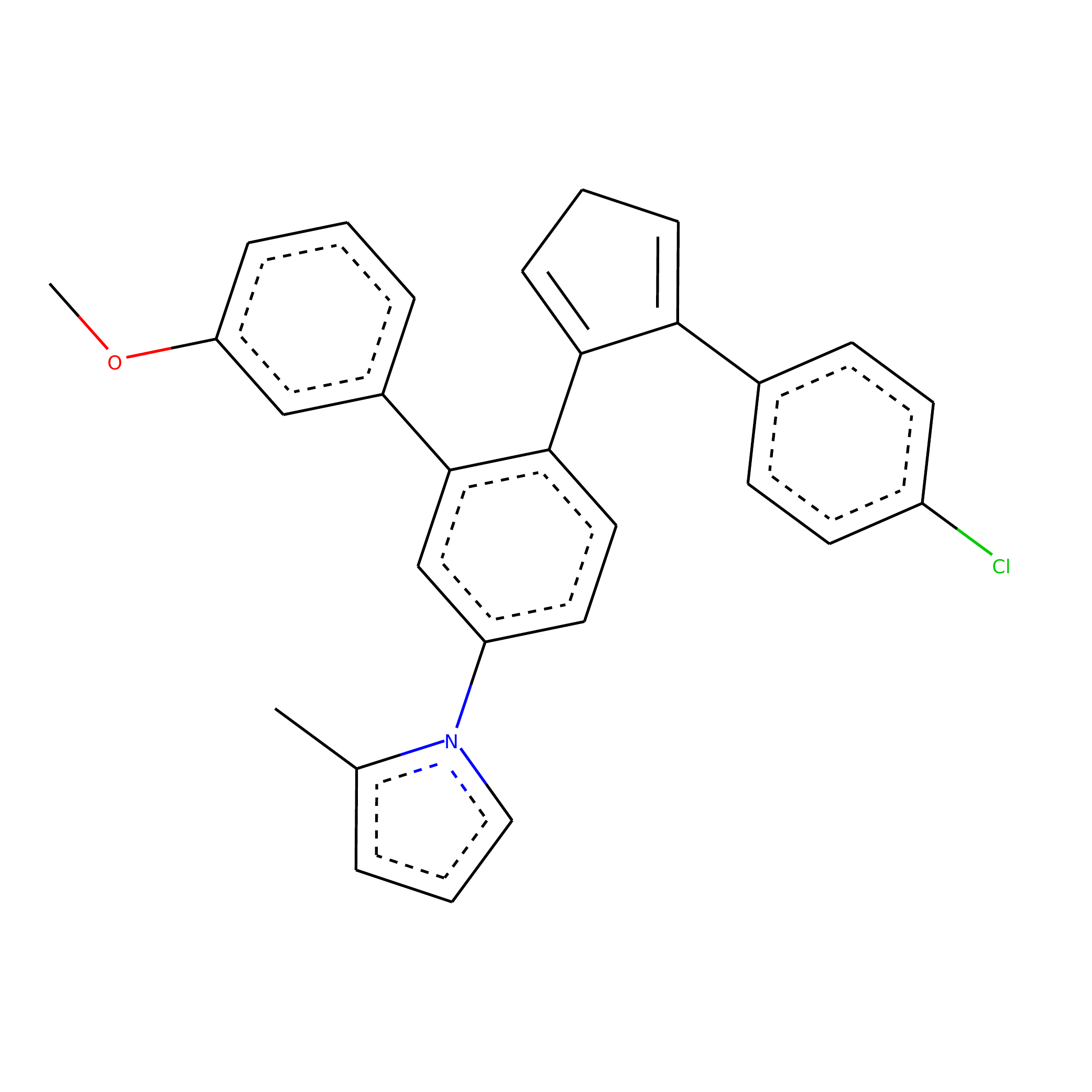}
4.464
\end{minipage} 
\end{figure*}
\begin{figure*}
\begin{minipage}{.24\hsize}
\centering
\includegraphics[width=\hsize]{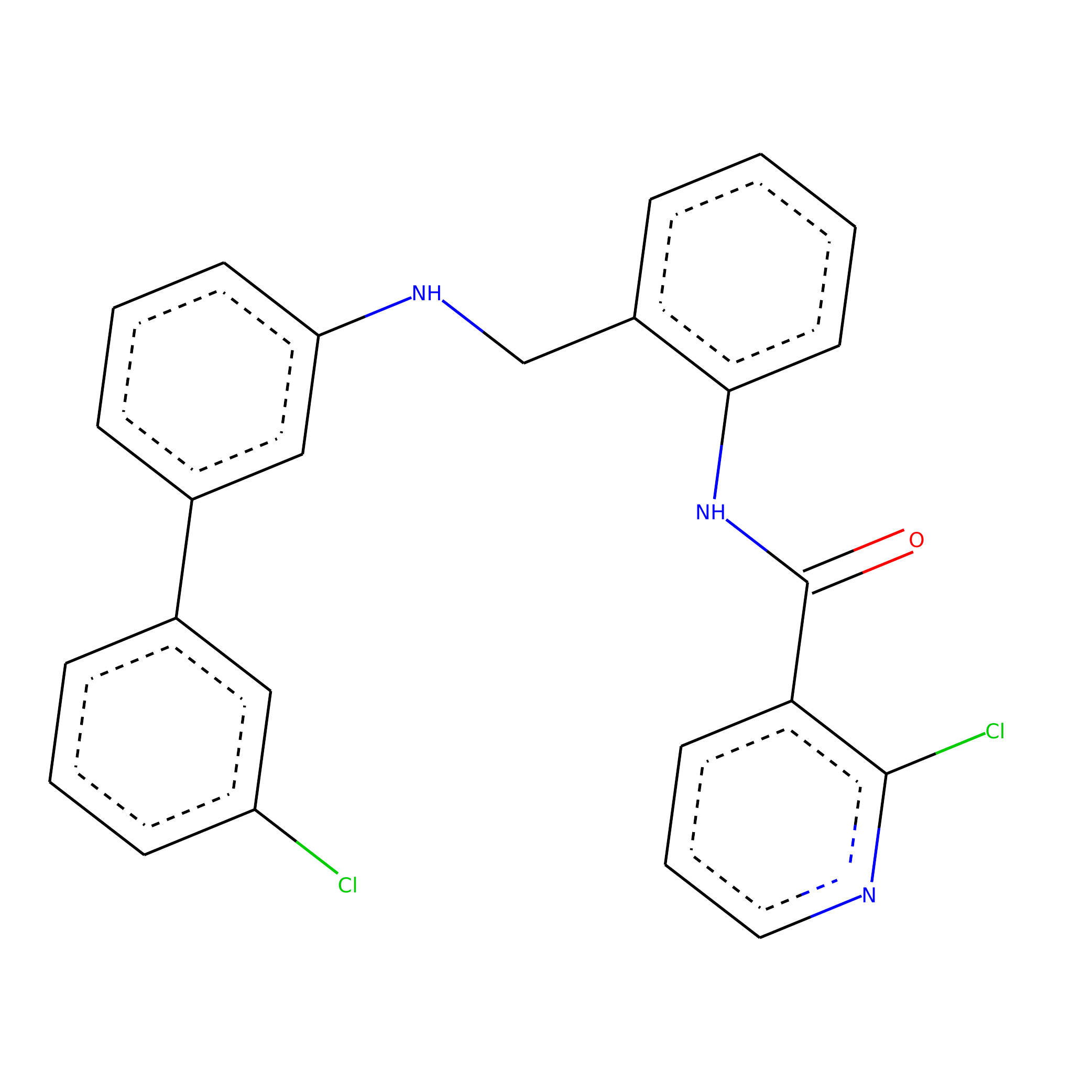}
4.450
\end{minipage}
\begin{minipage}{.24\hsize}
\centering
\includegraphics[width=\hsize]{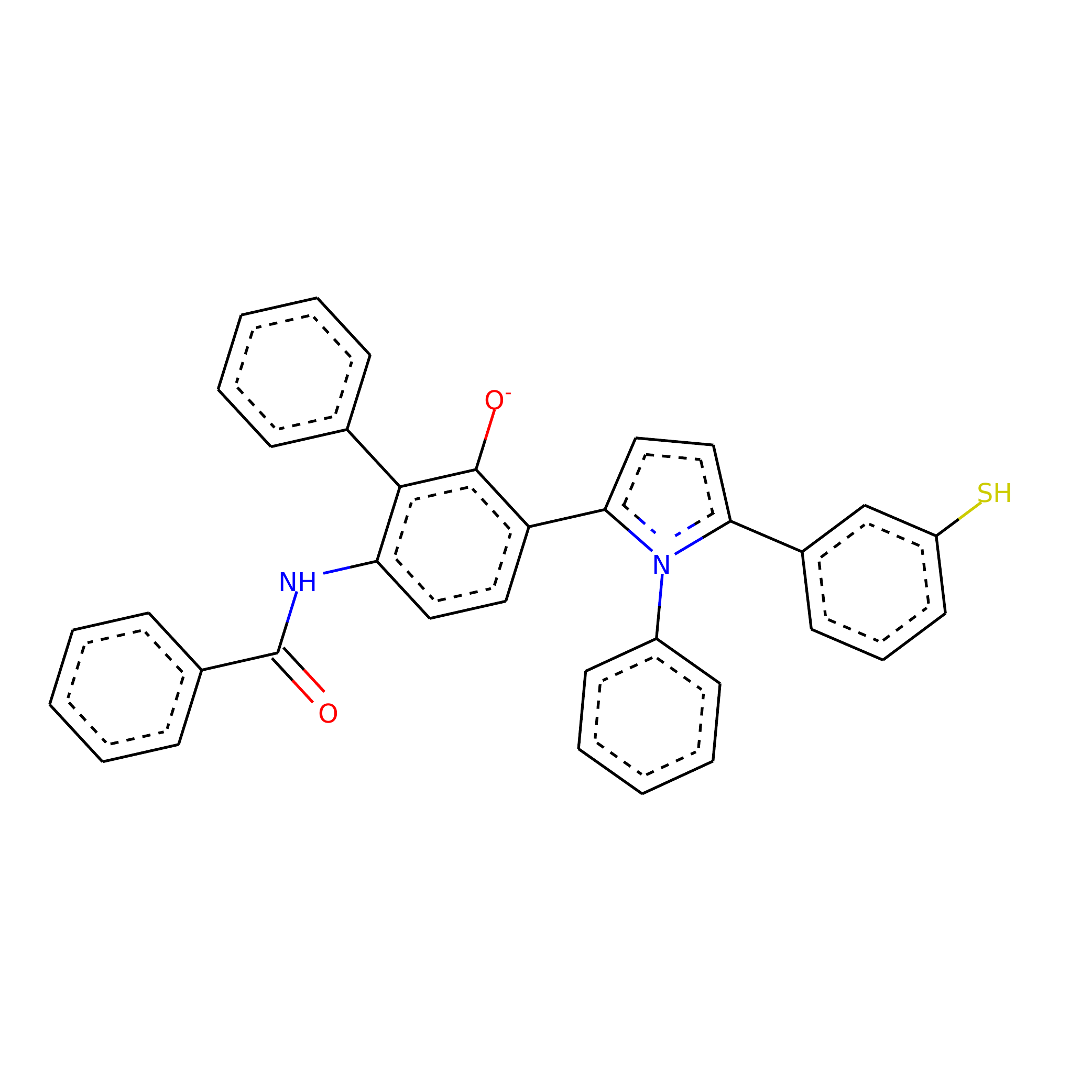}
4.443
\end{minipage}
\begin{minipage}{.24\hsize}
\centering
\includegraphics[width=\hsize]{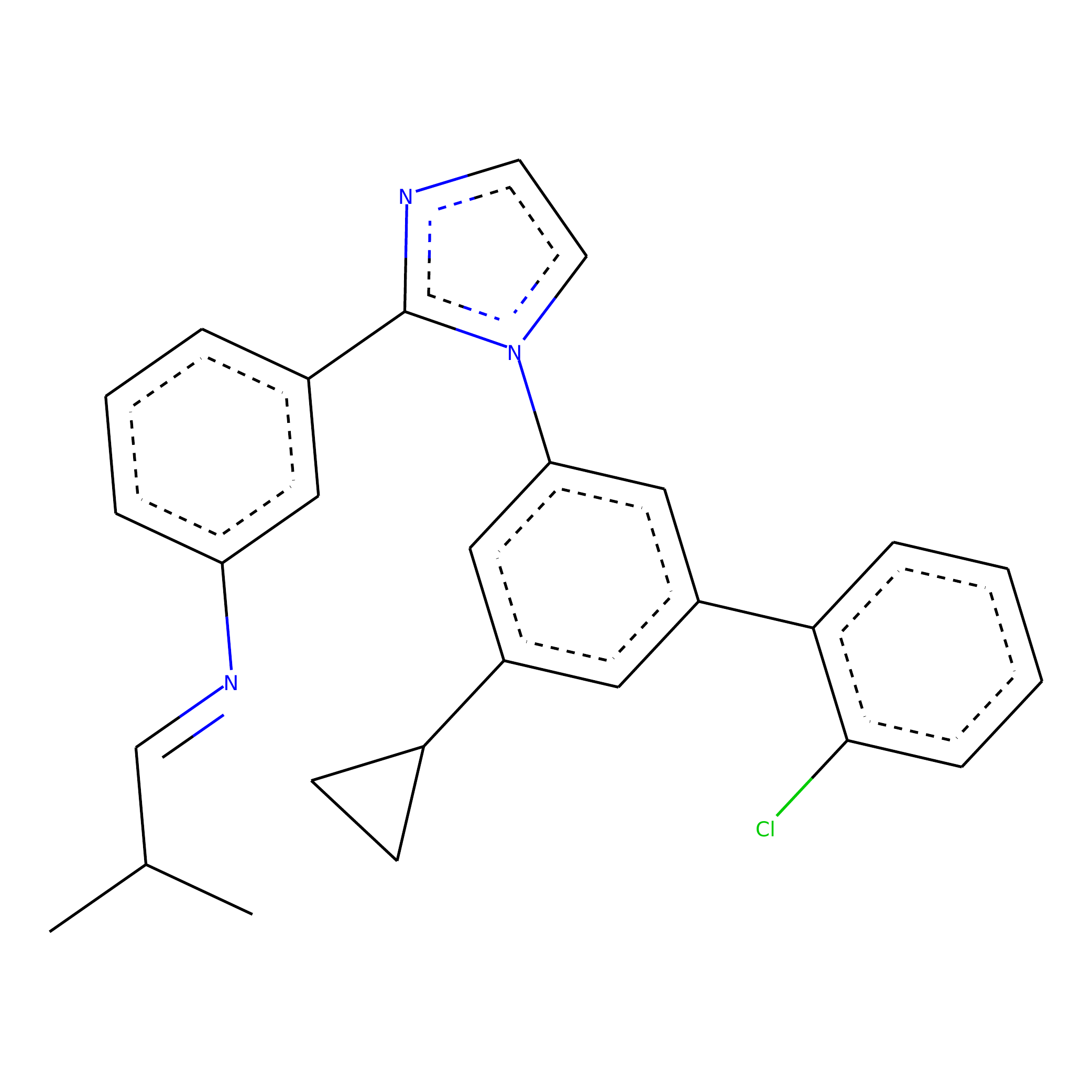}
4.408
\end{minipage}
\begin{minipage}{.24\hsize}
\centering
\includegraphics[width=\hsize]{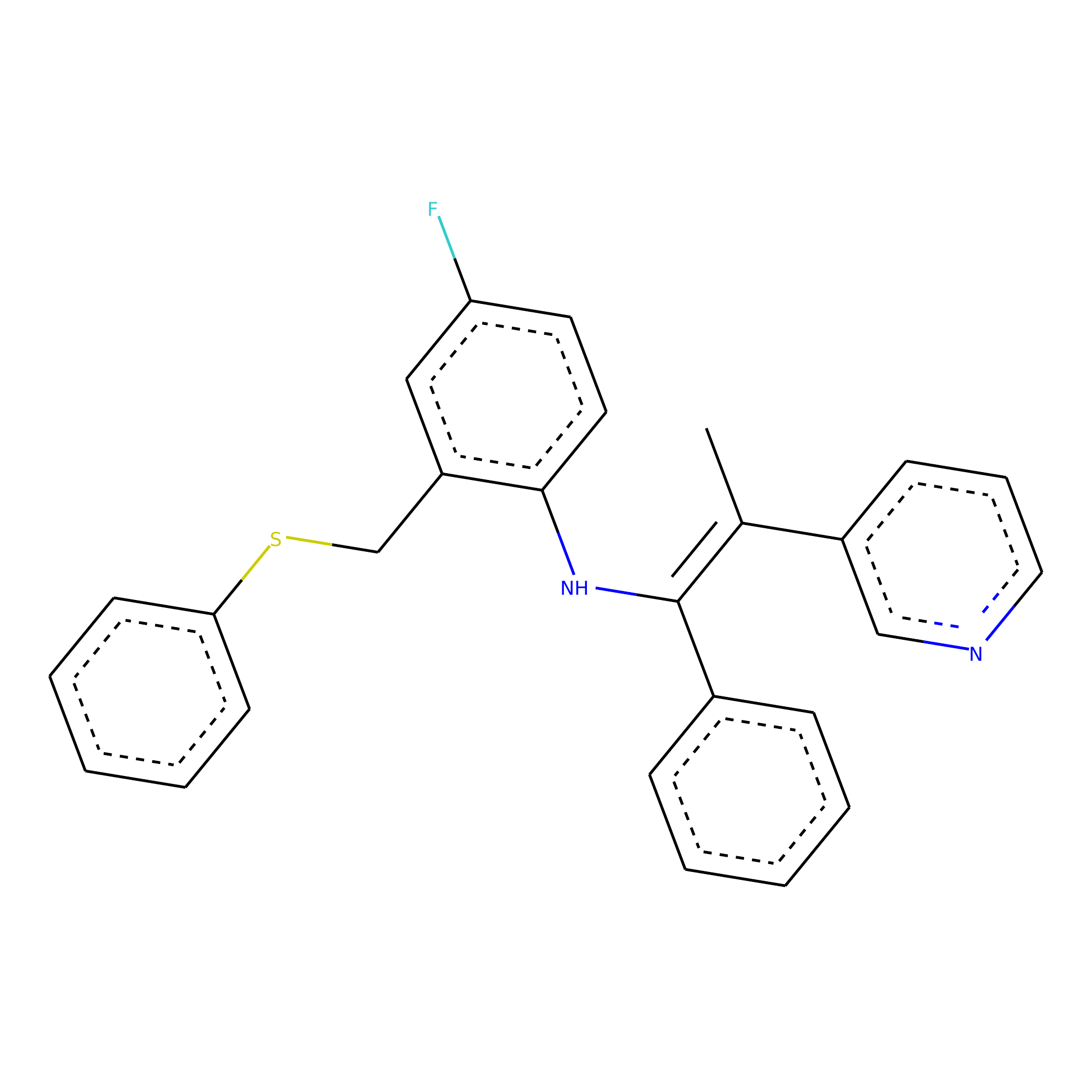}
4.404
\end{minipage}
\begin{minipage}{.24\hsize}
\centering
\includegraphics[width=\hsize]{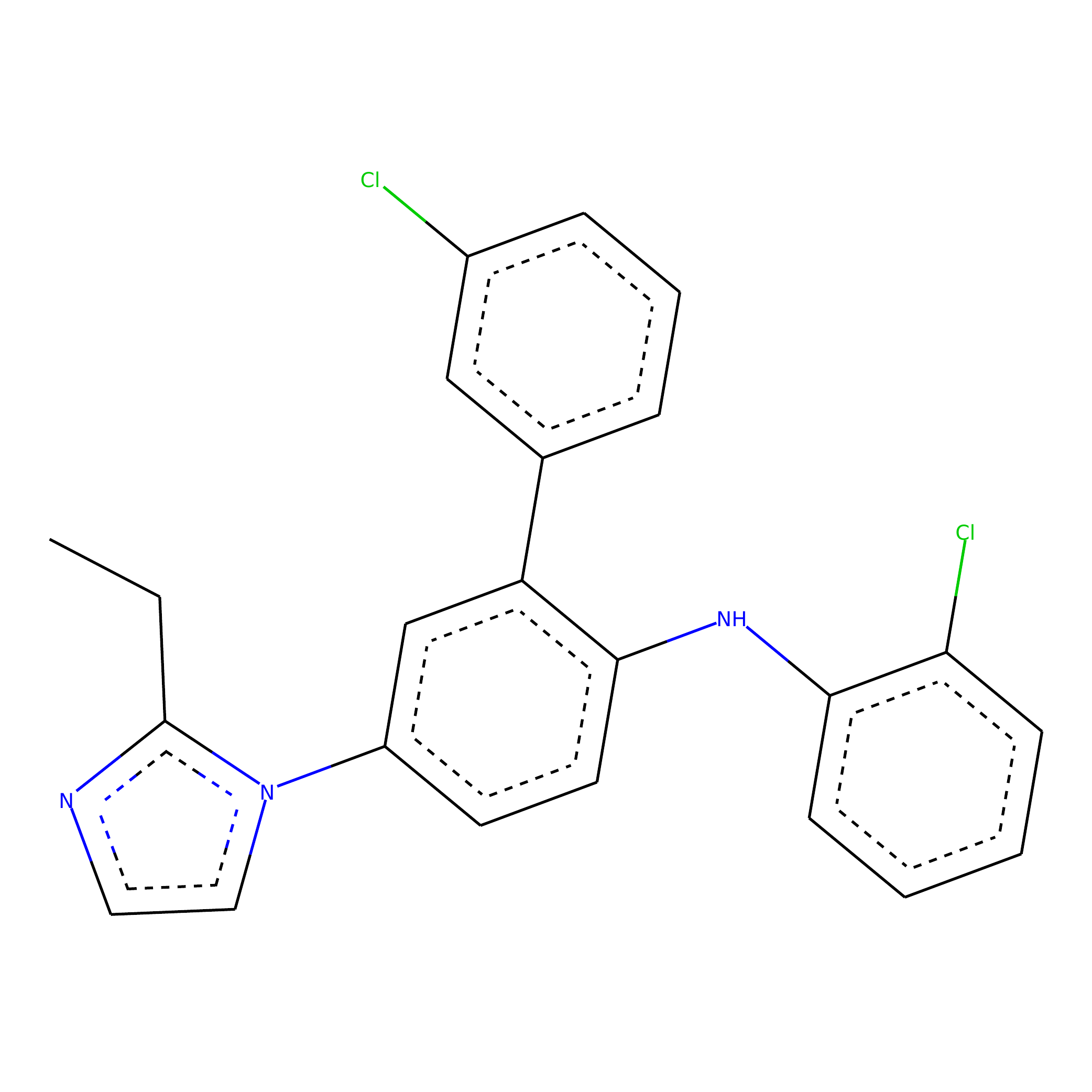}
4.374
\end{minipage}
\begin{minipage}{.24\hsize}
\centering
\includegraphics[width=\hsize]{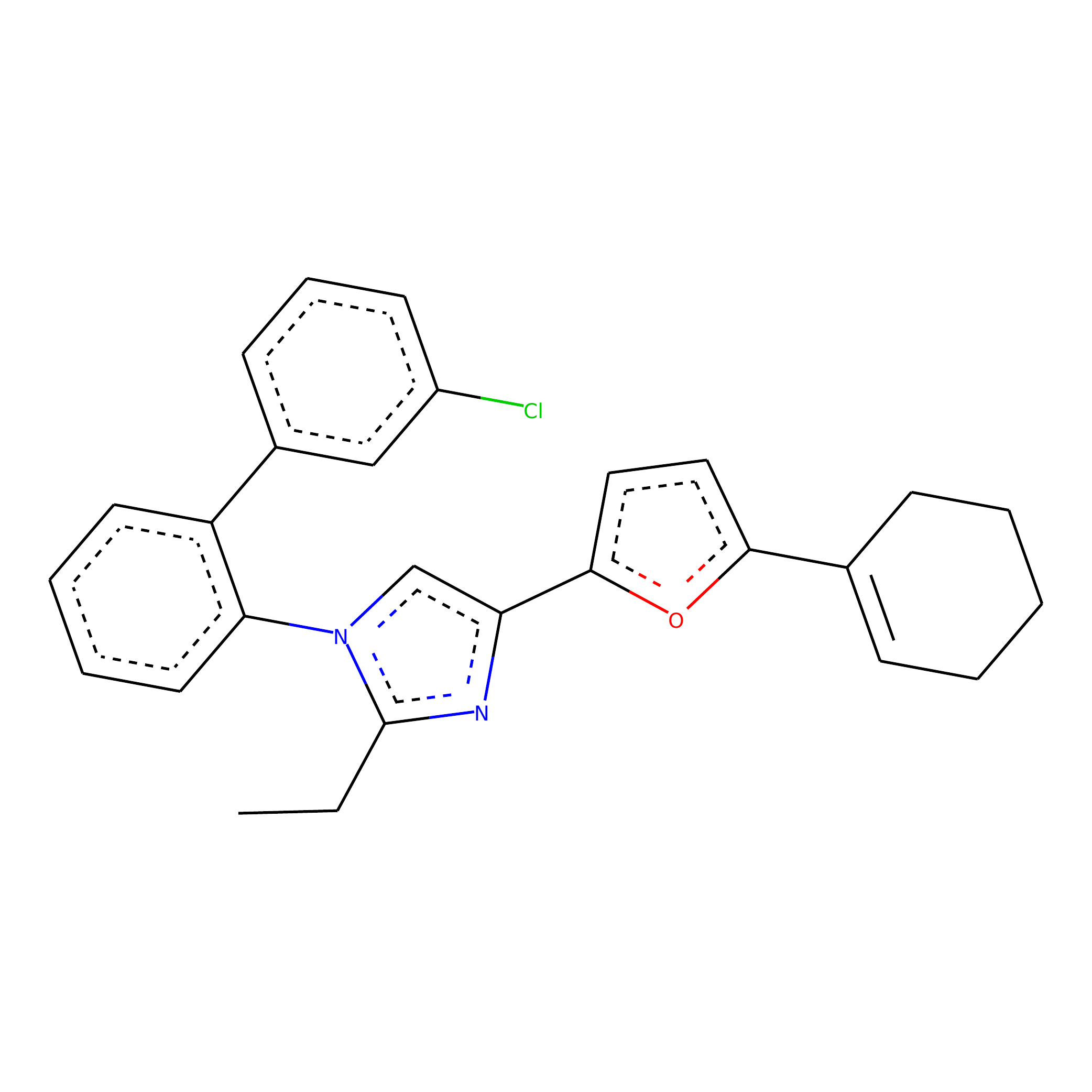}
4.362
\end{minipage}
\begin{minipage}{.24\hsize}
\centering
\includegraphics[width=\hsize]{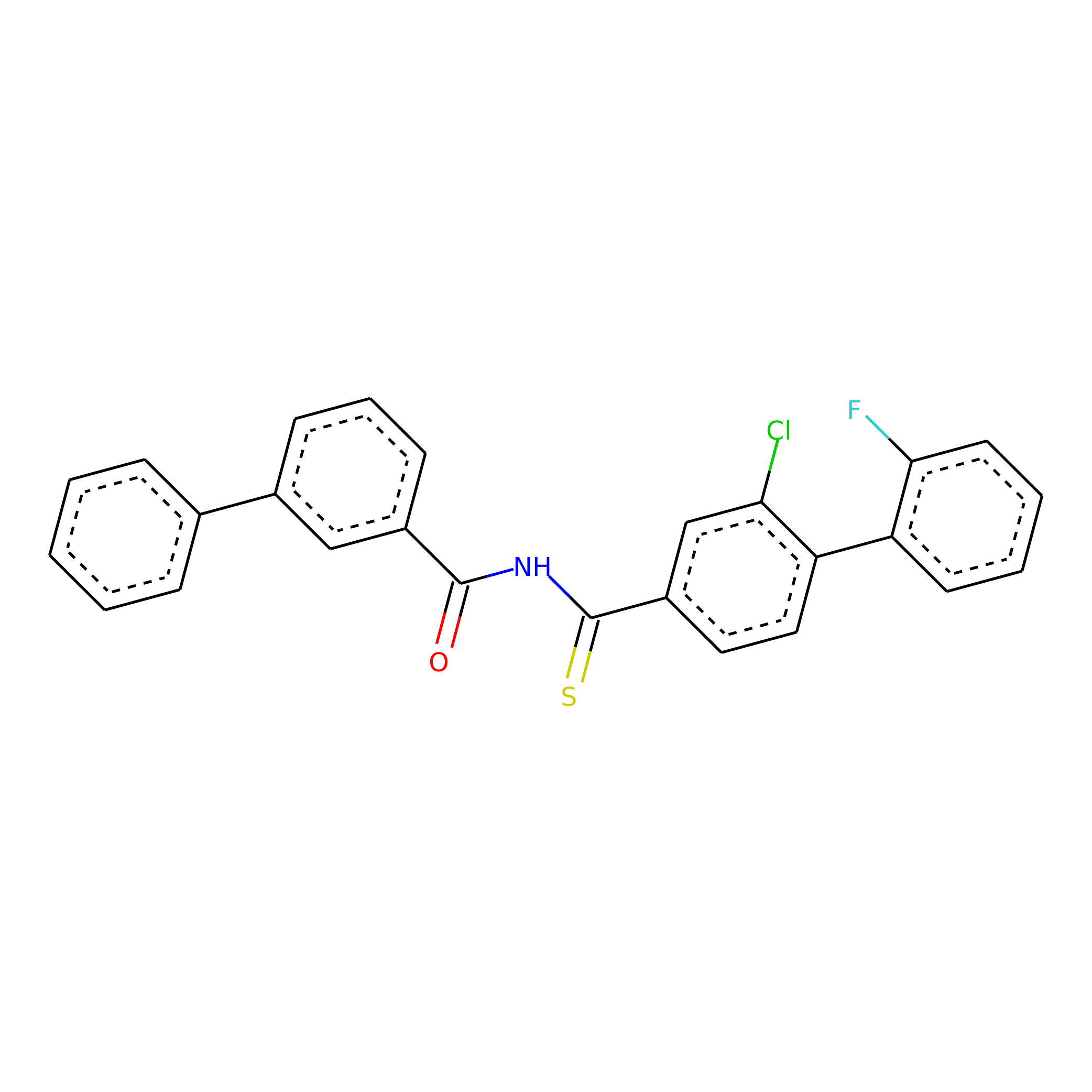}
4.354
\end{minipage}
\begin{minipage}{.24\hsize}
\centering
\includegraphics[width=\hsize]{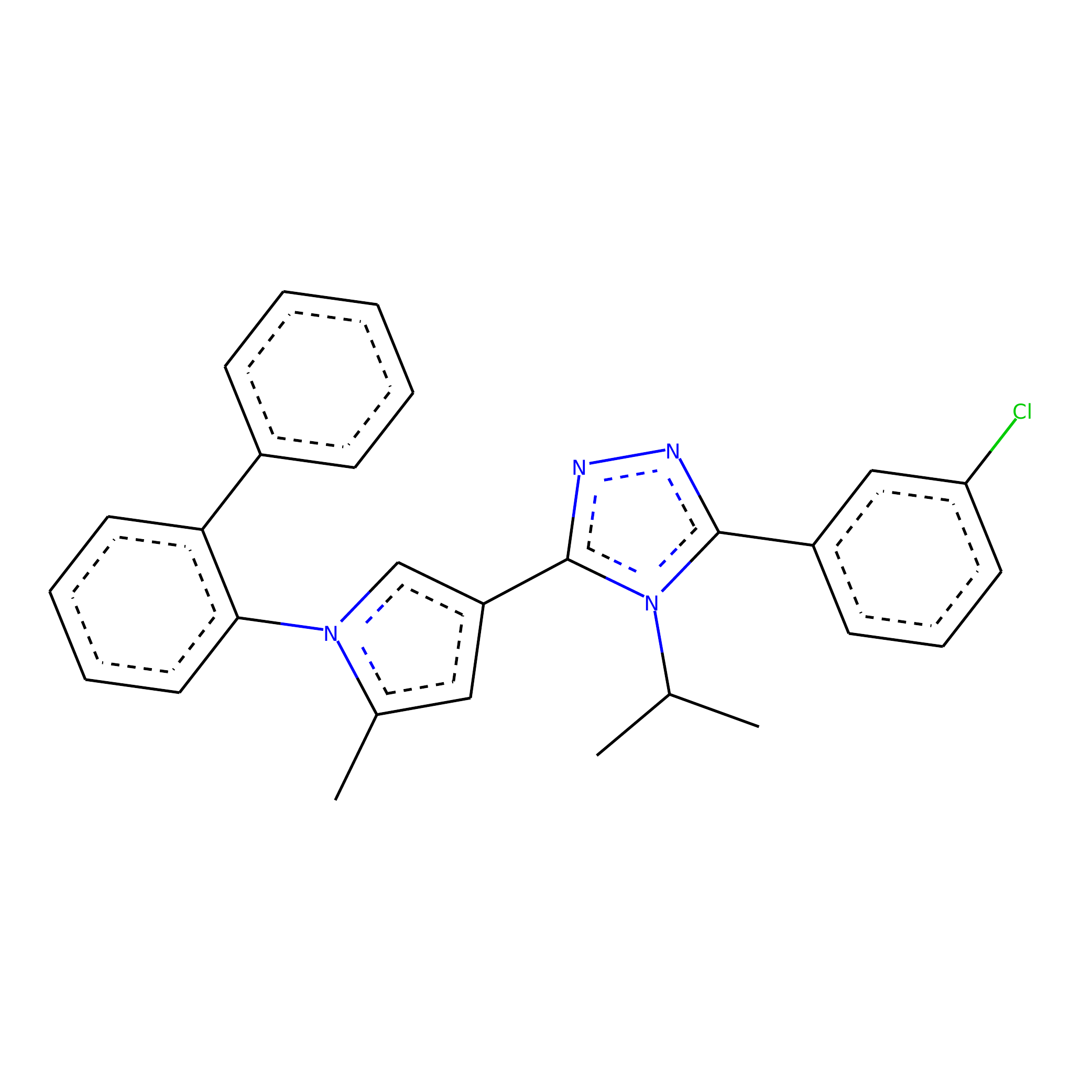}
4.351
\end{minipage}
\begin{minipage}{.24\hsize}
\centering
\includegraphics[width=\hsize]{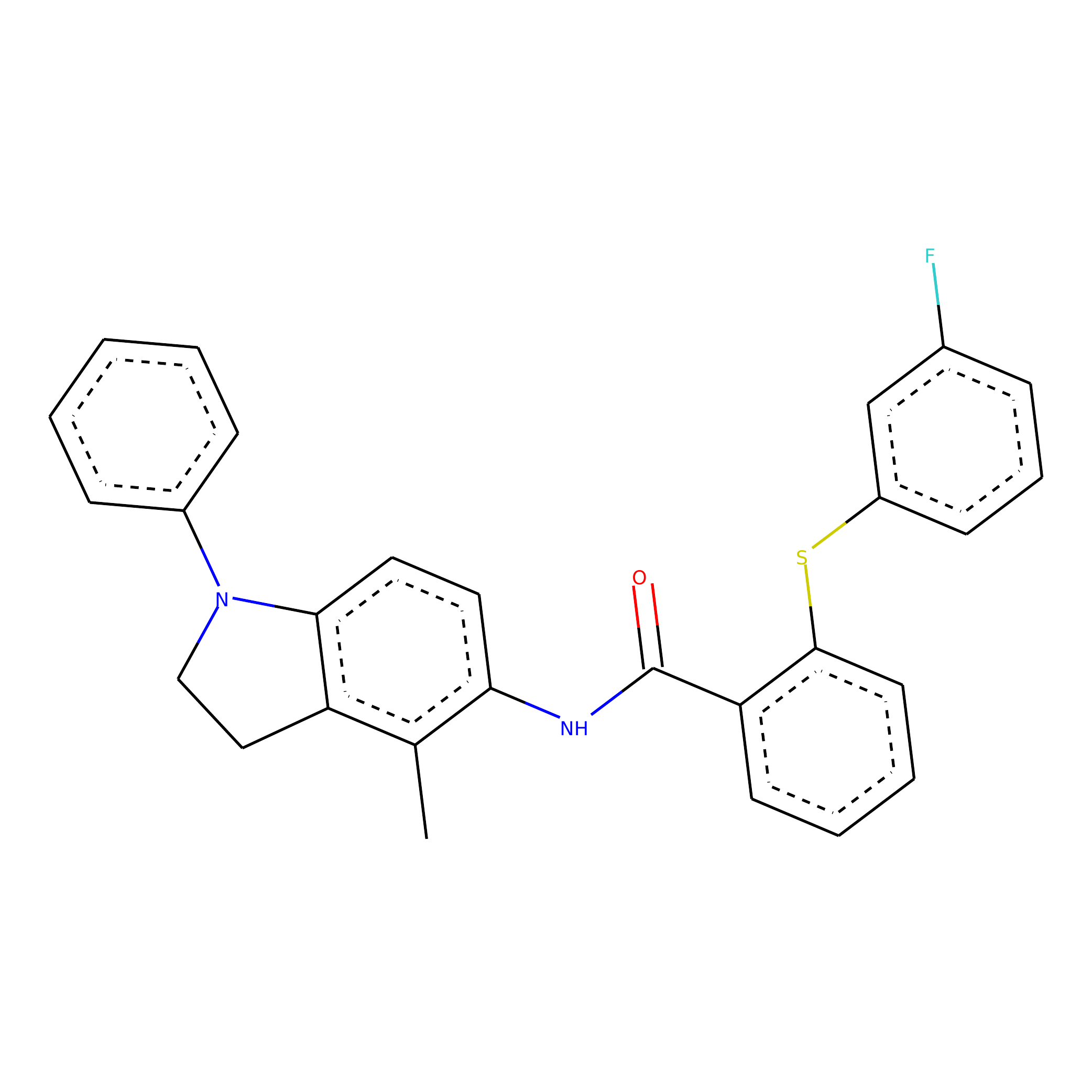}
4.349
\end{minipage}
\begin{minipage}{.24\hsize}
\centering
\includegraphics[width=\hsize]{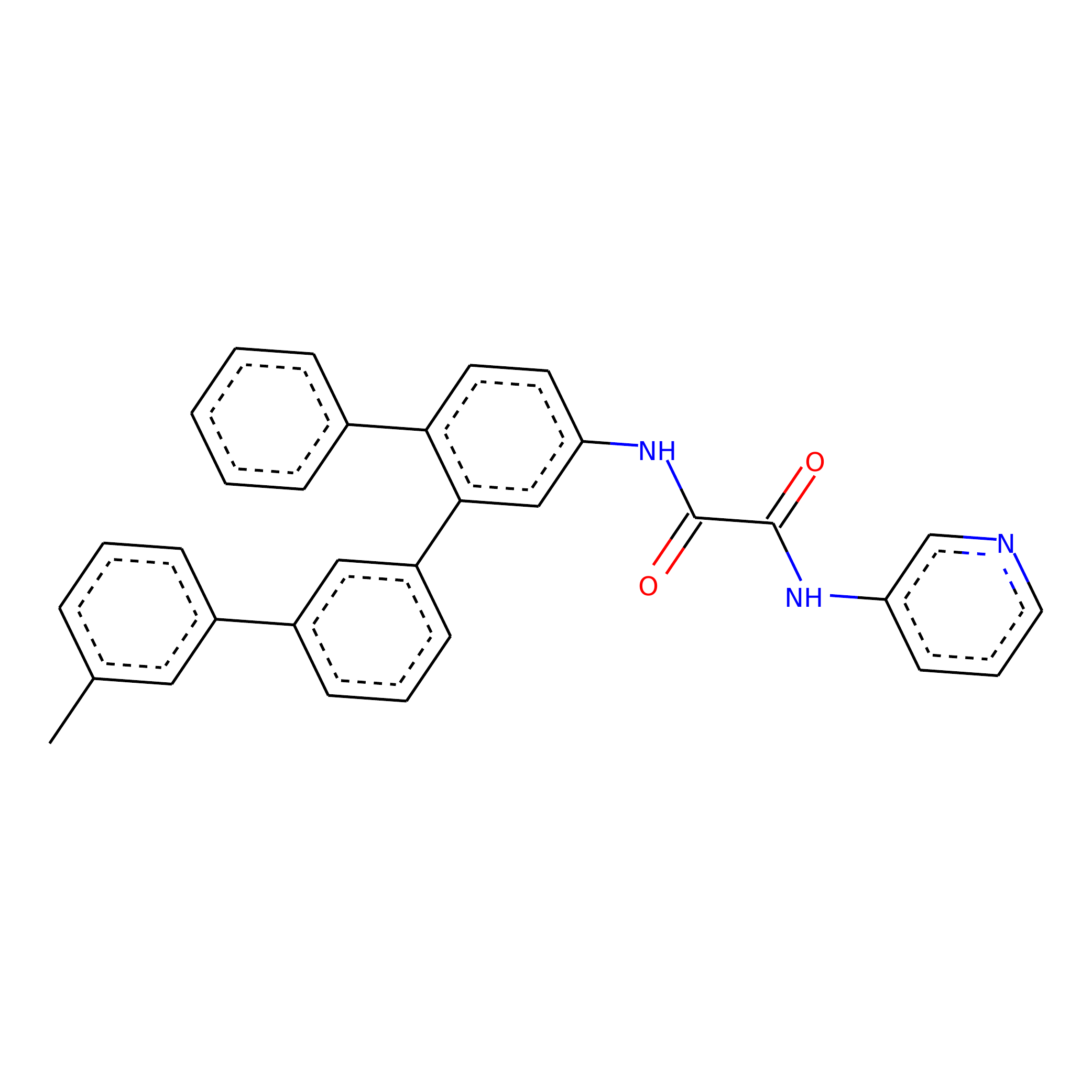}
4.341
\end{minipage}
\begin{minipage}{.24\hsize}
\centering
\includegraphics[width=\hsize]{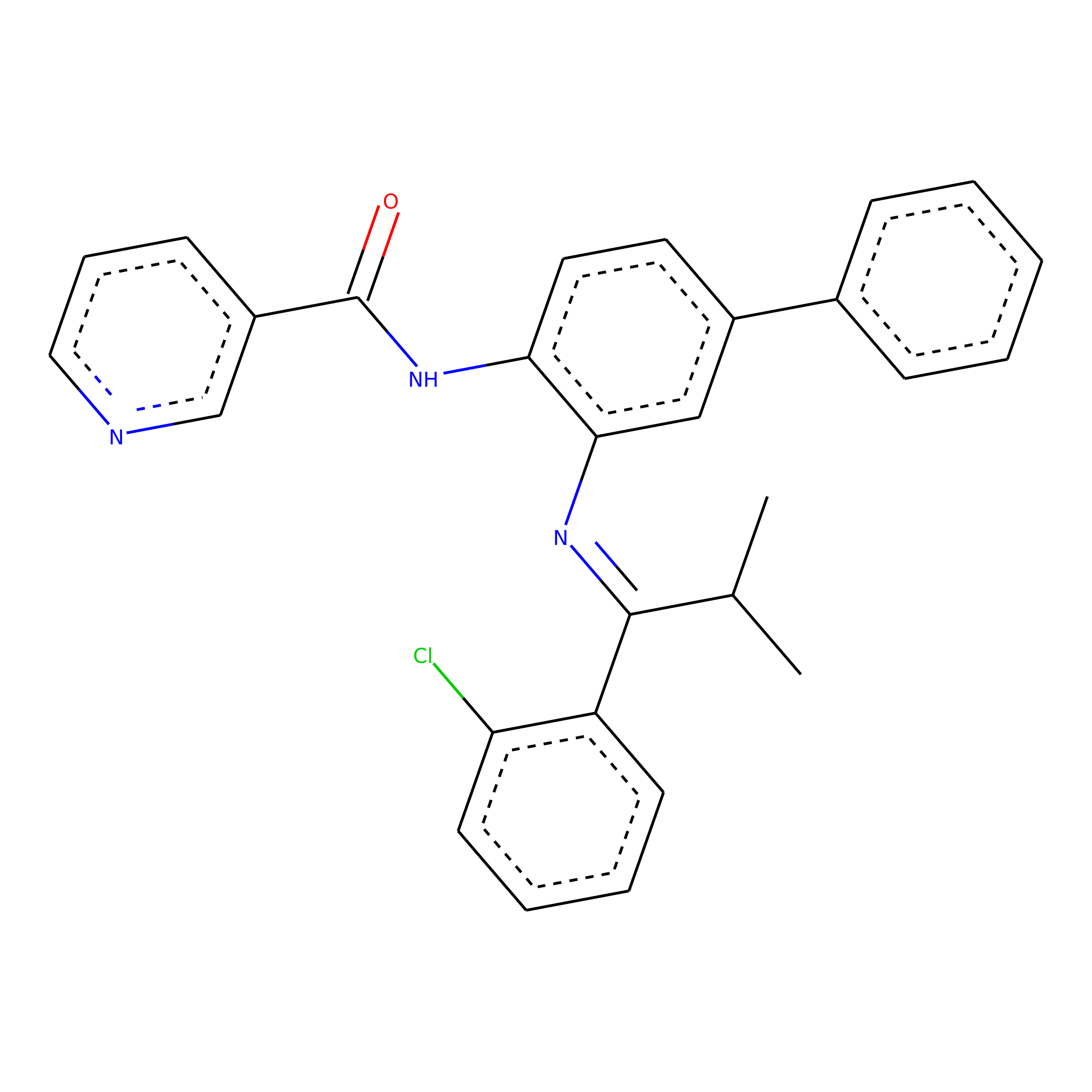}
4.335
\end{minipage}
\begin{minipage}{.24\hsize}
\centering
\includegraphics[width=\hsize]{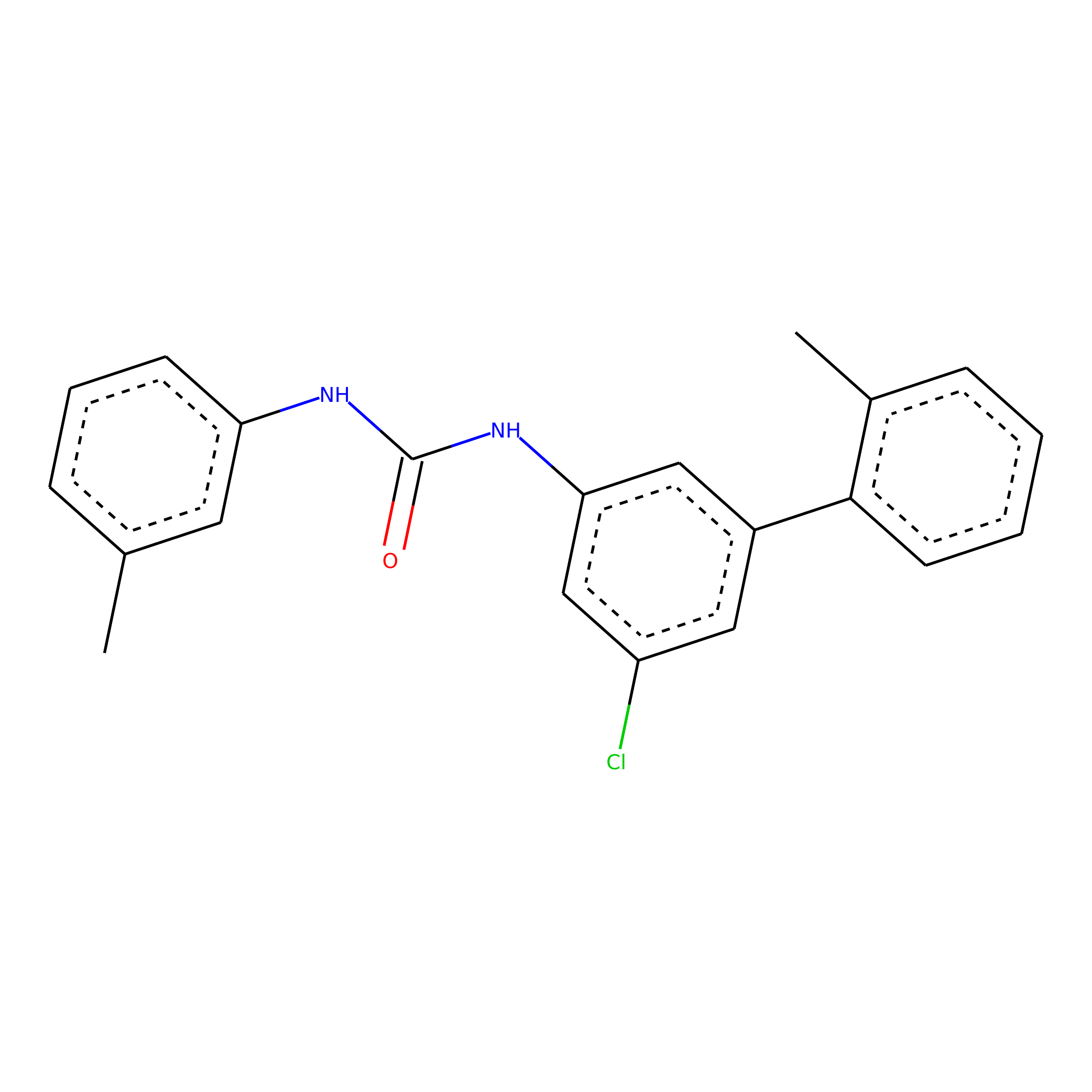}
4.294
\end{minipage}
\begin{minipage}{.24\hsize}
\centering
\includegraphics[width=\hsize]{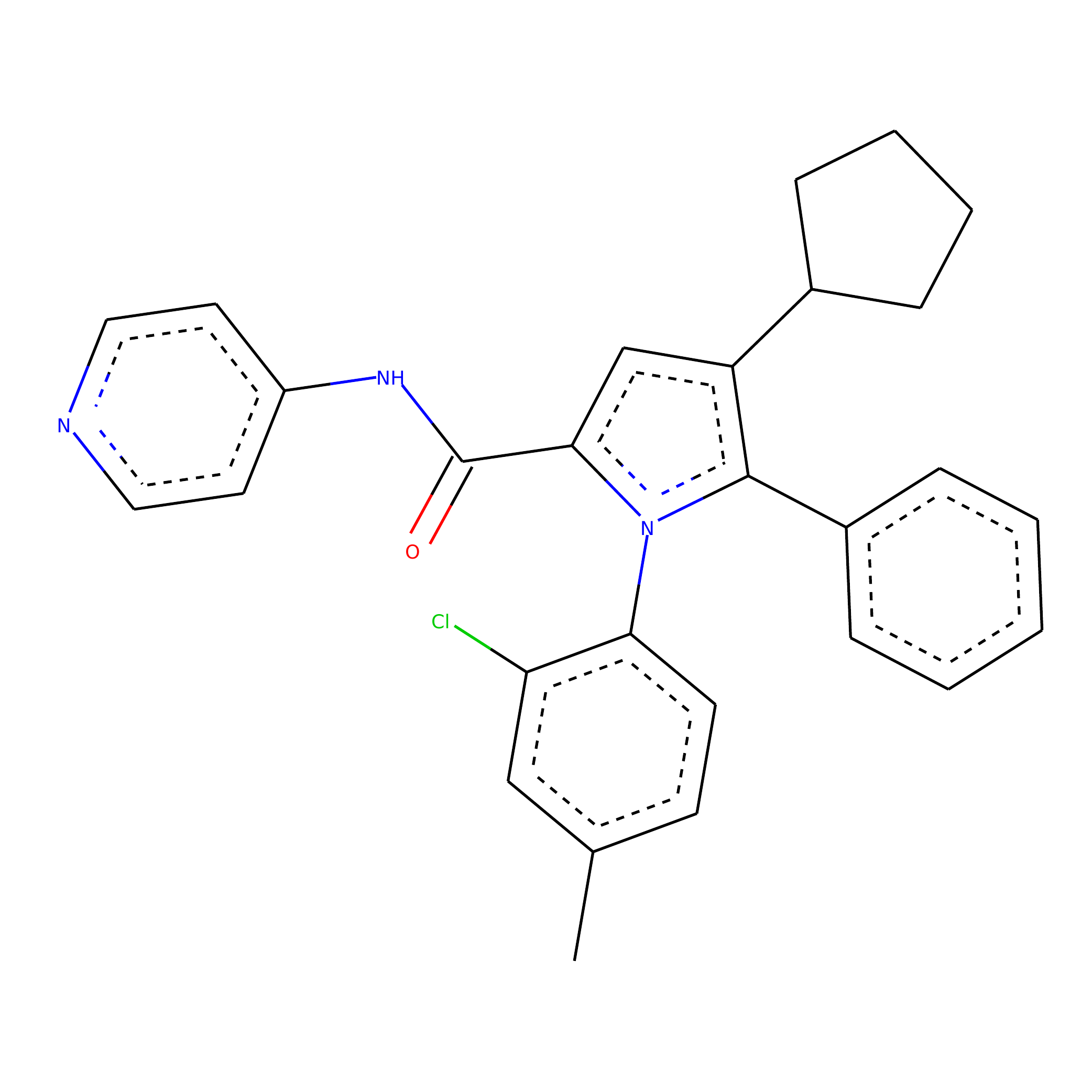}
4.274
\end{minipage}
\begin{minipage}{.24\hsize}
\centering
\includegraphics[width=\hsize]{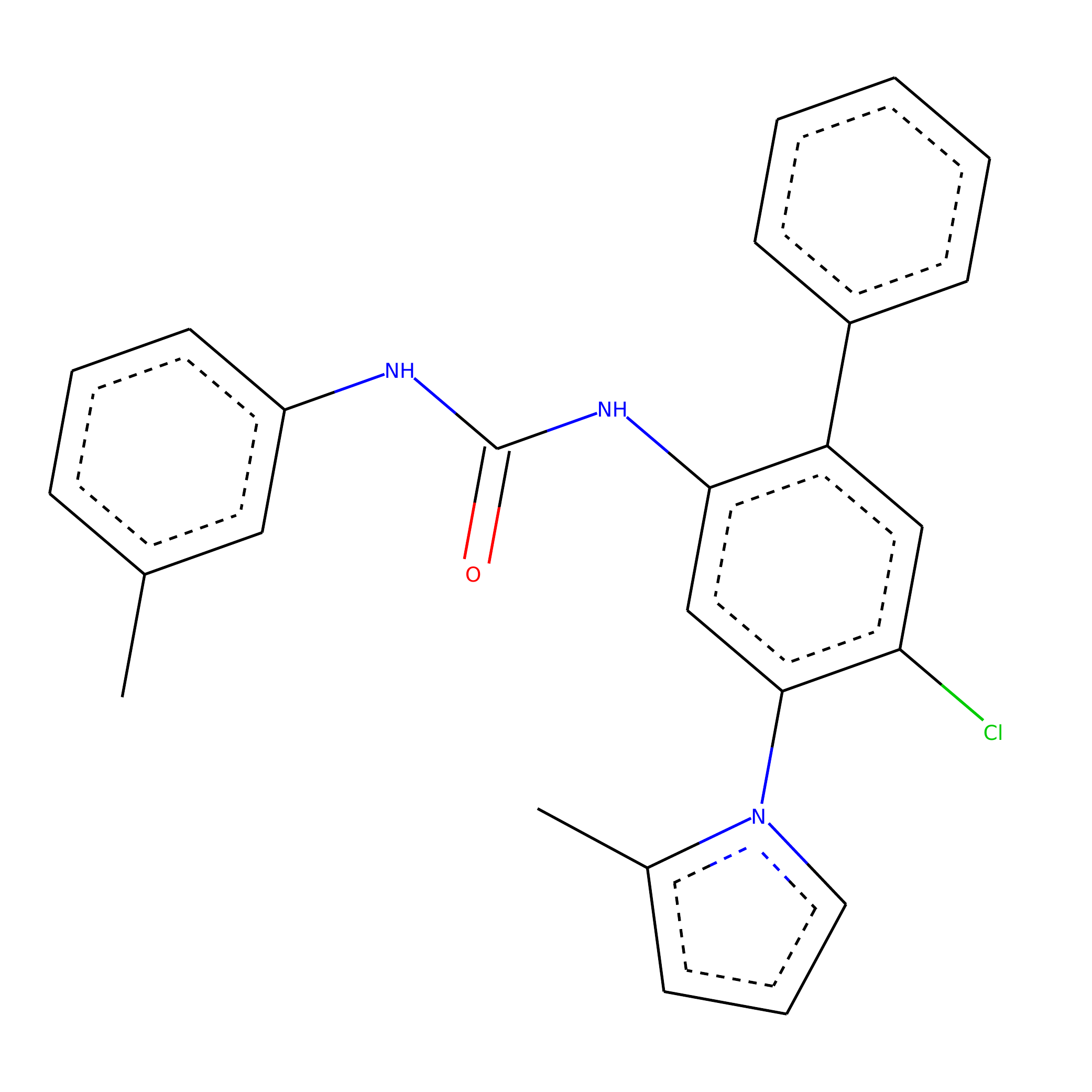}
4.265
\end{minipage}
\begin{minipage}{.24\hsize}
\centering
\includegraphics[width=\hsize]{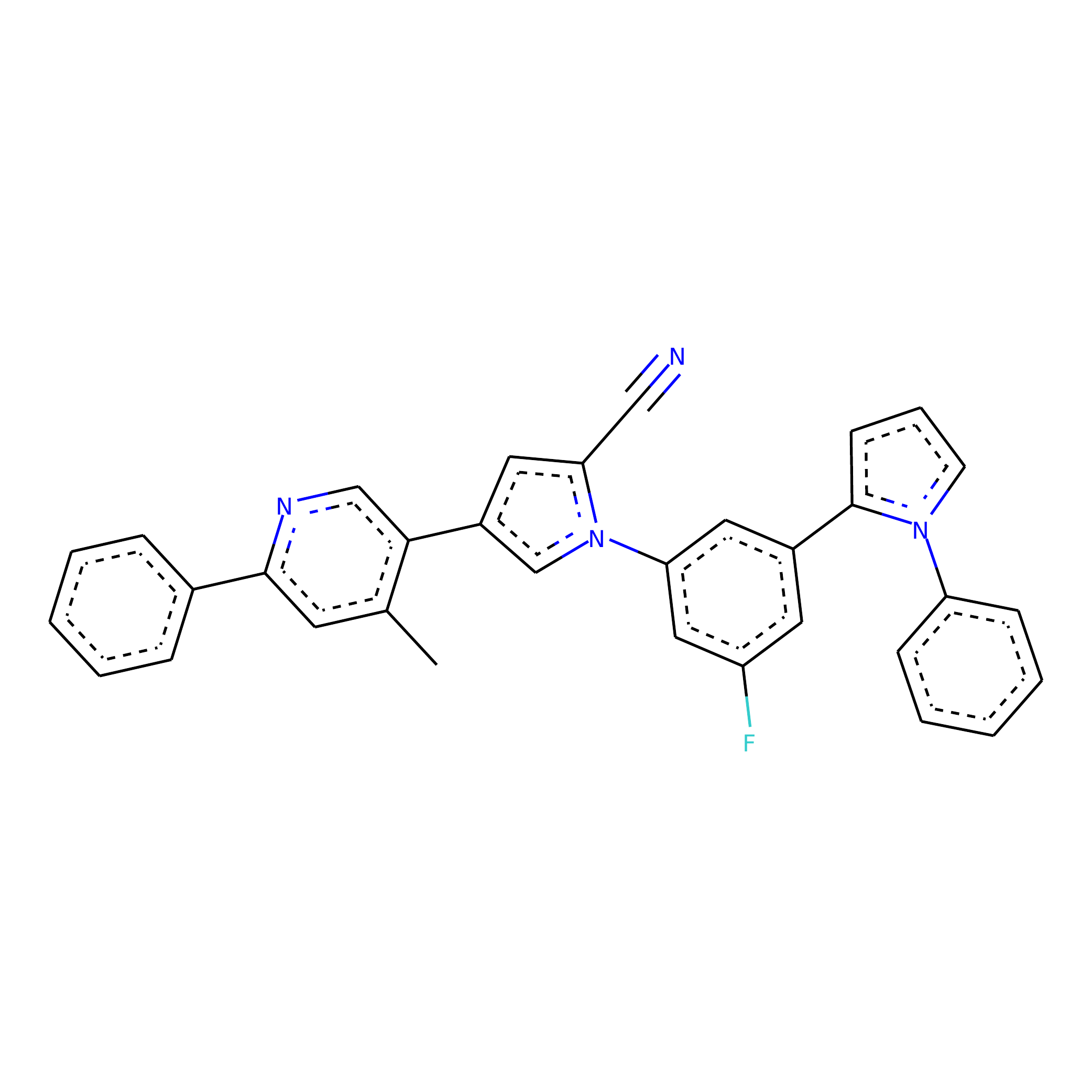}
4.259
\end{minipage}
\begin{minipage}{.24\hsize}
\centering
\includegraphics[width=\hsize]{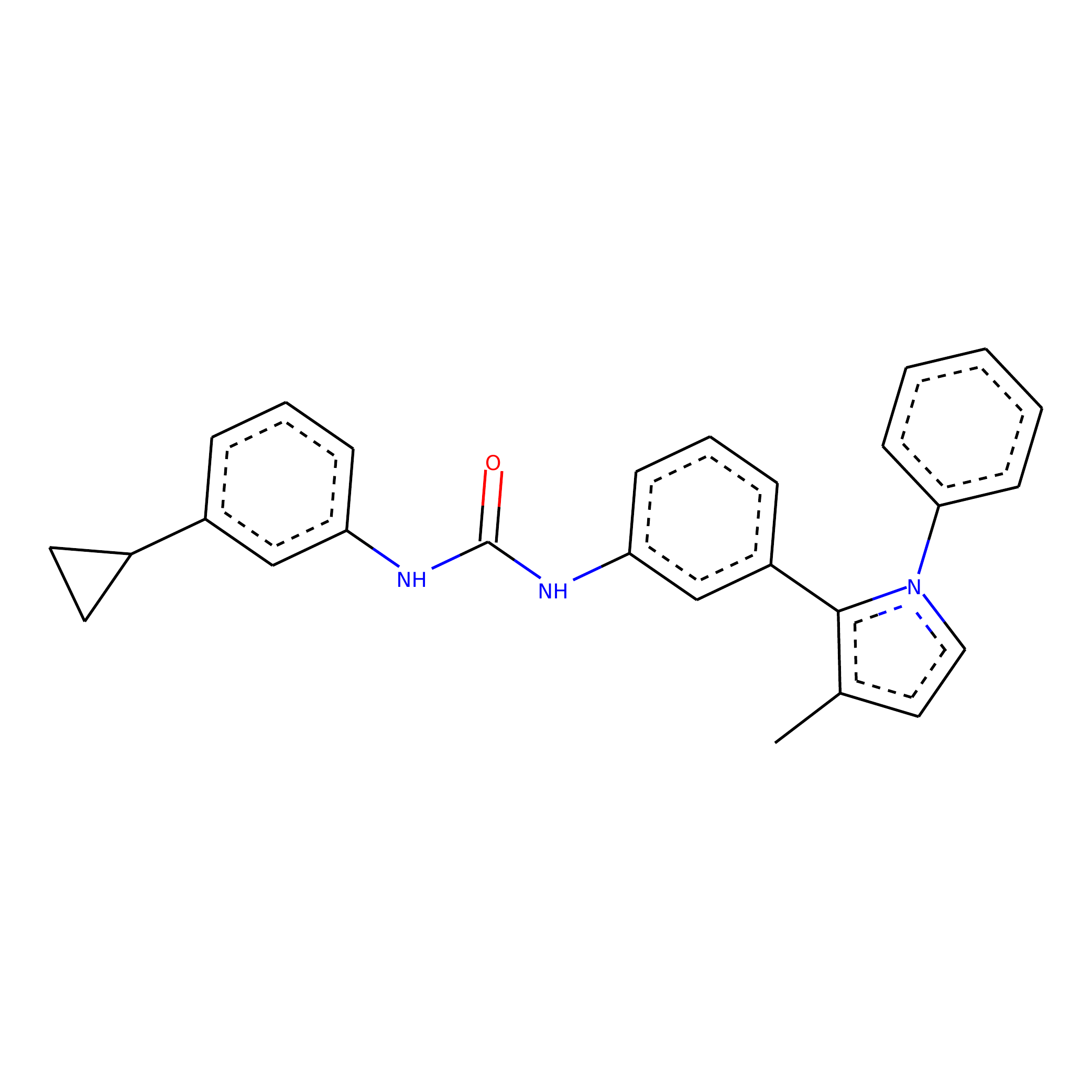}
4.258
\end{minipage}
\begin{minipage}{.24\hsize}
\centering
\includegraphics[width=\hsize]{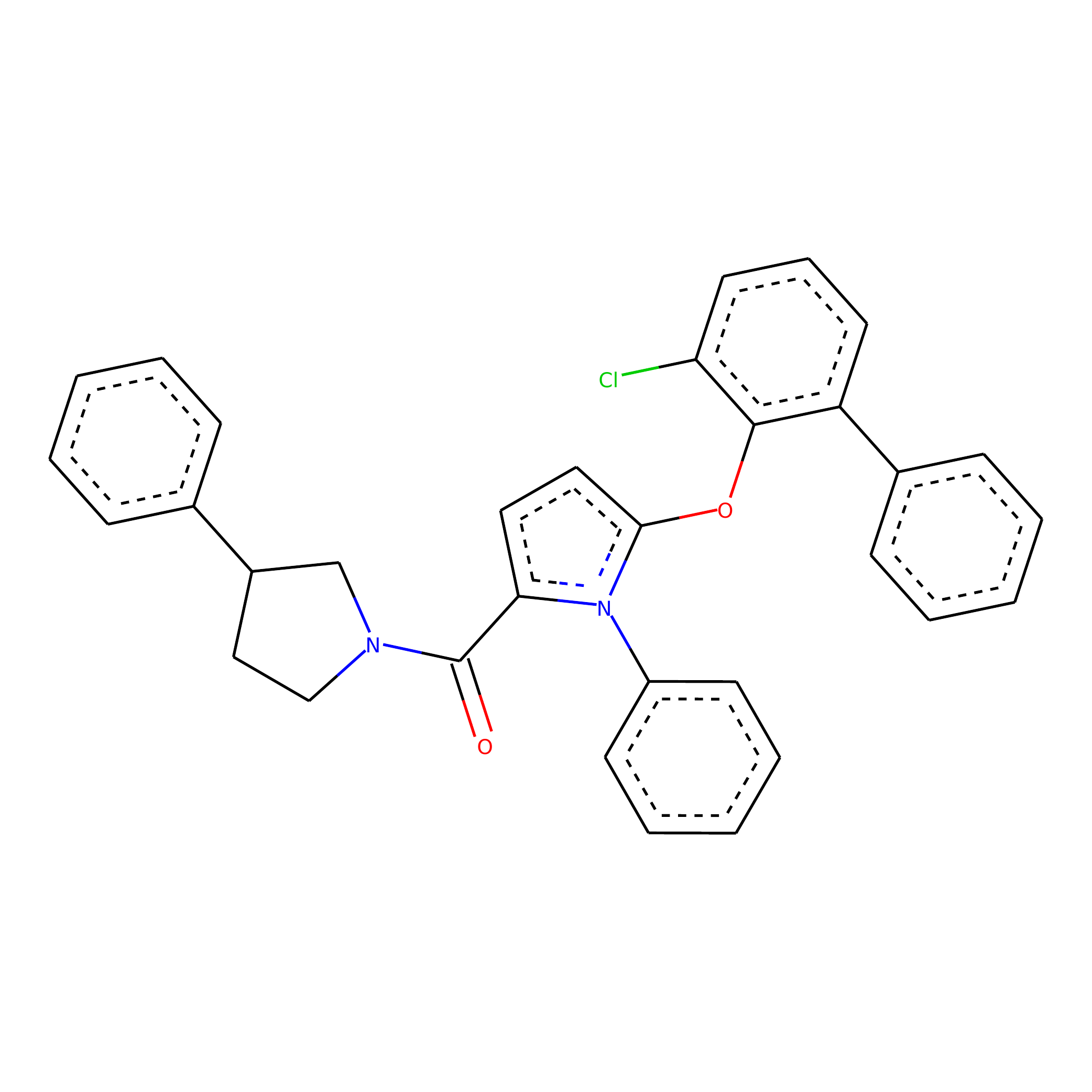}
4.242
\end{minipage}
\begin{minipage}{.24\hsize}
\centering
\includegraphics[width=\hsize]{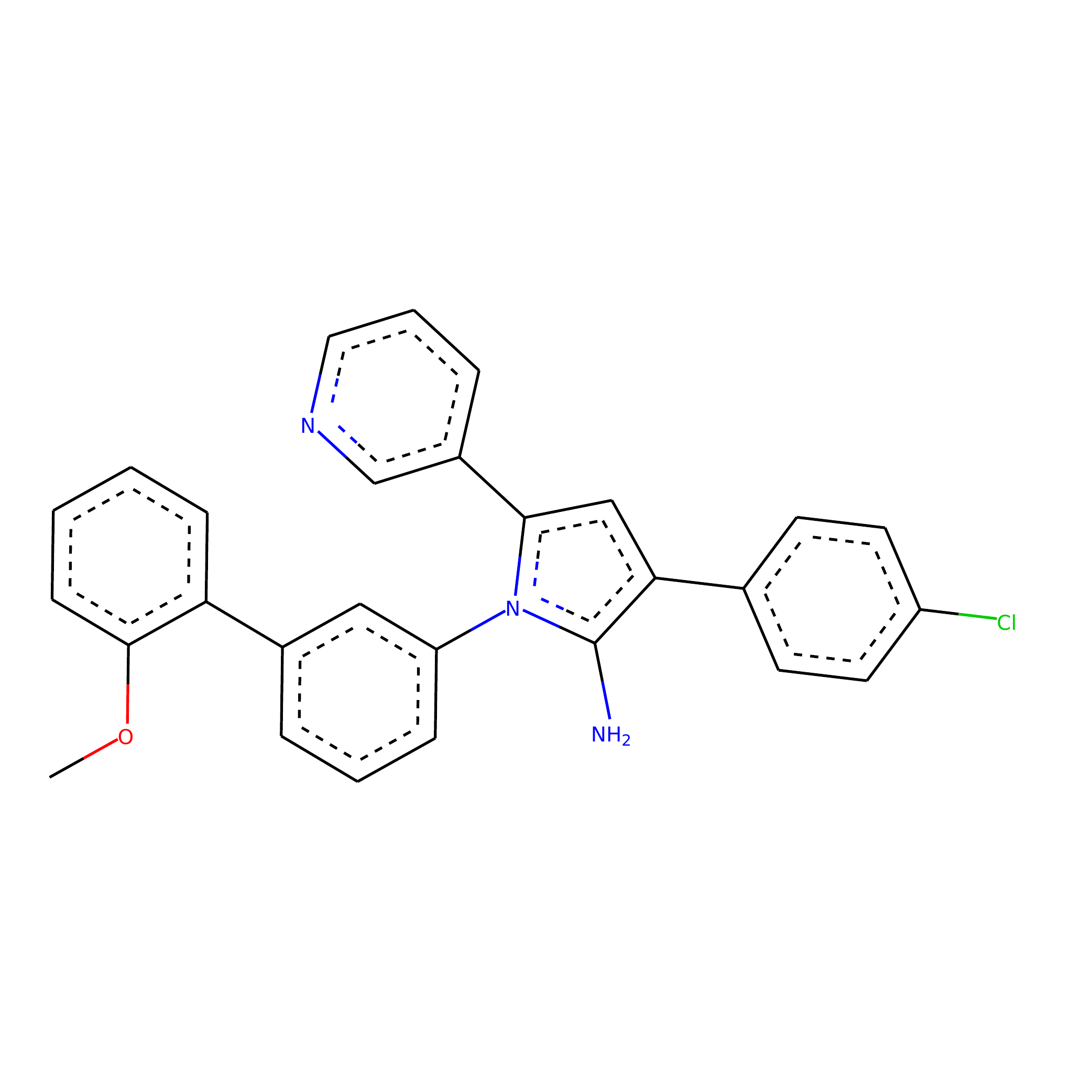}
4.240
\end{minipage}
\begin{minipage}{.24\hsize}
\centering
\includegraphics[width=\hsize]{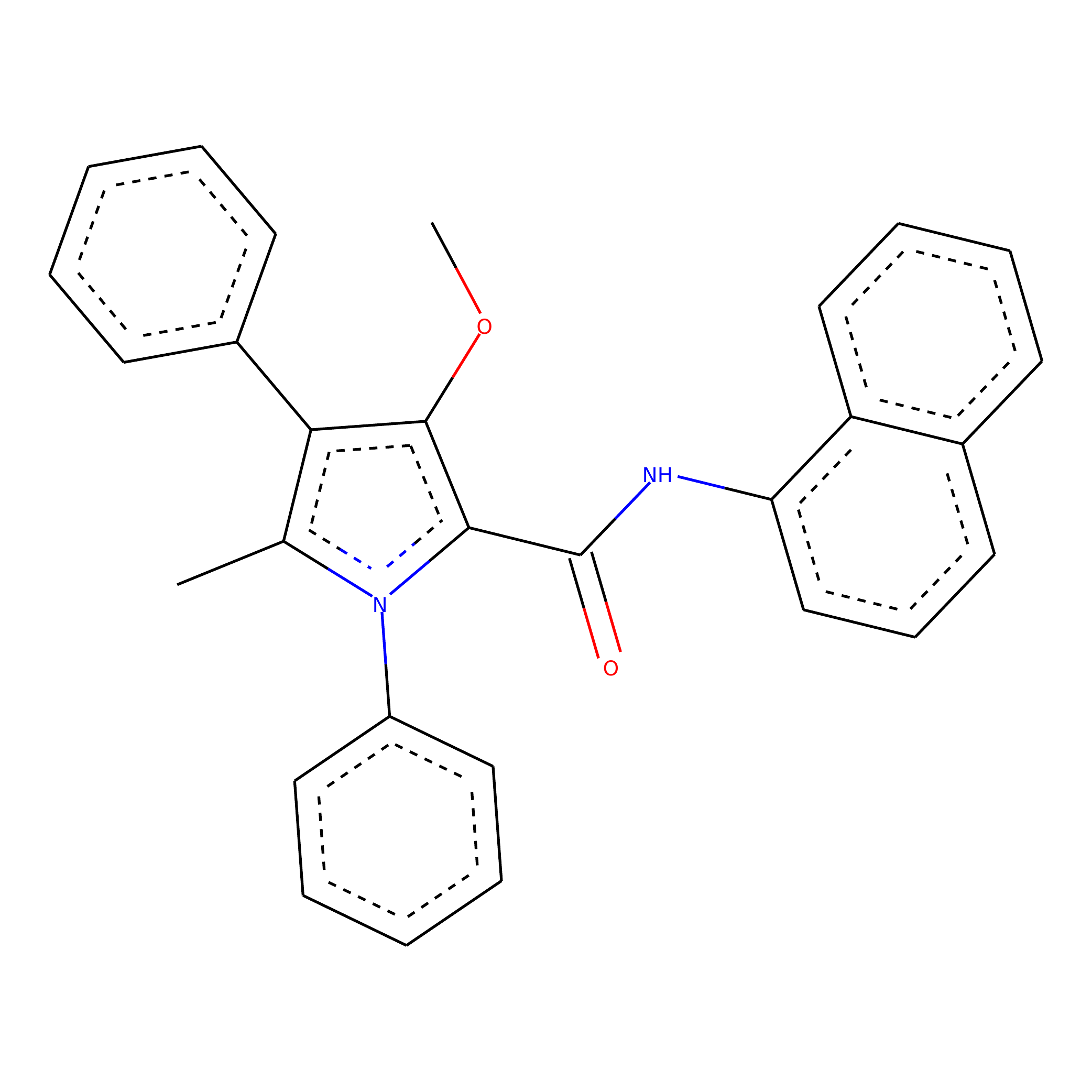}
4.233
\end{minipage}
\begin{minipage}{.24\hsize}
\centering
\includegraphics[width=\hsize]{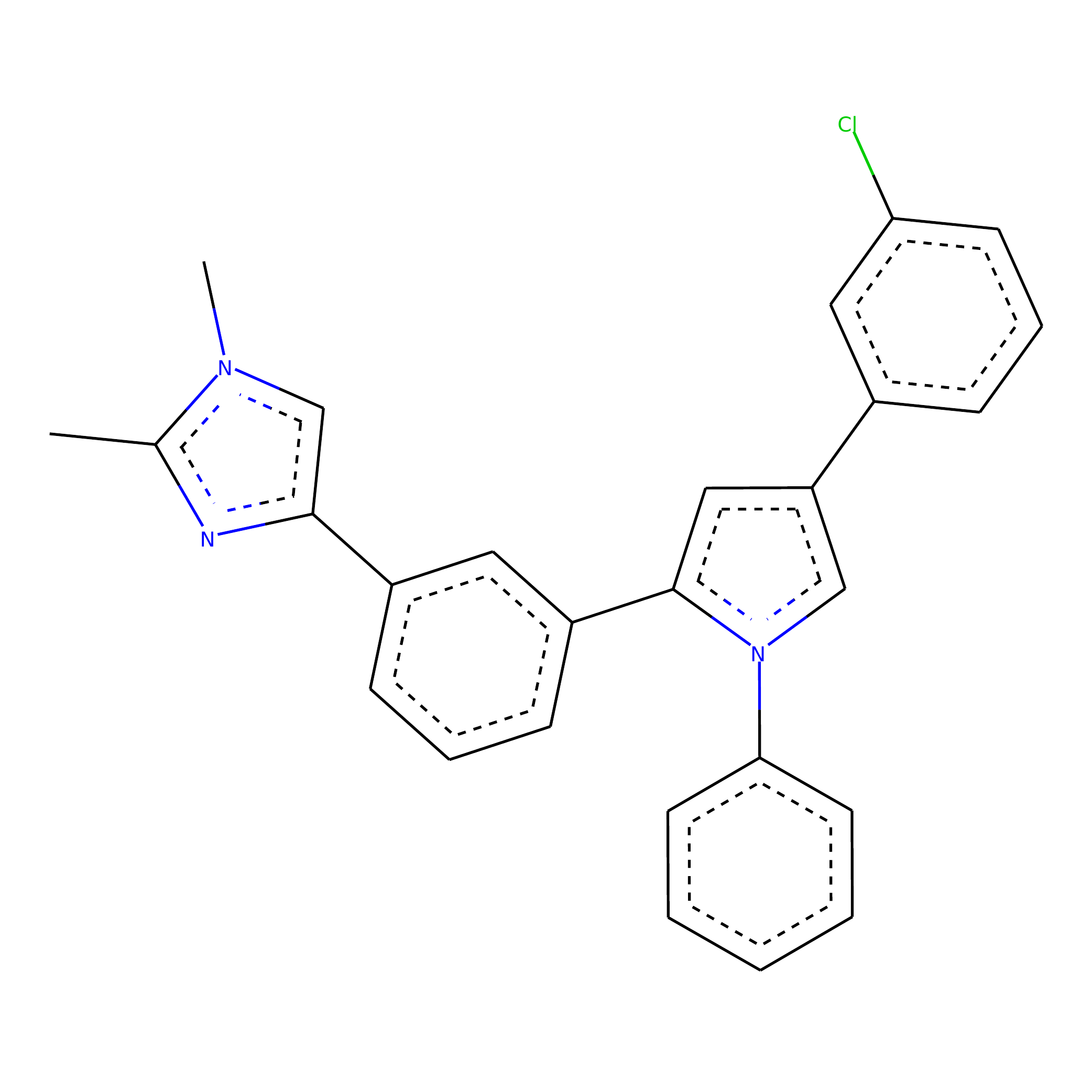}
4.209
\end{minipage}
 \end{figure*}
\begin{figure*}
\begin{minipage}{.24\hsize}
\centering
\includegraphics[width=\hsize]{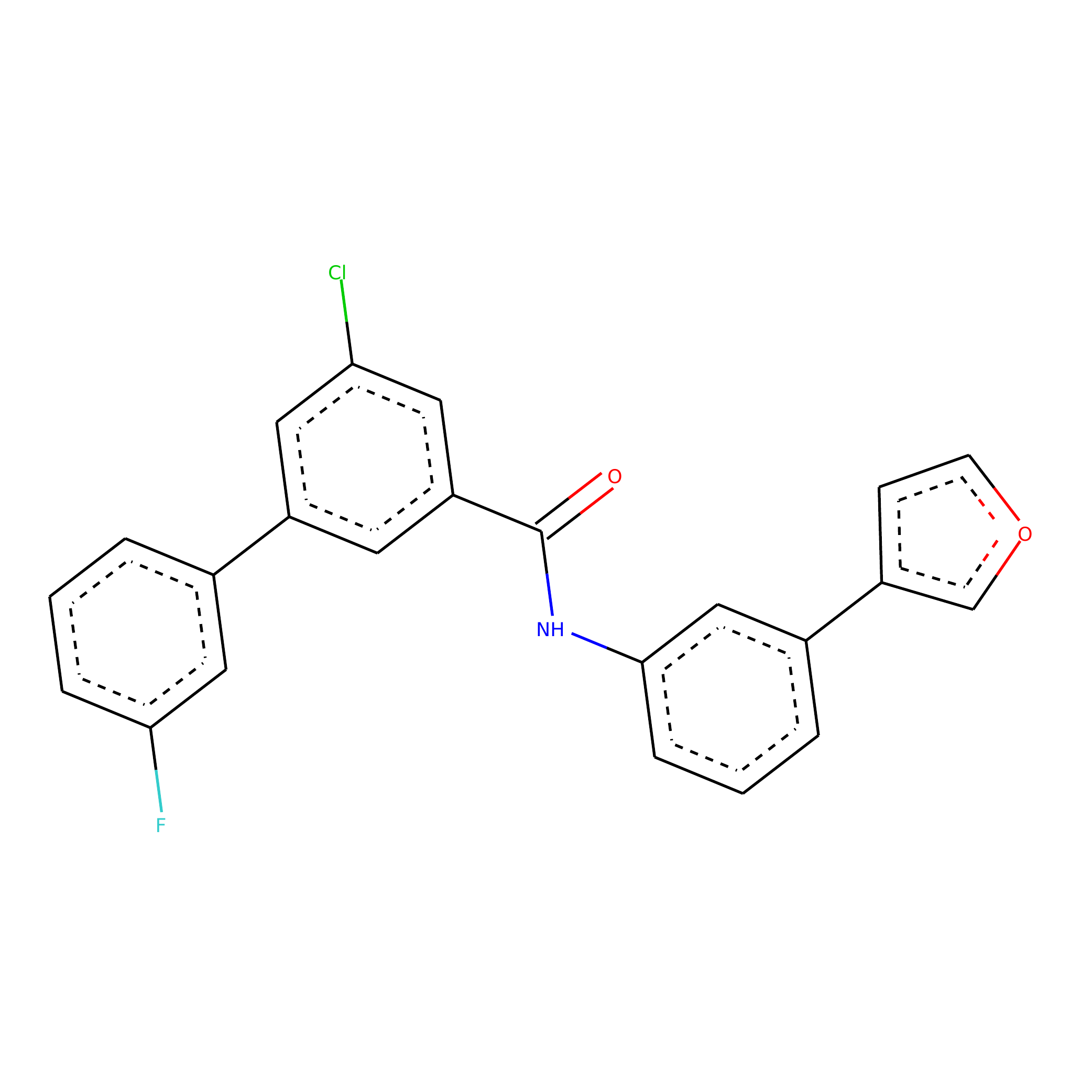}
4.202
\end{minipage}
\begin{minipage}{.24\hsize}
\centering
\includegraphics[width=\hsize]{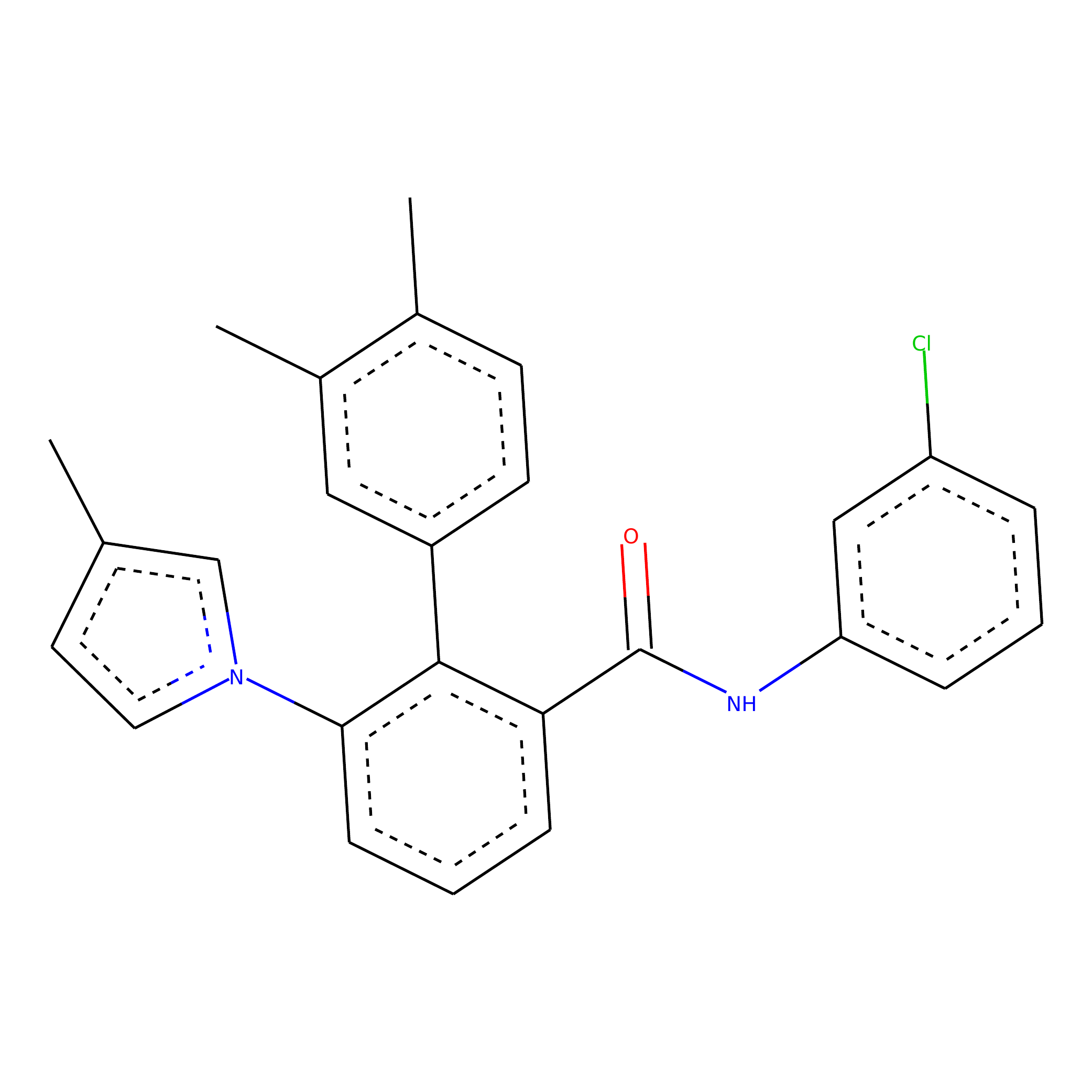}
4.196
\end{minipage}
\begin{minipage}{.24\hsize}
\centering
\includegraphics[width=\hsize]{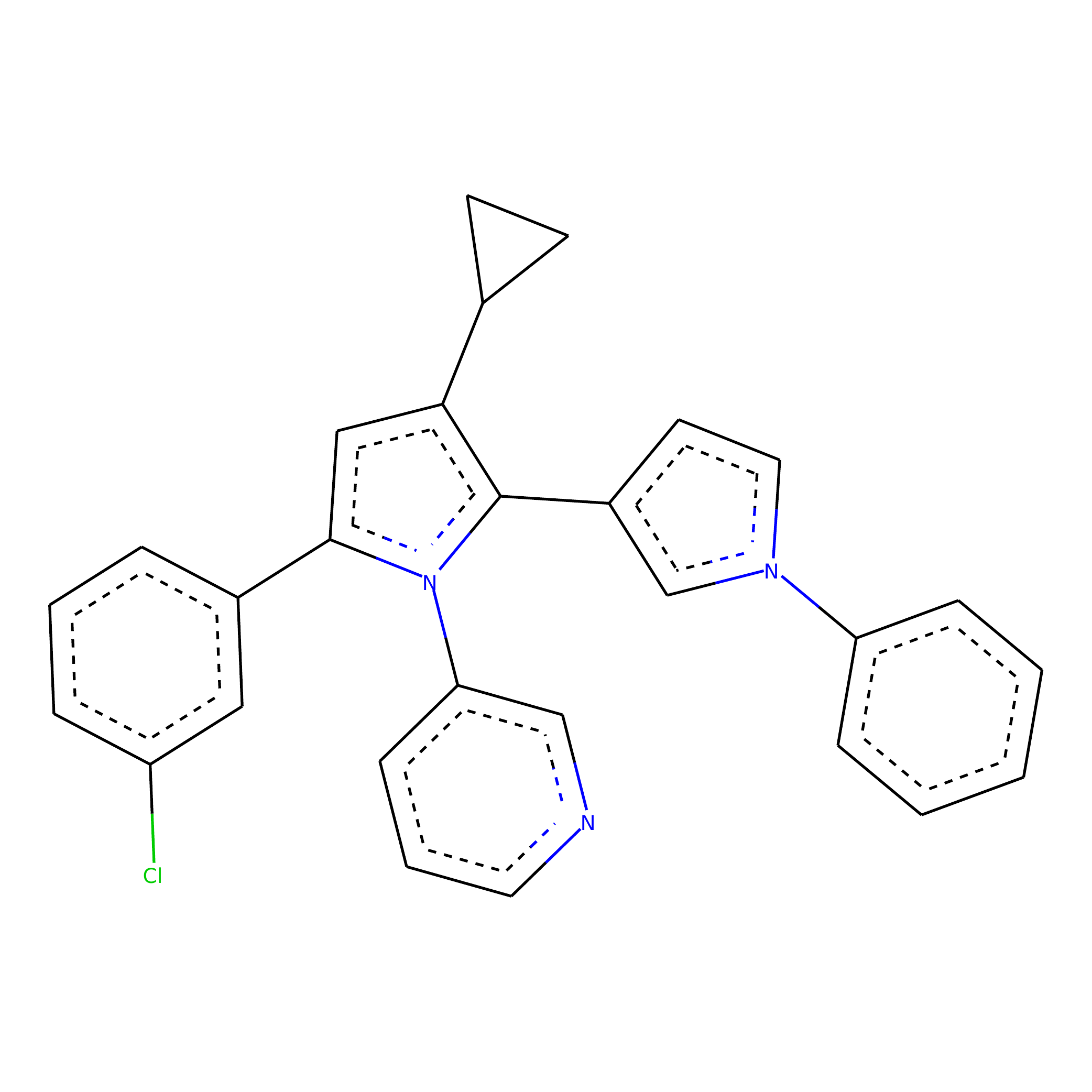}
4.185
\end{minipage}
\begin{minipage}{.24\hsize}
\centering
\includegraphics[width=\hsize]{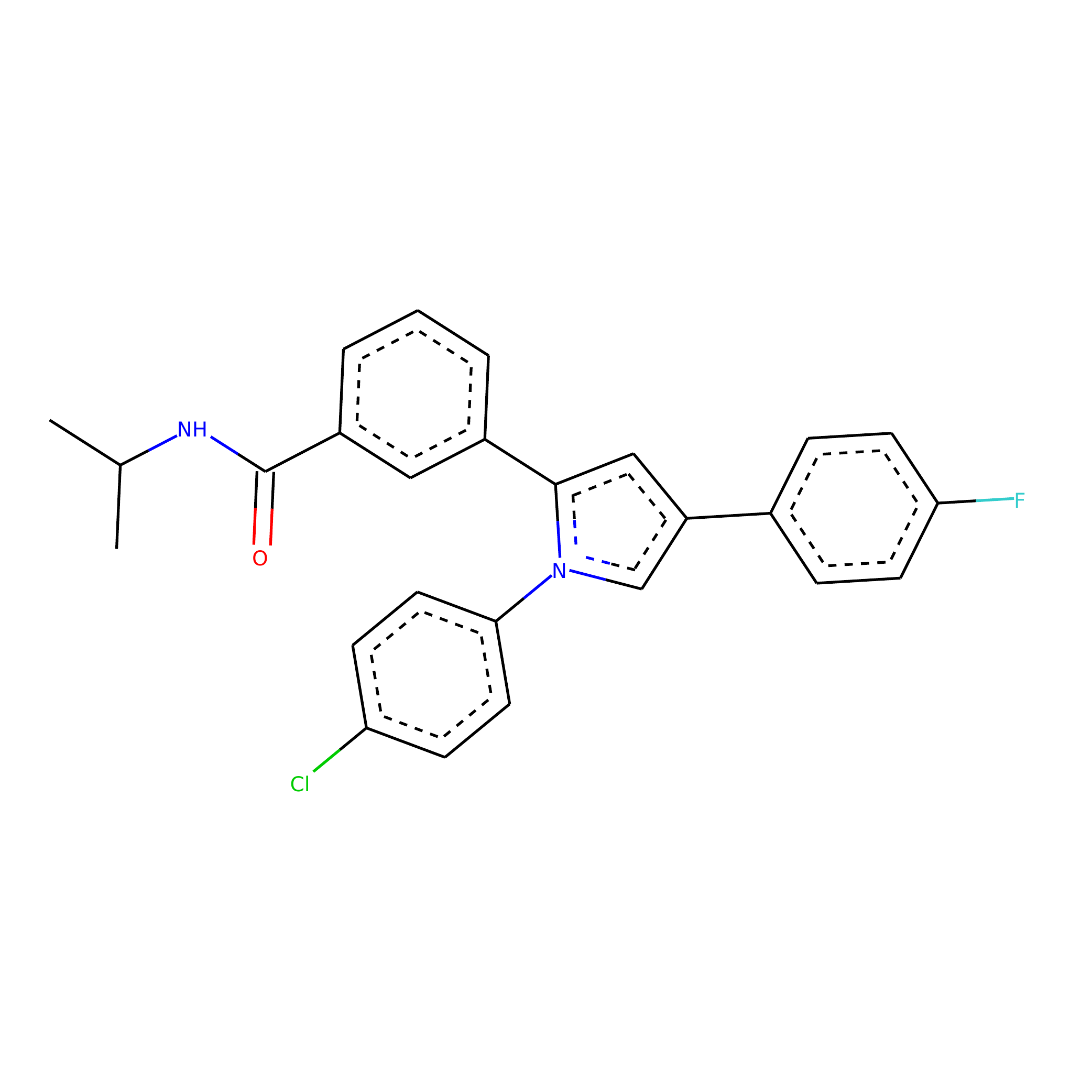}
4.171
\end{minipage}
\begin{minipage}{.24\hsize}
\centering
\includegraphics[width=\hsize]{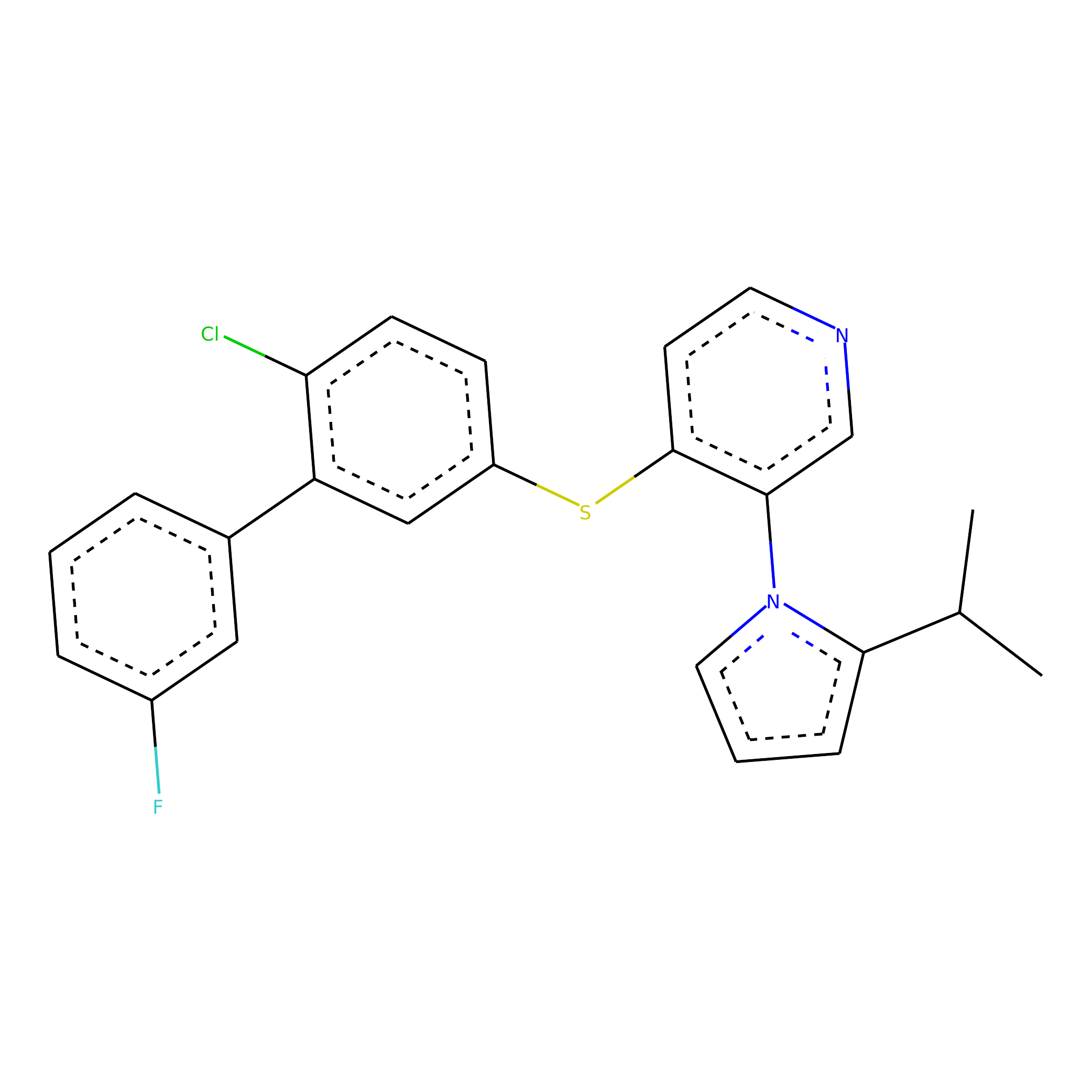}
4.169
\end{minipage}
\begin{minipage}{.24\hsize}
\centering
\includegraphics[width=\hsize]{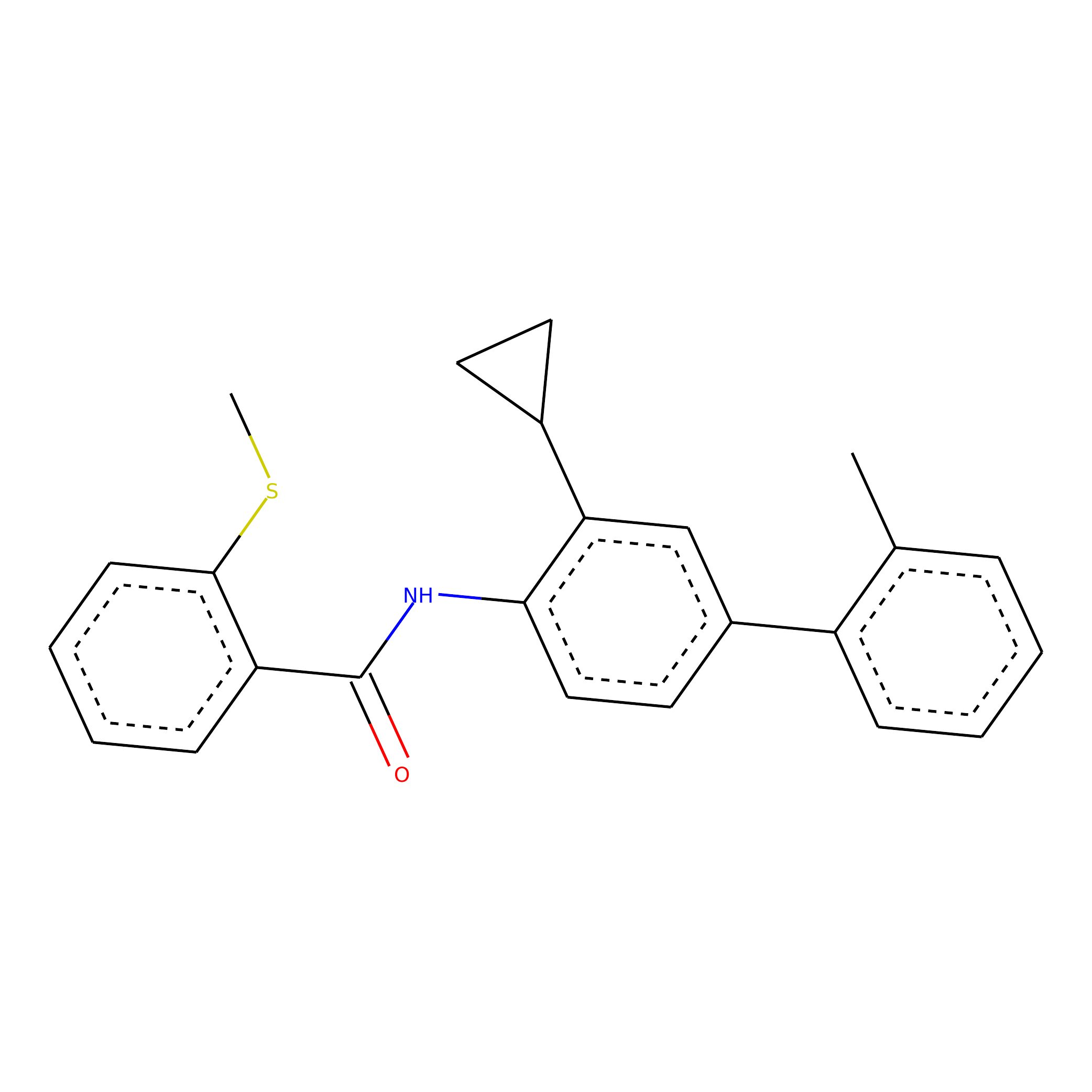}
4.163
\end{minipage}
\begin{minipage}{.24\hsize}
\centering
\includegraphics[width=\hsize]{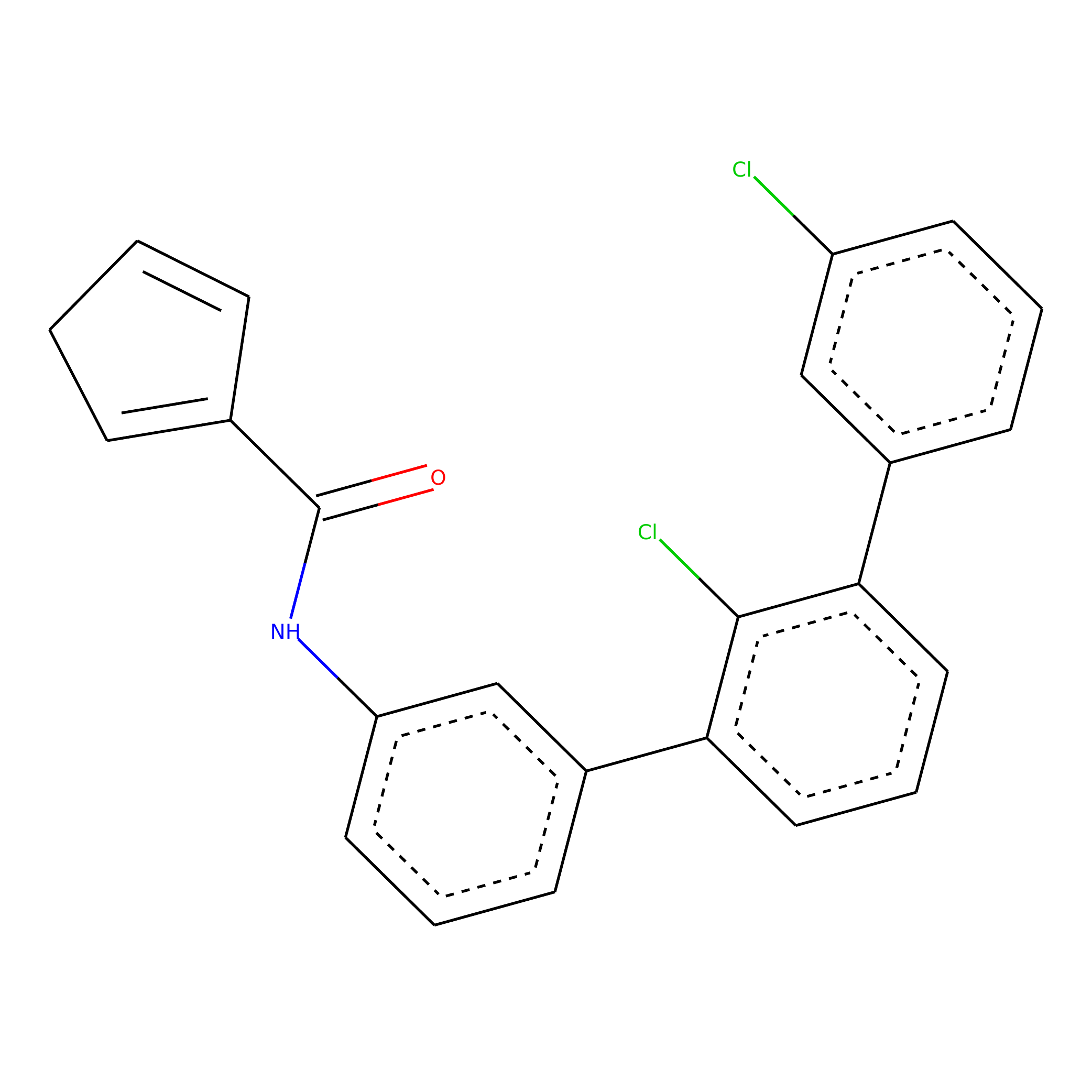}
4.153
\end{minipage}
\begin{minipage}{.24\hsize}
\centering
\includegraphics[width=\hsize]{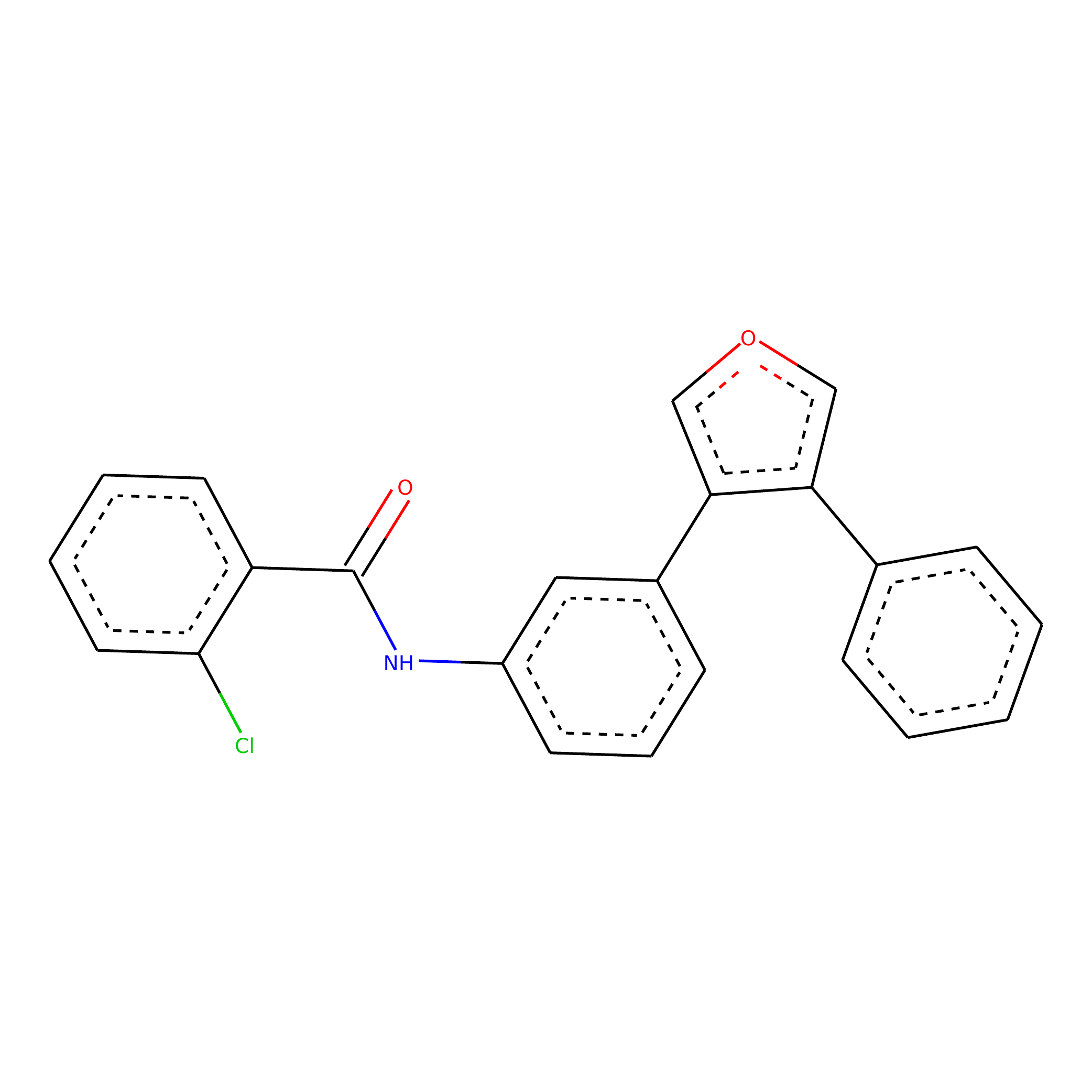}
4.150
\end{minipage}
\begin{minipage}{.24\hsize}
\centering
\includegraphics[width=\hsize]{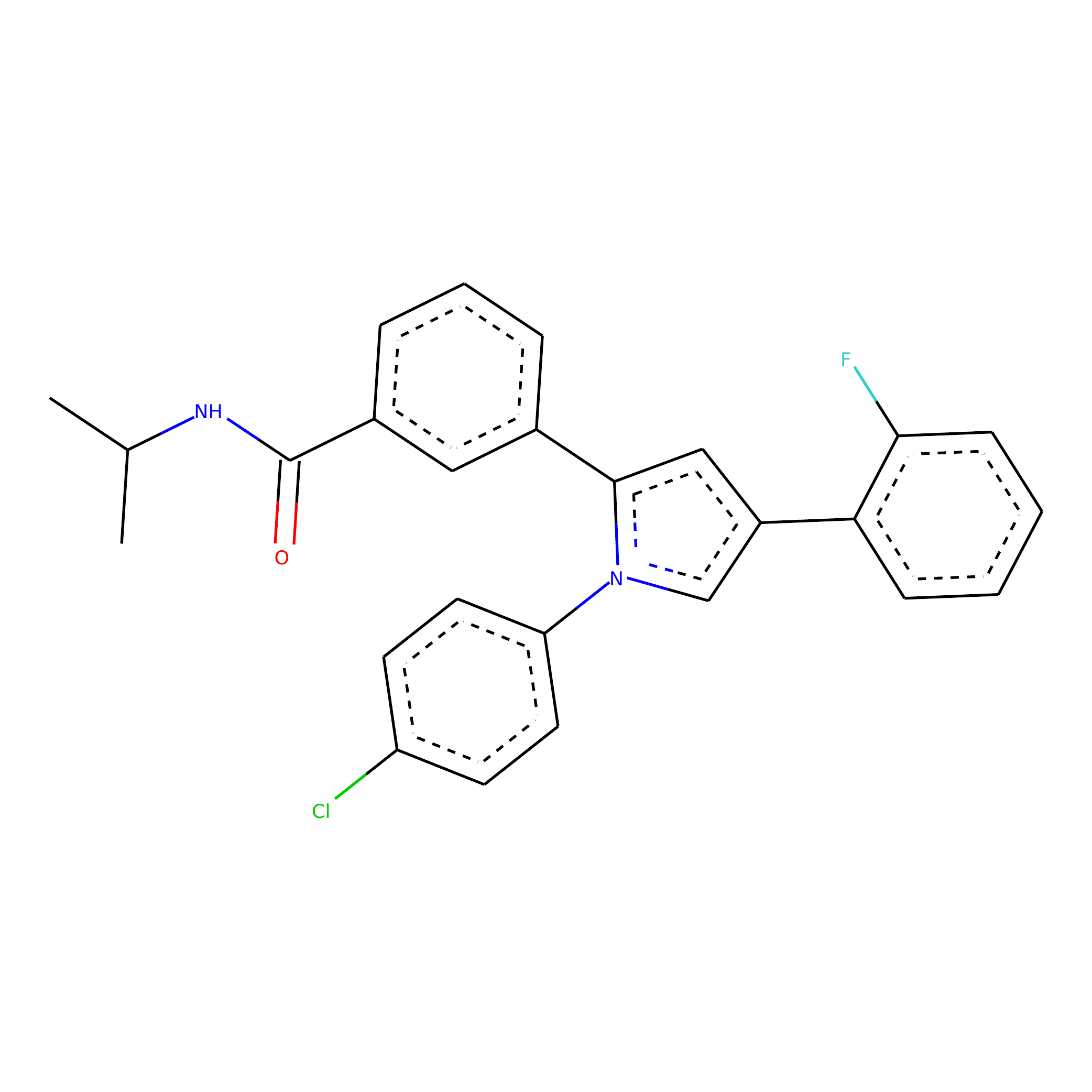}
4.124
\end{minipage}
\begin{minipage}{.24\hsize}
\centering
\includegraphics[width=\hsize]{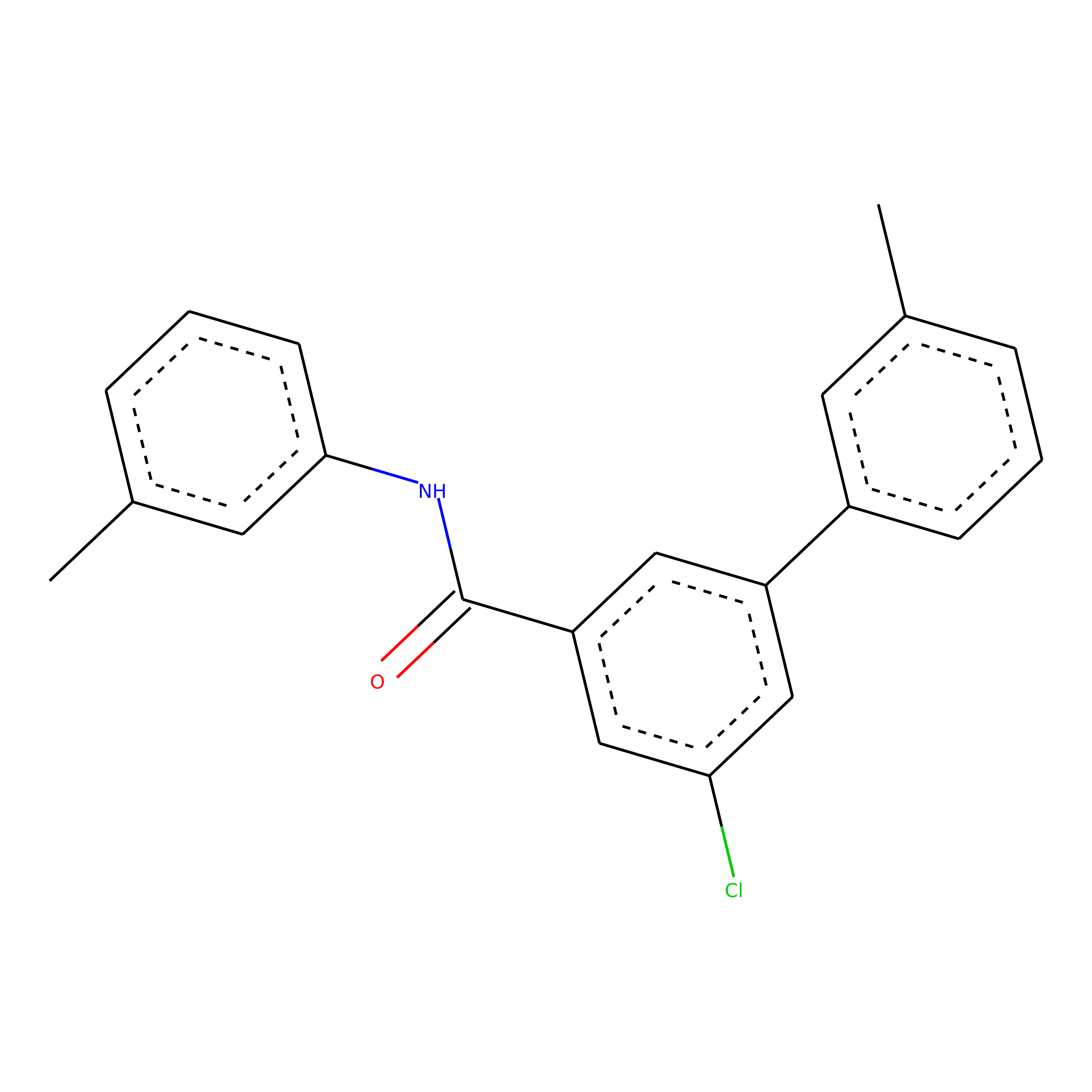}
4.124
\end{minipage}
\end{figure*}

\end{document}